\documentclass[11pt]{article}
\makeatletter
\@ifundefined{ifinsightrevision}{\newif\ifinsightrevision\insightrevisionfalse}{}
\makeatother
\usepackage{booktabs}
\usepackage[final]{acl}
\usepackage{amsmath}
\usepackage{amssymb}
\usepackage{amsthm}
\usepackage{booktabs}
\usepackage{graphicx}
\usepackage{svg}
\usepackage{times}
\usepackage{latexsym}
\usepackage{enumitem}

\usepackage[T1]{fontenc}

\usepackage[utf8]{inputenc}

\usepackage{microtype}

\usepackage[most]{tcolorbox}
\tcbset{
  promptcard/.style={
    enhanced,
    breakable,
    width=\linewidth,
    colback=black!3,
    colframe=black!70,
    colbacktitle=black!70,
    coltitle=white,
    fonttitle=\bfseries\footnotesize,
    boxrule=0.6pt,
    arc=2pt,
    outer arc=2pt,
    left=2mm,
    right=2mm,
    top=1.5mm,
    bottom=1.5mm,
    toptitle=0.8mm,
    bottomtitle=0.8mm,
    before skip=5pt,
    after skip=7pt
  }
}
\newtcolorbox{SystemPromptBox}{
  promptcard,
  title={System Prompt},
  fontupper=\rmfamily\footnotesize
}
\NewTCBListing{UserPromptBox}{!O{}}{
  promptcard,
  title={User Prompt Template},
  listing only,
  listing options={
    basicstyle=\rmfamily\footnotesize,
    breaklines=true,
    breakatwhitespace=false,
    columns=fullflexible,
    keepspaces=true,
    showstringspaces=false
  },
  #1
}

\newtcblisting{AlgorithmBox}[2][]{
  enhanced,
  breakable,
  width=\columnwidth,
  colback=black!2,
  colframe=black!65,
  colbacktitle=black!65,
  coltitle=white,
  fonttitle=\bfseries\footnotesize,
  title={#2},
  boxrule=0.6pt,
  arc=2pt,
  outer arc=2pt,
  left=1.5mm,
  right=1.5mm,
  top=1mm,
  bottom=1mm,
  toptitle=0.7mm,
  bottomtitle=0.7mm,
  before skip=5pt,
  after skip=7pt,
  listing only,
  listing options={
    basicstyle=\ttfamily\scriptsize,
    breaklines=true,
    breakatwhitespace=false,
    columns=fullflexible,
    keepspaces=true,
    showstringspaces=false
  },
  #1
}

\usepackage{inconsolata}

\usepackage{graphicx}

\newcommand{\tool}{InsightRAG~}
\newcommand{\toole}{InsightRAG}

\ifinsightrevision
  \newcommand{\lp}[1]{#1}
  \newcommand{\cltabref}[1]{\hyperref[#1]{Table~\ref*{#1}}}
  \def\clrefsetprefix#1:#2\clrefstop{%
    \def\clreftype{#1}\def\clrefapptype{app}%
    \ifx\clreftype\clrefapptype\def\clrefprefix{Appendix}\fi
  }
  \newcommand{\clsecref}[2]{%
    \begingroup
    \def\clrefprefix{Section}%
    \clrefsetprefix#1\clrefstop
    \hyperref[#1]{\clrefprefix~\ref*{#1} (#2)}%
    \endgroup
  }
\fi

\theoremstyle{plain}
\newtheorem{theorem}{Theorem}
\newtheorem{proposition}{Proposition}

\theoremstyle{remark}

\DeclareRobustCommand{\yuquan}[1]{%
  \if\relax\detokenize{#1}\relax
    \color[rgb]{0, 0, 0}%
  \else
    \textcolor[rgb]{0, 0, 0}{#1}%
  \fi
}
\tcbset{
  yuquan prompt/.style={
    listing options={
      basicstyle=\rmfamily\footnotesize\yuquan{},
      breaklines=true,
      breakatwhitespace=false,
      columns=fullflexible,
      keepspaces=true,
      showstringspaces=false
    }
  },
  yuquan algorithm/.style={
    listing options={
      basicstyle=\ttfamily\scriptsize\yuquan{},
      breaklines=true,
      breakatwhitespace=false,
      columns=fullflexible,
      keepspaces=true,
      showstringspaces=false
    }
  },
  yuquan orange algorithm/.style={
    listing options={
      basicstyle={\ttfamily\scriptsize\yuquan{}\color[rgb]{0, 0, 0}},
      breaklines=true,
      breakatwhitespace=false,
      columns=fullflexible,
      keepspaces=true,
      showstringspaces=false
    }
  }
}
\newcommand{\yuquanbibliographyitem}[2][]{%
  \normalcolor
  \yuquanoriginalbibitem[#1]{#2}%
  \ifcsname yuquannewbib@#2\endcsname
    \yuquan{}%
  \fi
}

\title{InsightRAG: Aggregating Distributed Scholarly Judgments for Scientific Survey Generation}

\author{
  Yuquan Lin$^{1}$\thanks{First author: \texttt{ylinfp@connect.ust.hk}},
  Wenshuo Zhang$^{1}$\thanks{Corresponding author: \texttt{wzhangeb@connect.ust.hk}},
  Yilin Ye$^{1}$, Shuchang Xu$^{1}$ \\
  {\bf Zhao Chen$^{2}$, Weiqi Wang$^{1}$, Huamin Qu$^{1}$} \\
  {\normalfont $^{1}$The Hong Kong University of Science and Technology} \\
  {\normalfont $^{2}$The Hong Kong University of Science and Technology (Guangzhou)}
}

\begin{document}
\ifinsightrevision
  \onecolumn
  \nolinenumbers
  \pagenumbering{roman}
  \pdfbookmark[0]{Cover Letter -- Submission 2120}{revision-cover-letter}
  \begin{center}
{\Large \lp{Reviewer-Wise Cover Letter}}\\[4pt]
{\large \lp{Round-Two Revision of Submission 2120}}
\end{center}


\lp{We thank the reviewers for their careful assessment of Submission 2120 and for recognizing the value of author-written Related Work for literature discovery, organization, and citation grounding. Their concerns fall into four groups, and the revision addresses each in turn.}

\lp{\textbf{Positioning and intended use.} We expand and reorganize the positioning in prior work on citation intent and importance, Related Work exploration, citation-based survey generation, and scholarly taxonomy construction, and replace the blanket ``untyped'' characterization with a conditional distinction (TDxV-1). We clarify how the research request is structured into a topic scope, where the initial seed comes from, and how it is verified before expansion, and we report concrete use cases from the expert interviews (TDxV-2).}

\lp{\textbf{Evidence beyond self-consistent LLM judgments.} We remove method names and source-path information from judge-facing inputs and rerun the taxonomy, key-point, and survey-quality evaluations, now scoring final surveys with three LLM judges (69jN-2). We add a passage-based citation audit of 84 claim--citation pairs against cited full texts, with 24 pairs checked manually (69jN-1). We add a method-blinded comparison by six domain experts on two topics, followed by semi-structured interviews (NqDY-2). We also distinguish what each taxonomy measure supports: correspondence to a human-survey reference, alignment with author-written evidence, or expert preference (NqDY-1).}

\lp{\textbf{Fairness and robustness of the comparisons.} We report baseline configurations, including candidate sources, cutoffs, and totals, and separate native-system comparisons from controlled retrieval comparisons; we also report backbone models, principal budgets, and available runtime and token estimates, and list the ten main-study seeds with their selection rule (69jN-3). We run seed-sensitivity experiments with three seeds on each of two topics (69jN-6, NqDY-3).}

\lp{\textbf{Scope of the claims.} We remove the approximation-guarantee statement, give the complete formulation with an NP-hardness proof, and clarify its relationship to greedy construction and LLM refinement (69jN-5). RW Tree patterns remain auxiliary structural evidence during global refinement, without a claim that they independently drive the performance gains; the ablation now evaluates refinement as a whole (69jN-4). We revise the interpretation of the results to strong taxonomy organization and comparable final-survey quality, together with higher citation support and reference reliability, rather than uniform superiority (69jN-2). We qualify claims about field-wide consensus as aggregation of recurring author-written perspectives within the retrieved literature (TDxV-4, NqDY-4), and qualify claims about recent-work coverage by clarifying forward expansion and adding the corresponding limitation (TDxV-3).}

\lp{The responses below explain each change, give the relevant evidence, and identify its location in the revised manuscript.}

\section*{\lp{Response to Reviewer TDxV}}

\paragraph{\lp{TDxV-1: Positioning within prior work.}}
\begin{itemize}
  \item \lp{\textbf{Concern.} The reviewer found the positioning too concentrated on recent LLM survey systems, with many references serving only as evaluation sources. The review identified relevant precedents in citation intent and importance, Related Work exploration, citation-based survey generation, and scholarly taxonomy construction, and questioned the broad statement that prior work treats citations as ``methodologically untyped.'' Relatedly was particularly relevant because it also treats Related Work sections as a record of distributed expert perspectives.}
  \item \lp{\textbf{Response.} We expanded and reorganized Related Work around literature discovery and understanding, taxonomy construction, and survey generation. The first part now discusses SciCite's labels for background, method, and result comparison, Valenzuela et al.'s distinction between important and incidental citations, and Relatedly's support for exploring Related Work paragraphs and developing topic outlines. The taxonomy discussion includes CHIME, HiGTL, and Zhu et al.'s TaxoBench-CS, while the survey-generation discussion adds Mohammad et al.'s use of citation texts and multi-document summarization. These additions establish the relevant precedents rather than presenting citation semantics or reuse of Related Work as new in themselves. We also replace the blanket ``untyped'' statement with a conditional distinction: citation connectivity alone does not preserve how a citing author describes earlier work. Our contribution is to extract those descriptions as attributed relational evidence, use them to discover a topic-specific scope, aggregate concept support across papers to organize that scope, and carry the organization into evidence-bounded writing. Citation-importance modeling offers a complementary way to distinguish the roles of individual citations.}
  \item \lp{\textbf{Where changed.} \clsecref{sec:related_work}{Related Work} now presents all three strands in that order. \clsecref{app:full_rw_citation}{Citation-Aware Scholarly Retrieval and Relation Modeling} expands the citation-semantics and Relatedly comparisons; \clsecref{app:full_rw_taxonomy}{Scholarly Taxonomy Construction} distinguishes content and graph representations from citing-author descriptions; and \clsecref{app:full_rw_asg}{Automatic Survey Generation} connects citation-text summarization to recent pipelines and explains how attributed relations support our retrieval-to-generation workflow.}
\end{itemize}

\paragraph{\lp{TDxV-2: Use case, query, and initial seed.}}
\begin{itemize}
  \item \lp{\textbf{Concern.} The reviewer found the intended usage difficult to picture because the query's breadth and specificity, the source of the seed paper, and the target user were unclear. In particular, possessing an in-scope seed might appear to require enough prior knowledge that the system would be less useful to someone entering a new area. The review requested concrete examples and use cases.}
  \item \lp{\textbf{Response.} We clarify the distinction between the research request, the topic scope, and the seed. The request is converted into a structured scope specifying the problem, focus, inclusion and exclusion boundaries, and any temporal requirement. It guides admission decisions; citation expansion starts from the seed rather than using the request directly as a keyword query for the entire literature. A researcher can therefore have a clear question and one relevant starting paper while still lacking the literature map needed to organize a review. The seed may be user-provided or obtained through search or an LLM, and the pipeline resolves its bibliographic record and verifies its fit before expansion. For reproducibility, we now list the ten manually selected, representative, high-impact, in-scope main-study seeds alongside the evaluation topics and source surveys. The six-expert interviews add concrete usage evidence: experts envisioned entering a new or adjacent area, organizing an initial review, identifying gaps, and positioning their own work. Their explanations emphasize topic-specific organization and understanding how different methods address the target problem. These observations clarify the practical role of the output as a starting scaffold for researcher inspection and revision.}
  \item \lp{\textbf{Where changed.} \clsecref{sec:method_workflow}{Workflow and Problem Formulation} defines the scope--seed input, and \clsecref{sec:method_scope_reconstruction}{Research-Scope Reconstruction} explains structuring and seed verification. \clsecref{app:evaluation-dataset}{Evaluation Dataset and Baseline Configurations} gives the ten topic titles, source surveys, and scope-construction procedure in and around \cltabref{app-dataset}. \clsecref{app:algorithm_retrieval}{Insight-Guided Retrieval} adds the main-study seed-selection rule and \cltabref{tab:initial_seeds} with the ten seed identities. \clsecref{app:expert_comparison}{Method-Blinded Expert Comparison} reports the interview-based use cases.}
\end{itemize}

\paragraph{\lp{TDxV-3: Coverage of recent work.}}
\begin{itemize}
  \item \lp{\textbf{Concern.} The reviewer identified delayed coverage of novel work as a critical limitation of relying on citation edges and citation counts. This delay could be especially consequential in rapidly developing areas, or longer in fields where citation accumulation and indexing proceed more slowly.}
  \item \lp{\textbf{Response.} We distinguish access through a citation relation from ranking by a candidate's citation count. Forward expansion retrieves later papers that cite an accepted paper; a new paper can thus be reached before it has accumulated citations of its own. Admission then depends on the target scope and supporting context, rather than a minimum citation count. This mechanism does not eliminate the coverage limitation. Papers that are not yet indexed or have no connection to the explored graph can remain unreachable. When a forward candidate pool exceeds the processing budget, prioritization by citation impact and recency may also omit valuable low-citation papers. We add this limitation explicitly and retain the candidate-prioritization description in Method so readers can see where the risk arises. The academic-profile analysis shows that the retrieved scope includes recent work in the tested topics, but the revised discussion does not equate that observation with exhaustive coverage of newly released literature.}
  \item \lp{\textbf{Where changed.} The ``Forward operator'' paragraph in \clsecref{sec:method_scope_reconstruction}{Research-Scope Reconstruction} explains the retrieval direction and prioritization when candidate pools are too large. The corresponding pseudocode appears in \clsecref{app:algorithm_retrieval}{Insight-Guided Retrieval}. The added ``Coverage of recent work'' paragraph in the Limitations identifies unindexed, disconnected, and low-citation papers as potential omissions. The Year Gap comparison remains in \clsecref{sec:exp_rq12}{Academic Profile of the Scope}.}
\end{itemize}

\paragraph{\lp{TDxV-4: Contribution-oriented Related Work perspectives.}}
\begin{itemize}
  \item \lp{\textbf{Concern.} The reviewer noted that Related Work sections are written to support individual contributions. Authors may characterize earlier work through the specific limitations their own paper addresses, rather than giving a balanced account of its strengths. The review asked how these local, contribution-oriented perspectives affect the resulting organization.}
  \item \lp{\textbf{Response.} We agree that author-written passages are attributed scholarly perspectives rather than objective descriptions of the entire field. The method preserves the source passages and collects accounts from multiple citing papers, allowing overlapping and conflicting concept assignments to remain visible in the candidate evidence. Aggregating recurring support reduces dependence on a single paper's framing, but it can still preserve shared omissions, selective emphasis, or systematic bias across authors. We therefore qualify the interpretation of the resulting taxonomy throughout the framing: it reflects recurring author-written perspectives in the retrieved literature. The Limitations describe how shared omissions or biases may carry over into the resulting taxonomy and survey, and how citation traversal may leave some parts of the field underrepresented. This qualification retains the value of existing human synthesis while making its provenance and coverage boundary clear. We also retain the recommendation that researchers inspect and revise the generated scope, taxonomy, and survey rather than treating them as authoritative substitutes for literature analysis.}
  \item \lp{\textbf{Where changed.} The Abstract and \clsecref{sec:introduction}{Introduction}, including the contribution statement, use the qualified description of recurring author-written perspectives. The ``Candidate concept space'' paragraph in \clsecref{sec:method_taxonomy}{Minimum-Conflict Taxonomy Construction and Survey Generation} explains how conflicting assignments and their support weights are retained. The ``Scope of author perspectives'' and ``Risks of over-reliance'' paragraphs in the Limitations state the remaining biases and the need for researcher verification.}
\end{itemize}

\section*{\lp{Response to Reviewer 69jN}}

\paragraph{\lp{69jN-1: Passage-based citation support and human audit.}}
\begin{itemize}
  \item \lp{\textbf{Concern.} The reviewer found the central citation-faithfulness evidence insufficient to establish that the cited papers actually support the generated claims. The review requested evaluation against passages in the cited full texts and human auditing of the automatic judgments, going beyond citation-entry checks and reference existence.}
  \item \lp{\textbf{Response.} We add a passage-based evaluation covering 84 claim--citation pairs: two topics, six methods, and seven pairs per topic--method cell. The Qwen-Plus judge receives no method identity. For the 74 pairs with a retrievable cited source, it reads the extracted full text and returns page-marked supporting evidence or the closest passage; the ten unverifiable references are counted as Unsupported. The reported binary comparison combines fully and partially supported automatic labels as Supported, so the rate indicates at least partial passage support. We additionally audit 24 pairs manually, two from each topic--method cell, with method identities withheld. \tool has the highest automatic support rate, 13/14 (92.9\%), compared with 11/14 for Vanilla LLM and SurveyGen, 10/14 for Deep Research, 4/14 for SurveyForge, and 2/14 for AutoSurvey. In the human subset, \tool and Vanilla LLM each have 4/4 supported pairs. Automatic and human judgments agree on 20/24 pairs (83.3\%, $\kappa=0.655$), with all four disagreements reflecting more lenient automatic judgments. This adds source-passage evidence and a direct human check while keeping the limited sample size visible.}
  \item \lp{\textbf{Where changed.} We add \clsecref{sec:exp_passage_faithfulness}{Passage-Based Citation Faithfulness} before the ablation study, with separate Setting and Results paragraphs. \clsecref{app:passage_faithfulness}{Passage-Based Citation Faithfulness Audit} provides source availability, label aggregation, audit size, agreement, and the complete six-method \cltabref{tab:passage_faithfulness}. \clsecref{app:prompt_passage_faithfulness}{Passage-Based Citation Support Judgment} gives the full-text judgment template. \clsecref{sec:exp_rq32}{Content Correspondence and Citation Support of Generated Surveys} retains the ten-topic citation-entry results as complementary evidence.}
\end{itemize}

\paragraph{\lp{69jN-2: Method-blinded evaluation and revised comparisons.}}
\begin{itemize}
  \item \lp{\textbf{Concern.} The reviewer requested removal of method names from evaluation prompts to reduce potential judge bias. Because the taxonomy, content, and final-survey assessments rely substantially on LLM judgments, identifying the generating system could influence the scores independently of the quality of the artifact.}
  \item \lp{\textbf{Response.} We remove method names and source-path information from the relevant judge-facing inputs and rerun the taxonomy-outline, key-point-recall, and survey-quality evaluations. The revised taxonomy average is highest for \tool at 3.6000, with leading Structural Balance (3.7500), Hierarchical Clarity (3.5000), and Logical Organization (3.6750); Deep Research remains slightly higher on Topic Uniqueness. \tool ranks second on Key-Point Recall at 60, behind AutoSurvey at 63. Final surveys are scored by GPT-5.2, Gemini 3.1 Pro, and DeepSeek V4 Pro for Coverage, Relevance, and Structure, giving 30 topic--judge averages per method. \tool is close to Vanilla LLM under GPT-5.2 (4.42 versus 4.48) and tied under Gemini (5.00) and DeepSeek (4.67). We revise the interpretation accordingly: the results support strong taxonomy organization and comparable final-survey quality, together with higher citation support and reference reliability. For example, \cltabref{tab:survey_faithfulness} reports Citation Recall/Precision of 0.856393/0.629516 and VRR of 0.995528 for \tool, versus 0.441418/0.406859 and 0.550306 for Vanilla LLM. These distinct measures explain the contribution more precisely than a claim of uniform superiority in content scores.}
  \item \lp{\textbf{Where changed.} The method-blinded outline scores and interpretation appear in \cltabref{tab:taxonomy_quality} and \clsecref{sec:exp_rq2}{Taxonomy Construction}. Updated Key-Point Recall appears in \cltabref{tab:survey_faithfulness}, interpreted in \clsecref{sec:exp_rq32}{Content Correspondence and Citation Support of Generated Surveys}. \clsecref{sec:exp_rq31}{LLM-Judged Survey Quality} describes the three-judge protocol and revised conclusion; \clsecref{app:judge_results}{Detailed LLM-Judge Results} gives all dimension-level scores in \cltabref{tab:survey_llm_judges}. \clsecref{app:evaluation_prompts}{Evaluation Prompt Templates} supplies the revised taxonomy, survey-quality, and key-point judgment inputs.}
\end{itemize}

\paragraph{\lp{69jN-3: Baseline settings, seeds, and resource reporting.}}
\begin{itemize}
  \item \lp{\textbf{Concern.} The reviewer found the backbone models, actual retrieval budgets, initial seeds, and baseline settings insufficiently specified. Without those conditions and runtime/token figures, the large performance gaps were difficult to attribute fairly: differences could reflect candidate access, retrieval volume, model choice, or computation as well as the proposed evidence use.}
  \item \lp{\textbf{Response.} We report the comparison conditions explicitly. Offline BM25 and Dense use the same frozen corpus of 1,064 papers and retrieve the top 50 per topic, totaling 500 each. Online BM25 and Dense use Semantic Scholar candidate pools, with top-200 and top-100 cutoffs and totals of 1,995 and 993; Dense uses BGE-base-en-v1.5 for ranking. Both follow the same RQ1 protocol within each access setting. SurveyGen, SurveyForge, and AutoSurvey are evaluated from their native system outputs, with 246, 1,063, and 2,063 papers respectively, so these end-to-end comparisons are distinguished from controlled retrieval comparisons. For \tool, we identify Qwen-Max for insight-guided retrieval and global refinement, deterministic greedy taxonomy construction, and Qwen-Plus for survey generation. The reported limits are retrieval depth four, paper budget 400, and maximum taxonomy depth four. We list the ten main-study seeds and their manual selection rule. Resource reporting gives retrieval input estimates of 1.4M--43.7M tokens per topic (median 9.1M), recorded refinement inputs of 4.7k--31.1k, and a generation final-output proxy of 14k--84k. Across ten topics, retrieval produces 1,879 papers, averaging 187.9 per topic; survey generation takes 5.0--28.2 minutes (median 6.4). We identify these accounting bases separately because they cover different stages and do not provide a complete end-to-end cost total.}
  \item \lp{\textbf{Where changed.} \clsecref{app:baseline_settings}{Baseline Configurations} introduces the comparison conditions, lists candidate sources, cutoffs, and totals in \cltabref{tab:baseline_settings}, and states the scope of the native-system comparisons. \clsecref{app:implementation_budget}{Backbone Models, Retrieval Budgets, Runtime, and Token Usage} gives the stage-level model, limit, and token table (\cltabref{tab:implementation_budget}), followed by the retrieved-paper total and generation-time statistics. \clsecref{app:algorithm_retrieval}{Insight-Guided Retrieval} adds the selection rule and ten seeds in \cltabref{tab:initial_seeds}. \clsecref{sec:exp_rq11}{Scope Coverage and Correctness} points readers to the seed appendix, and the Method overview now points to the implementation details and method prompts.}
\end{itemize}

\paragraph{\lp{69jN-4: RW Tree evidence and component attribution.}}
\begin{itemize}
  \item \lp{\textbf{Concern.} The reviewer questioned whether the distinctive contribution had sufficiently strong component controls. In the original submission, removing RW Tree evidence changed outline quality only from approximately 4.17 to 4.03 and barely changed assignment accuracy. Those observations did not support attributing the overall gains specifically to RW Tree evidence in refinement.}
  \item \lp{\textbf{Response.} Our follow-up control did not show a consistent independent benefit from RW-tree refinement, and we revise the component claim accordingly. RW Trees remain part of the implemented pipeline: they guide backward retrieval, and recurring category and parent--child patterns provide auxiliary structural evidence during global refinement. The primary taxonomy evidence consists of concept phrases and support weights extracted from author-written citation contexts; refinement also receives the scope definition and compact paper and concept profiles. We state these roles consistently rather than treating RW Tree evidence as the key source of performance gains. The revised main ablation evaluates taxonomy refinement as a whole. With the method-blinded evaluation, removing refinement reduces OutlineAvg from 3.6000 to 0.4813 and paper-assignment accuracy from 0.747758 to 0.257588, while HSR is slightly higher without refinement. These results support the role of the complete refinement stage in organization and assignment; they do not isolate the contribution of its auxiliary RW Tree input. The contribution statement therefore centers on author-written citation relations across scope reconstruction, taxonomy construction, and evidence-bounded generation.}
  \item \lp{\textbf{Where changed.} \clsecref{sec:method_workflow}{Workflow and Problem Formulation} identifies RW Tree patterns as auxiliary structural evidence. The ``Global refinement and generation'' paragraph in \clsecref{sec:method_taxonomy}{Minimum-Conflict Taxonomy Construction and Survey Generation} specifies all refinement inputs. \clsecref{app:algorithm_taxonomy}{Minimum-Conflict Taxonomy Construction} and \clsecref{app:prompt_global_taxonomy_refinement}{Global Taxonomy Structural Refinement} give the corresponding algorithm and actual prompt. \clsecref{sec:exp_ablation}{Ablation Study} and \cltabref{tab:taxonomy_ablation} report the revised full-refinement comparison without claiming a separate RW Tree gain.}
\end{itemize}

\paragraph{\lp{69jN-5: Theoretical claim and refinement boundary.}}
\begin{itemize}
  \item \lp{\textbf{Concern.} The reviewer found the theoretical claim incomplete because the approximation ratio, assumptions, and proof were missing. The connection between the formal objective, greedy construction, and subsequent LLM refinement was also unclear, making it difficult to determine which part of the implemented pipeline was covered by the claimed guarantee.}
  \item \lp{\textbf{Response.} We remove the approximation-guarantee statement and separate the optimization formulation from the practical solution procedure. The revised formulation makes the hierarchy, concept-selection, and leaf-assignment variables explicit, together with depth and branching bounds, one parent per selected concept, assignment of each paper to exactly one leaf, and sibling-dimension constraints. The objective rewards aggregate concept support and supported paper assignments while penalizing taxonomy complexity. We provide an NP-hardness proof by reduction from Exact Cover by 3-Sets; hardness already holds for a depth-one taxonomy, one semantic dimension, binary evidence weights, and a unit complexity penalty. This establishes why exact optimization is difficult, rather than a guarantee for the heuristic. In Method, greedy construction yields the initializer, after which the LLM resolves semantic and structural issues, paper assignments are updated, and deterministic validation repairs missing or duplicate assignments. The resulting greedy-and-refine pipeline is evaluated empirically, with no approximation ratio claimed for either the combined procedure or the subsequent LLM adjustments.}
  \item \lp{\textbf{Where changed.} \clsecref{sec:method_taxonomy}{Minimum-Conflict Taxonomy Construction and Survey Generation} revises the NP-hardness statement and distinguishes ``Optimization objective,'' ``Greedy taxonomy construction,'' and ``Global refinement and generation.'' \clsecref{app:complexity_statement}{Problem Restatement} specifies the complete variables, constraints, and objective; \clsecref{app:complexity_theorem}{NP-Hardness} gives the reduction and proof. \clsecref{app:algorithm_taxonomy}{Minimum-Conflict Taxonomy Construction} supplies the operational greedy and refinement steps corresponding to the Method description.}
\end{itemize}

\paragraph{\lp{69jN-6: Seed sensitivity.}}
\begin{itemize}
  \item \lp{\textbf{Concern.} The reviewer requested seed sensitivity and uncertainty across topics. Because expansion begins from a starting paper, a favorable seed could influence both the reachable literature and the measured scope-reconstruction gains; a single run per topic leaves that dependence unclear.}
  \item \lp{\textbf{Response.} We rerun retrieval on Multi-Agent Workflow Automation and AI-Powered Video Accessibility for Blind and Low Vision Viewers using the original seed and two alternative in-scope seeds per topic. The alternatives are suggested by an LLM from the topic, while the topic definition, human reference scope, and RQ1 evaluation protocol remain fixed. Correctness remains high across the six runs (0.8361--0.9783), but Coverage and F1 vary. For Multi-Agent Workflow Automation, the original seed retrieves 69 papers with Coverage/Correctness of 0.6400/0.8841; FlowReasoner retrieves 105 with 0.7200/0.9238, improving both measures; and EvoAgentX retrieves 122 with 0.8400/0.8361, attaining higher Coverage at lower Correctness. For Video Accessibility, Coverage ranges from 0.4615 to 0.6154 and F1 from 0.6272 to 0.7532, without increasing monotonically with paper count. The evidence therefore suggests that different seeds expose different reachable portions of the citation graph, rather than merely trading precision for retrieval volume. We report that dependence explicitly. This experiment provides per-seed variation on two topics; it does not supply repeated-run uncertainty estimates for all ten main-study topics.}
  \item \lp{\textbf{Where changed.} \clsecref{sec:exp_rq11}{Scope Coverage and Correctness} adds a concise seed-sensitivity finding and appendix pointer. \clsecref{app:seed_sensitivity}{Seed Sensitivity} uses Setting and Results paragraphs, explains how the alternatives were obtained and what was held fixed, and reports all six seed identities, paper counts, Coverage, Correctness, and F1 in \cltabref{tab:seed_sensitivity}. Its Results paragraph discusses the non-monotonic count--quality relationship and reachable-graph interpretation.}
\end{itemize}

\section*{\lp{Response to Reviewer NqDY}}

\paragraph{\lp{NqDY-1: Independence of taxonomy evaluation.}}
\begin{itemize}
  \item \lp{\textbf{Concern.} The reviewer noted that some taxonomy metrics use concepts extracted from the same Related Work evidence that informs construction. Such alignment can measure internal consistency but gives weaker independent evidence of expert or field-level organization. The review also cautioned that one human-written survey or survey section per topic provides an incomplete scope reference, as the retrieved-paper audit itself indicates.}
  \item \lp{\textbf{Response.} We distinguish the evaluation sources and the conclusions each can support. HSR and CEDS compare generated hierarchies with the human-survey reference: HSR measures heading correspondence with a depth penalty, and CEDS measures normalized tree-edit similarity using semantic node-matching costs. RWTA complements these measures by checking how well the taxonomy reflects organizational concepts expressed across the explored author-written literature. Because those concepts also inform construction, RWTA indicates evidence alignment and internal consistency rather than independent agreement with human organization. We preserve its positive role while making that boundary explicit. The results are interpreted by dimension: SurveyGen most closely matches the human-survey structure and granularity, whereas \tool has the highest average outline-quality score and RWTA; we do not claim that \tool leads every human-reference metric. We further add an independent six-expert comparison on two topics, with 5/6 taxonomy preferences for \tool. For scope coverage, the existing audit finds 90/100 sampled papers outside the human reference relevant, showing why absence from one survey is not sufficient to treat a retrieved paper as incorrect. We therefore distinguish reference correspondence, evidence alignment, and expert preferences rather than treating any one of them as proof of field-wide consensus.}
  \item \lp{\textbf{Where changed.} \clsecref{sec:exp_rq2}{Taxonomy Construction} identifies the human-survey hierarchy as the reference, presents the complementary evaluation dimensions, and interprets \cltabref{tab:taxonomy_quality} by those dimensions. \clsecref{app:rq2_metrics}{Taxonomy Evaluation Metrics} states the evidence-source distinction directly after the RWTA definition, with HSR and CEDS explained first. \clsecref{app:retrieval_expert_audit}{Expert Audit of Retrieved Papers} documents the incomplete-reference issue and \cltabref{tab:manual_false_positive}. \clsecref{app:expert_comparison}{Method-Blinded Expert Comparison} provides the independent taxonomy comparison.}
\end{itemize}

\paragraph{\lp{NqDY-2: Independent expert and claim-level checks.}}
\begin{itemize}
  \item \lp{\textbf{Concern.} Given the substantial use of LLM-based evaluation, the reviewer encouraged a small independent expert study of the resulting taxonomy and a sample of final survey claims, as well as extracted Related Work relations where feasible. Such evidence would help distinguish consistency with the system's own inputs from useful scholarly organization and supported writing.}
  \item \lp{\textbf{Response.} We add a comparison of \tool and SurveyGen on Multi-Agent Workflow Automation and Video Accessibility, with three domain experts per topic. Experts receive the topic scope and the two taxonomies and surveys as System A/B, with method names withheld and presentation order balanced. The guidance covers conceptual coverage, non-redundancy, and hierarchy for taxonomies, and content coverage, topical relevance, and organization for surveys. Experts choose A, B, or similar for each artifact and then complete a semi-structured interview about organization, strengths, weaknesses, and realistic use. \toole{}'s taxonomy is preferred in 5/6 judgments (2/3 and 3/3 by topic), and its survey in all 6/6 (3/3 on each topic). Interviews attribute the advantage chiefly to categories tied to the target problem and to problem-solving or technical progression, rather than simply broader coverage. All six perceive a connection between taxonomy and survey organization, with one noting that this applies to both systems. Experts also identify missing directions, limited technical specificity, occasional loss of focus, parallel rather than progressive subtopics, and weak integration of future-oriented sections. Separately, the 24-pair human citation audit checks source support for final claims. These additions directly assess the artifacts and sampled claims; comprehensive human validation of extracted intermediate relations remains outside their scope.}
  \item \lp{\textbf{Where changed.} \clsecref{sec:exp_rq2}{Taxonomy Construction} adds the 5/6 taxonomy-preference finding, and \clsecref{sec:exp_rq31}{LLM-Judged Survey Quality} adds the 6/6 survey-preference finding. \clsecref{app:expert_comparison}{Method-Blinded Expert Comparison}, within the expanded Human Evaluation appendix, gives the A/B protocol, evaluation guidance, topic-level counts, and positive and negative interview findings. \clsecref{sec:exp_passage_faithfulness}{Passage-Based Citation Faithfulness} summarizes the source-level check; \clsecref{app:passage_faithfulness}{Passage-Based Citation Faithfulness Audit} reports the full human-audit results.}
\end{itemize}

\paragraph{\lp{NqDY-3: Sensitivity to the starting paper.}}
\begin{itemize}
  \item \lp{\textbf{Concern.} The reviewer requested sensitivity to the initial seed paper. Since citation traversal begins from one anchor, different valid starting papers may emphasize different neighborhoods and produce different reconstructed scopes, even when the target topic is unchanged.}
  \item \lp{\textbf{Response.} We conduct a three-seed comparison on each of the same two topics used for the additional expert analyses, using one original and two LLM-suggested in-scope alternatives. The target definition, gold scope, and evaluation protocol are held fixed. On Multi-Agent Workflow Automation, Coverage ranges from 0.6400 to 0.8400, Correctness from 0.8361 to 0.9238, and F1 from 0.7425 to 0.8380. On Video Accessibility, the corresponding ranges are 0.4615--0.6154, 0.9516--0.9783, and 0.6272--0.7532. These results show relatively high topical Correctness alongside more substantial changes in coverage. FlowReasoner improves both Coverage and Correctness over the original workflow seed, and the video results are not monotonic in retrieved-paper count. We interpret this as sensitivity to which portions of the citation graph become reachable under the traversal schedule, rather than only a change in volume. The revised account therefore states seed dependence explicitly instead of assuming that all valid seeds reconstruct the same scope.}
  \item \lp{\textbf{Where changed.} \clsecref{sec:exp_rq11}{Scope Coverage and Correctness} adds the headline observation that Correctness remains high while Coverage and F1 vary. \clsecref{app:seed_sensitivity}{Seed Sensitivity} explains seed selection, fixed conditions, and the reachable-graph interpretation, with all per-seed scores and paper identities in \cltabref{tab:seed_sensitivity}. The ten main-study seeds and their selection rule are separately listed in \clsecref{app:algorithm_retrieval}{Insight-Guided Retrieval} so the original runs can also be identified.}
\end{itemize}

\paragraph{\lp{NqDY-4: Strength of the consensus claim.}}
\begin{itemize}
  \item \lp{\textbf{Concern.} The reviewer recommended qualifying ``consensus-aligned.'' The current evidence supports aggregation of recurring author-written perspectives, but does not by itself establish that those perspectives represent field-wide consensus. This distinction is particularly relevant when the explored literature and the human reference survey cover only part of a research area.}
  \item \lp{\textbf{Response.} We replace the consensus characterization with the claim that \tool aggregates recurring author-written perspectives within the retrieved literature. This wording describes the source and extent of the evidence: citing authors provide local judgments, recurring support across papers informs the taxonomy, and the resulting organization remains bounded by the explored scope. The Limitations describe shared omissions or biases that may carry over from author-written Related Work, together with incomplete coverage of the field due to the literature reachable through citation traversal. The evaluation discussion is consistent with that qualification. Human-reference metrics assess correspondence to a particular survey hierarchy, RWTA assesses alignment with author-written organizational evidence, and the expert comparison supplies preferences on two topics. Together they support the usefulness of the organization under the tested conditions without changing those observations into a claim of universal scholarly consensus.}
  \item \lp{\textbf{Where changed.} The Abstract and \clsecref{sec:introduction}{Introduction}, including the contribution list, describe recurring author-written perspectives in the retrieved literature. The ``Scope of author perspectives'' paragraph in the Limitations explicitly states the field-coverage and bias boundary. \clsecref{app:rq2_metrics}{Taxonomy Evaluation Metrics} distinguishes human-reference correspondence from RWTA evidence alignment, and \clsecref{sec:exp_rq2}{Taxonomy Construction} reports the complementary results without claiming field-wide consensus.}
\end{itemize}

\lp{We appreciate the reviewers' suggestions, which helped us strengthen the prior-work positioning, make the comparison conditions more transparent, and add direct evidence for source support and expert-perceived organization. The revised paper presents author-written citation relations as a useful scaffold for scope reconstruction, topic-specific taxonomy construction, and evidence-bounded survey generation, with the new analyses and limitations specifying the scope of that evidence.}

  \clearpage
  \pagenumbering{arabic}
  \pdfbookmark[0]{Revised Manuscript -- Submission 2120}{revision-manuscript}
  \linenumbers
\fi
\maketitle
\begin{abstract}
The rapid growth of scientific literature makes survey writing increasingly challenging, motivating automatic survey generation. 
Yet, existing systems must search an open space of papers and an open space of concepts with no trace of how the community itself has navigated them, leading to incomplete scope, weak taxonomies, and unreliable citations.
Our core insight is that author-written Related Work sections contain the human judgments needed to address both limitations: 
authors curate the citations relevant to their contributions and explain how the cited works relate, yielding paths for scope reconstruction and relational evidence for organization. 
We propose \toole. Its insight-guided retrieval follows author-supported citation relations to reconstruct a verified research scope and structure the resulting relational evidence.
\yuquan{It aggregates these judgments under a minimum-conflict objective into a taxonomy that reflects recurring author-written perspectives in the retrieved literature, guides generation, and bounds each section's evidence.}
\yuquan{\textcolor[rgb]{0, 0, 0}{Results on ten topics show gains in scope reconstruction, taxonomy quality, citation support, and reference reliability.}}

\end{abstract}
\section{Introduction}
\label{sec:introduction}
\begin{figure}[t]
    \centering
    \includegraphics[width=\columnwidth]{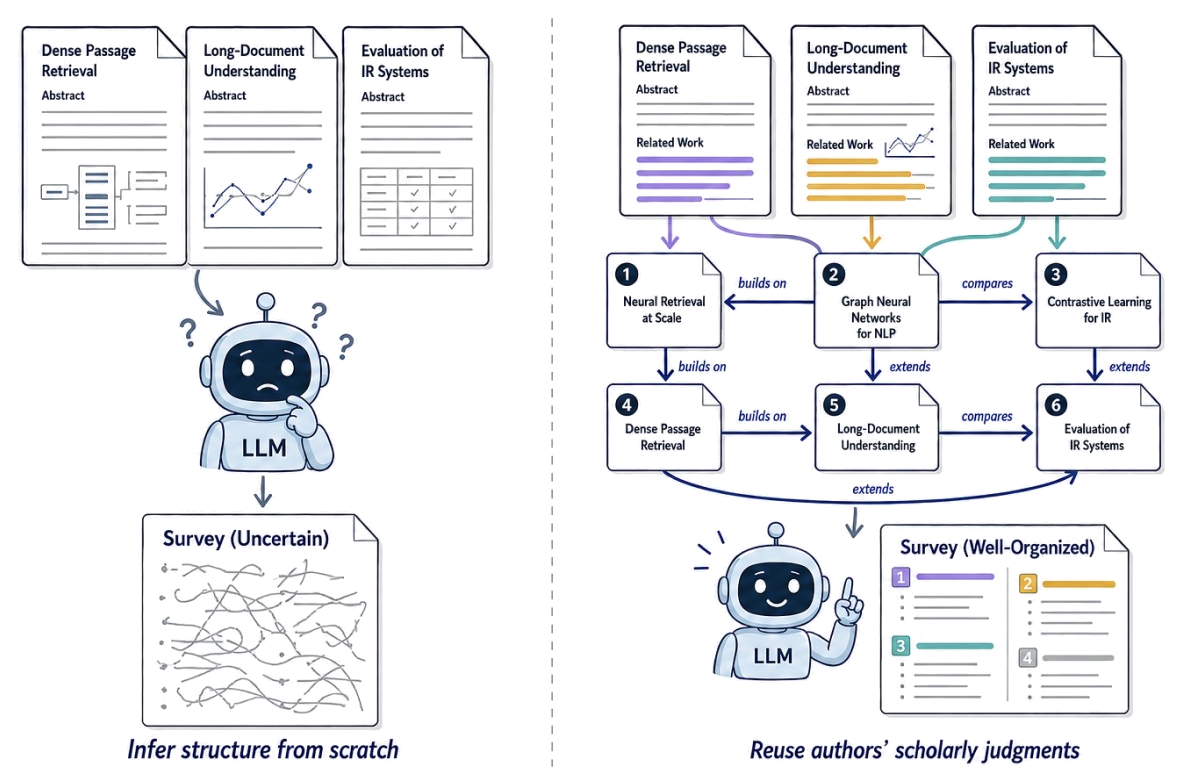}
    \caption{\yuquan{Motivation of \toole. Author-written Related Work sections record how different authors characterize prior work from their own perspectives; aggregating these local views grounds the survey in recurring perspectives within the explored literature.}}
    \label{fig:motivation}
\end{figure}
\yuquan{\textcolor[rgb]{0, 0, 0}{Writing a scientific survey is labor-intensive: authors must find, assess, and organize relevant work into a coherent taxonomy and narrative.}}
As publication rates grow, this cost rises faster than most groups can absorb. 
This motivated automatic survey generation (ASG), which agentizes this workflow into retrieval, reading, planning, and writing \citep{wang2024autosurvey,zhou2024llmmapreduce}. 
Subsequent systems refine its components with memory-driven retrieval and outline planning \citep{yan2025surveyforge}, structured preprocessing \citep{liang2025surveyx}, and quality-aware feedback \citep{bao2025surveygen}.

What a survey requires, however, is not a ranked list but a research scope: which works are foundational, which extend or challenge them, and which directions belong to the same topic. Reconstructing it confronts a language model with two open-ended search problems: retrieval must find relevant work in an open space of millions of papers ~\cite{coha2020specter, ajith2024litsearch, he2025pasa}, and writing must organize that evidence in an equally open space of concepts, inventing a taxonomy from scratch. Both spaces are too large for a model to navigate reliably—on our ten topics, no existing system recovers more than 11\% of a human survey's reference scope, and every baseline falls below 0.50 correctness. Recent work constrains this search with structure extracted from paper content ~\cite{pedinotti2026structsurvey,yang2026deepsurvey} or citation graphs \cite{zhang2024litfm,hu2025cgrag}, but neither provides the traces the search needs: a paper's content reflects only its own vantage point, and citation edges record that two papers are linked, not how ~\cite{zhang2026internatlas}. The model is thus left to search these open spaces without any trace of how the community itself has navigated them.

These traces already exist, written by authors themselves. When authors write a Related Work section, they do more than cite: they search the same two open spaces and record how they navigated them, selecting relevant prior work and articulating relations such as similarity, difference, limitation, and progression. Across a field, these local judgments accumulate into a crowdsourced account of its intellectual structure, supplying exactly the traces current pipelines discard: curated citation paths for finding literature, and explicit relational statements for organizing it. \yuquan{\textcolor[rgb]{0, 0, 0}{This motivates our research question: can author-written Related Work judgments serve as a reliable scaffold for reconstructing a field's scope, taxonomy, and coherent survey narrative?}}

We answer this question with \toole, which builds one stage of the workflow on each of these two signals. 
\emph{Insight-guided retrieval} follows the citation paths authors have already curated: from a seed paper, it traverses the citation graph forward and backward, admitting a paper only when an author-written passage supports the relation and the paper lies within the user-defined scope. 
It returns a verified paper set and two records of author judgment: a research-scope graph whose edges carry the relation an author asserted between two papers---improvement, extension, comparison, or critique---and, per paper, a Related Work Tree recording how its authors grouped prior work and placed their own contribution within it. 
\emph{Taxonomy construction} then aggregates the relational statements themselves: the concept phrases stated on the graph's edges are merged into a shared candidate concept space, over which a minimum-conflict objective selects a taxonomy and paper assignment that maximize concept support and assignment support while penalizing taxonomy complexity, subject to structural and semantic constraints. 
Since exact optimization is NP-hard, \tool approximates it with a coverage-oriented greedy initializer and a global LLM refinement pass, then assigns each paper to one leaf. 
The taxonomy and its assigned papers then serve as the writing plan: leaf sections synthesize the assigned papers, internal sections explain how their children relate, and each section may cite only its own papers. 

Experiments on ten computer-science survey topics, each annotated with a gold reference scope and a gold taxonomy from a human-written survey, evaluate the intermediate artifacts as well as the final text. 
\tool reconstructs research scope far more completely and accurately than retrieval and survey-generation baselines, builds taxonomies that are better organized and better grounded in author-written evidence, and produces surveys with substantially more faithful citations. 

Overall, Our contributions are threefold:
\begin{itemize}[leftmargin=*, labelindent=0pt, itemindent=0pt]
\item We identify author-written Related Work judgments as the scholarly-judgment signal missing from autonomous survey generation, and reorganize the end-to-end workflow around it.
\item \yuquan{\textcolor[rgb]{0, 0, 0}{We propose \toole, combining insight-guided citation traversal with minimum-conflict aggregation of distributed author judgments to produce a taxonomy that reflects recurring author-written perspectives in the retrieved literature, bounds the evidence available to each survey section, and guides generation of the final survey.}}
\item \yuquan{\textcolor[rgb]{0, 0, 0}{We introduce an evaluation setting over ten survey topics that scores research-scope reconstruction, taxonomy quality, and final survey quality separately at successive pipeline stages.}}
\end{itemize}

\section{Related Work}
\label{sec:related_work}

\yuquan{\textcolor[rgb]{0, 0, 0}{\textbf{Citation-aware scholarly retrieval and exploration.} SciCite classifies citation intent as background, method, or result comparison \citep{cohan2019scicite}, while Valenzuela et al. distinguish important from incidental citations \citep{valenzuela2015meaningful}. Scholarly retrieval also uses lexical or embedding similarity \citep{coha2020specter} and citation-graph propagation \citep{zhang2024litfm,hu2025cgrag}; these signals alone do not preserve how citing authors describe prior work. Relatedly helps readers explore Related Work paragraphs across papers and form topic outlines \citep{palani2023relatedly}, but leaves the synthesis of these perspectives to readers. \tool uses passages surrounding citations as author-attributed relational evidence to discover and verify a topic-specific literature scope.}}

\yuquan{\textcolor[rgb]{0, 0, 0}{\textbf{Scholarly taxonomy construction.} CHIME organizes a given paper set with expert corrections to category relations and paper assignments \citep{hsu2024chime}. HiGTL combines paper text and citation-graph structure for hierarchical clustering and concept naming \citep{hu2025higtl}, while Zhu et al. use LLM-generated aspect representations and clustering and introduce TaxoBench-CS \citep{zhu2025context}. These methods do not explicitly use citing authors' descriptions of prior work as evidence for taxonomy construction. \tool extracts concept phrases from author-written Related Work passages surrounding citations and aggregates their support across papers.}}

\yuquan{\textcolor[rgb]{0, 0, 0}{\textbf{Automatic survey generation.} Mohammad et al. combined citation texts with multi-document summarization to generate technical surveys \citep{mohammad2009using}. LLM-based systems commonly follow a retrieve--plan--write pipeline \citep{wang2024autosurvey,zhou2024llmmapreduce}, with later work improving retrieval, preprocessing, feedback, collaboration, and evaluation \citep{wu2025autosurvey2,yan2025surveyforge,liang2025surveyx,bao2025surveygen,go2025lira,surveylens2026}. Some systems also extract relations or taxonomies from paper content \citep{pedinotti2026structsurvey,yang2026deepsurvey}, but these structures need not capture how subsequent authors position earlier work. \tool carries its reconstructed scope and taxonomy into an evidence-bounded writing plan. Appendix~\ref{app:full_related_work} provides a more detailed comparison.}}

\section{Methodology}
\label{sec:methodology}

\setlength{\abovedisplayskip}{0pt}
\setlength{\belowdisplayskip}{4pt plus 1pt minus 1pt}
\setlength{\abovedisplayshortskip}{0pt}
\setlength{\belowdisplayshortskip}{3pt plus 1pt minus 1pt}

\subsection{Workflow and Problem Formulation}
\label{sec:method_workflow}

\tool operationalizes the two signals recorded in author-written Related Work sections: curated citation paths for reconstructing a research scope, and explicit relational statements for organizing it. \yuquan{\textcolor[rgb]{0, 0, 0}{Rather than use keywords to retrieve references and infer the field's structure from scratch, \tool recovers and aggregates these distributed human judgments across the retrieved literature.}}

The input is \(Q=(d,p_0)\): 
a scope definition \(d\) specifying the survey topic, focus, and literature boundary, 
and a seed paper \(p_0\) given by the user or retrieved by an LLM.
\textbf{Insight-guided retrieval} (\S\ref{sec:method_retrieval}) follows from \(p_0\) the curated citation paths supported by author-written contexts, returning a topic-specific paper set \(P\), a research-scope graph \(G\) whose edges carry those contexts, and for each \(p\in P\) a Related Work (RW) Tree \(R_p\) recording how its authors grouped prior work. 
\yuquan{\textcolor[rgb]{0, 0, 0}{\textbf{Taxonomy construction} (\S\ref{sec:method_taxonomy}) takes \((P,G,\{R_p\}_{p\in P})\), builds a draft hierarchy from concept phrases and support weights in the author-written citation contexts of \(G\), refines it into a taxonomy \(T\) using the scope definition, compact concept and paper profiles, and RW Tree patterns as auxiliary structural evidence, and assigns papers to its leaves through a map \(A\).}}
Together \(T\) and \(A\) bound the evidence available to every survey section: a leaf is written from its assigned papers alone, while an internal node explains how its children relate. An LLM then generates the final survey.
\yuquan{\textcolor[rgb]{0, 0, 0}{Detailed algorithms, backbone models, principal budgets, and resource-use figures are provided in Appendix~\ref{app:algorithmic_details}; method prompts for all three stages appear in Appendix~\ref{app:method_prompts}.}}

\begin{figure*}[t]
    \centering
    \includegraphics[width=\textwidth]{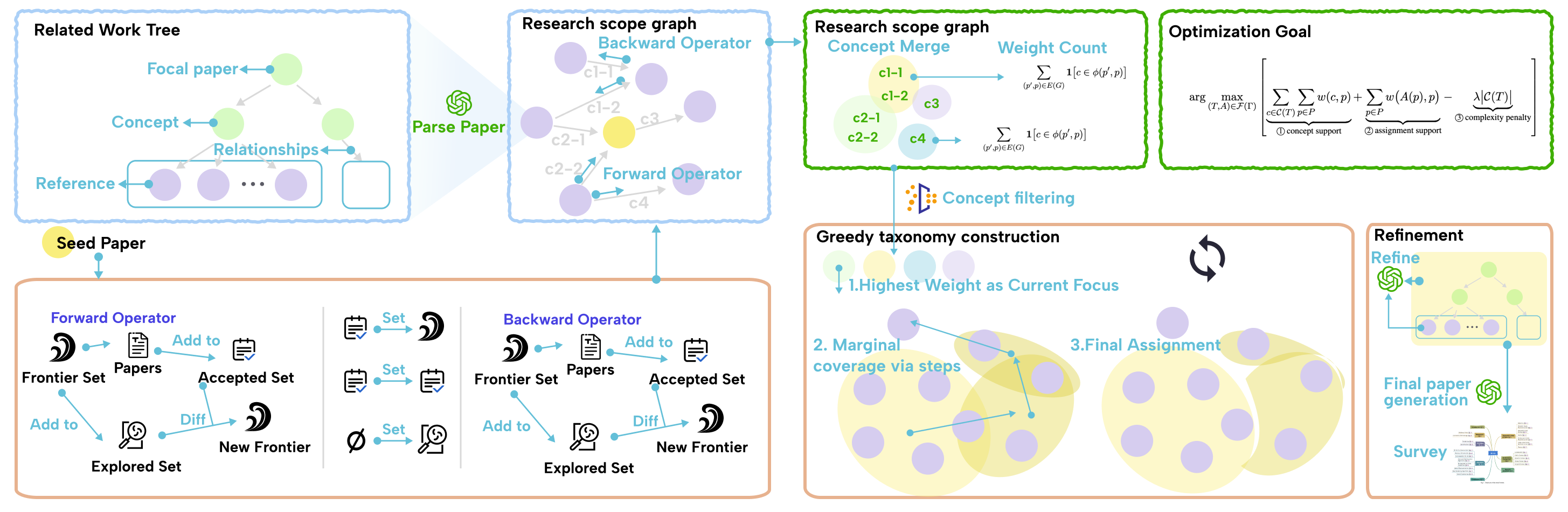}
    \caption{Overview of \toole. Insight-guided retrieval performs scope-verified forward and backward expansion to construct a verified paper set, a relation-bearing research-scope graph, and per-source Related Work trees. Minimum-conflict taxonomy construction produces and refines the taxonomy, assigns papers to leaves, and validates the resulting structure. Evidence-bounded survey generation then converts taxonomy nodes into writing units with local evidence packages and constrained citation identifiers.}
    \label{fig:method_overview}
\end{figure*}

\subsection{Insight-Guided Retrieval}
\label{sec:method_retrieval}

This stage must recover a research scope rather than a ranked list: which works are foundational, which extend or challenge them, and which directions belong to the same topic. 
\tool recovers it from author-written Related Work sections, which already carry local relevance judgments and explicit statements of how a citing paper relates to the work it cites: 
\yuquan{\textcolor[rgb]{0, 0, 0}{a citation relation is followed only when such a passage supports it, and the paper it leads to is admitted only after verification against the user-defined research scope.}}

\subsubsection{Related Work Tree Extraction}
\label{sec:method_rwt}
For each accepted paper \(p\), \tool extracts an RW Tree \(R_p\) from the paper's author-written Related Work discussion and the citations resolved from that text. 
Its internal nodes are the categories the authors used to organize prior work, so that papers under one internal node form an author-defined conceptual group; 
its leaves are the cited papers, with their meta information and the role the authors assigned them, such as method or limitation; 
and its root is the focal paper with its research context. 
The edges inside the tree \(E_p\) encode the authors' conceptual organization, such as the improvement from a leaf paper to the root, rather than the raw citation links used for traversal. 
If the Related Work section does not provide a complete hierarchy, \tool retains only the structure supported by the text instead of asking the model to fill in the missing parts.

\subsubsection{Research-Scope Reconstruction}
\label{sec:method_scope_reconstruction}

Reconstruction turns input \(Q\) into the paper set \(P\) and the research-scope graph \(G\). 
\tool first converts \(d\) into a structured research context holding the scope constraints, the intended perspective, and any temporal requirement, and resolves \(p_0\) against a scholarly index for its meta information and citation relations; retrieval starts only if the seed's research problem, methodology, and contribution satisfy the scope.

The search then proceeds in rounds. A candidate is accepted only if it falls within the target scope and then stored as RW Tree \(R_p\), then becomes nodes in \(G\), and each directed edge in \(G\) stores the supporting passage from the citing paper to the cited paper. 
To support this iterative process, the algorithm maintains three sets: the accepted set \(S_a\), the frontier set \(S_f\) of papers awaiting expansion, and the explored set \(S_e\) of papers already expanded. The accepted set and the frontier are initialized with the seed, \(S_a=S_f=\{p_0\}\), while the explored set is initially empty, \(S_e=\emptyset\).

\noindent\textbf{Forward operator.}
Given a paper \(p\), the forward operator retrieves the papers that cite \(p\), each paired with the citation sentence or Related Work passage in which the citing paper discusses \(p\).
This direction targets later developments, extensions, applications, comparisons, and critiques.
When the citing papers are too numerous to process exhaustively, highly cited and recent papers are prioritized to keep the operator technically feasible.
A permissive pre-filter removes clearly uninformative candidates, and each remaining candidate must pass an admission test against the target scope.
Admitted candidates are deduplicated against the accepted set \(S_a\) and, if new, added to it; their supporting passages are incorporated into \(G\) as relation evidence.

\noindent\textbf{RW-guided backward operator.}
Given a paper \(p\), the backward operator retrieves the references of \(p\) by extracting its RW Tree \(R_p\) (\S\ref{sec:method_rwt}) and checking whether each leaf of \(R_p\), paired with the author-written passage that describes it, passes the same admission test.
Admitted predecessors are deduplicated against \(S_a\) and, if new, added to it; their relation evidence is incorporated into \(G\).
\yuquan{Each admitted predecessor is expanded through its own RW Tree, so backward exploration recursively exposes earlier portions of the verified scope.}

\paragraph{\textbf{Round structure and state update.}}
Starting from \(S_a=S_f=\{p_0\}\) and \(S_e=\emptyset\), a run executes a bounded number of forward steps followed by a bounded number of backward steps.
In each forward step, the forward operator is applied to every paper in the frontier \(S_f\); the expanded papers are added to the explored set \(S_e\), newly admitted papers are added to \(S_a\), and the frontier is updated to \(S_a\setminus S_e\).
The forward phase comes first so that major recent research directions are identified before their earlier foundations are traced.
Once the forward steps are complete, the frontier is reset to the full accepted set and the explored set is cleared, i.e., \(S_f\leftarrow S_a\) and \(S_e\leftarrow\emptyset\), while \(S_a\) is kept unchanged.
Each backward step then applies the backward operator to every paper in \(S_f\) and parses its RW Tree; the expanded papers are added to \(S_e\), newly admitted predecessors are added to \(S_a\), and the frontier is again updated to \(S_a\setminus S_e\).
After the backward steps terminate, a loop-closure step densifies \(G\) by recovering missing citation links among accepted papers, merging duplicate nodes, and attaching new author-written evidence to existing relations.
This completes the reconstruction run.

\paragraph{\textbf{Post-termination reseeding and graph merging.}}
After a run terminates, \tool selects new seeds with high in-degree in \(G\), reruns reconstruction independently from each seed, and merges the resulting graphs with deduplication and loop closure.
\yuquan{\textcolor[rgb]{0, 0, 0}{The merged graph and its verified papers define \(G\) and \(P\), respectively, for taxonomy construction. The RW Trees are retained for global refinement, while the author-written citation contexts on the edges of \(G\) provide concept phrases and support weights for the draft hierarchy.}}

\subsection{Minimum-Conflict Taxonomy Construction and Survey Generation}
\label{sec:method_taxonomy}

The retrieval stage produces a verified literature space but not the global organization a survey requires.
\yuquan{The author-written citation contexts retained on the edges of \(G\) encode local views that may conflict: different citing papers may associate the same paper with different concepts, use different names for similar concepts, or organize the field along different dimensions.}
Semantic clustering provides no hierarchy and can mix abstraction levels; a one-shot LLM offers little control over evidence support;
\yuquan{\textcolor[rgb]{0, 0, 0}{and directly following a single RW Tree can privilege one author's local organization over perspectives from other papers.}}

\yuquan{\tool therefore treats taxonomy construction as aggregation: given \(P\) and \(G\), it seeks an admissible hierarchy that reconciles the author-written judgments encoded in \(G\)'s citation contexts under the minimum-conflict objective, assigns papers to leaves through \(A\), and generates the survey from the result.}
\yuquan{\textcolor[rgb]{0, 0, 0}{The problem is NP-hard, so \tool follows a greedy-and-refine strategy: a greedy initializer on the combinatorial core, followed by global LLM refinement of category meanings and hierarchical structure.}}
\paragraph{\textbf{Candidate concept space.}}
\label{sec:method_taxonomy_candidates}

\tool first constructs a shared candidate concept space from the relation statements carried by the edges of \(G\), so that the greedy initializer begins from relational evidence shared across the retrieved literature. Each directed edge \((p',p)\in E(G)\) stores the author-written passage in which the citing paper \(p'\) states how it categorizes the cited paper \(p\); let \(\phi(p',p)\) denote the set of concept phrases stated on edge \((p',p)\).
The candidate concept set \(\mathcal{C}\) collects the concept phrases on all edges, \(\bigcup_{(p',p)\in E(G)}\phi(p',p)\), and merges semantically equivalent ones using their textual meaning and supporting edges.
Each concept \(c\) is assigned a semantic dimension \(\delta(c)\in\mathcal{D}\), where \(\mathcal{D}\) is the set of organizational dimensions. 
The candidate space retains weighted paper assignments, defined below.
\yuquan{\textcolor[rgb]{0, 0, 0}{For a paper \(p\) and a concept \(c\), the assignment weight \(w(c,p)\) counts incoming citation edges whose passages associate \(p\) with \(c\):}}
\begin{equation*}
w(c,p)=\sum_{(p',p)\in E(G)}\mathbf{1}\bigl[c\in\phi(p',p)\bigr],
\end{equation*}
Conflicts remain explicit: if different edges associate a paper with different concepts, all candidate assignments and their support weights are retained.
A zero-weight fallback concept \(c_{\bot}\) remains available for papers with sparse concept evidence.

\paragraph{\textbf{Optimization objective.}}
\label{sec:method_taxonomy_formulation}

A taxonomy is a rooted tree \(T=(\mathcal{C}(T),E(T))\) over selected concepts \(\mathcal{C}(T)\subseteq\mathcal{C}\), with leaf-category set \(L(T)\subseteq\mathcal{C}(T)\). A map \(A:P\rightarrow L(T)\) assigns every paper to exactly one leaf category. Let \(\Gamma\) denote a configurable set of structural and semantic constraints, and let \(\mathcal{F}(\Gamma)\) be the taxonomy--assignment pairs satisfying them. We solve $(T^*,A^*)$ as follows
\begin{equation}
\label{eq:mctc_objective}
\resizebox{0.95\columnwidth}{!}{$\displaystyle \arg\max_{(T,A)\in\mathcal{F}(\Gamma)}\left[\underbrace{\sum_{c\in\mathcal{C}(T)}\sum_{p\in P}w(c,p)}_{\text{\textcircled{\scriptsize 1}\,concept support}}+\underbrace{\sum_{p\in P}w\bigl(A(p),p\bigr)}_{\text{\textcircled{\scriptsize 2}\,assignment support}}-\underbrace{\lambda\bigl|\mathcal{C}(T)\bigr|}_{\text{\textcircled{\scriptsize 3}\,complexity penalty}}\right].$}
\end{equation}
Term \textcircled{\scriptsize 1} favors concepts carrying strong paper support, term \textcircled{\scriptsize 2} rewards agreement with author judgments at the paper-assignment level, and \textcircled{\scriptsize 3} controls taxonomy complexity. The constraint set \(\Gamma\) can, for example, limit the maximum depth \(B\) and branching factor \(k\).
\yuquan{\textcolor[rgb]{0, 0, 0}{The constrained optimization is NP-hard; Appendix~\ref{app:complexity} specifies the full constraint set and gives the NP-hardness proof.}}

\paragraph{\textbf{Greedy taxonomy construction.}}
\label{sec:method_taxonomy_draft}

Since exact optimization is NP-hard, \tool constructs a draft solution with a greedy top-down procedure over the shared candidate concept space derived from the citation contexts retained in \(G\).

The root initially contains the full paper set \(P\), and the procedure applies the same three-step split recursively at each category. 
\textbf{Candidate preparation:} for the papers currently under a category, it retains the supported categories from the shared candidate space, removes those with insufficient support or that violate the constraints, and selects the organizational dimension with the strongest total support so that the resulting siblings remain coherent. 
\textbf{Greedy selection:} within that dimension, it selects child categories one at a time by marginal coverage, choosing the candidate that adds the most support beyond the children already selected, and stops when the branching budget is reached. 
\textbf{Assignment and recursion:} each paper is assigned to the child with the highest normalized support; low-confidence papers are placed in a fallback category or deferred to the final assignment stage. The procedure then recurses on each child until the depth bound is reached, too few papers remain, or no supported split is available, yielding the draft taxonomy \(T_0\) and assignment \(A_0\). Appendix~\ref{app:complexity} provides the formal analysis.

\paragraph{\textbf{Global refinement and generation.}}
\label{sec:method_taxonomy_refinement}

\yuquan{\textcolor[rgb]{0, 0, 0}{\textbf{Structural refinement:} The greedy draft may contain vague, redundant, or unevenly granular categories. An LLM refines \(T_0\) using the scope definition \(d\), compact profiles of its concepts, profiles of the retrieved papers, and recurring category and parent--child patterns aggregated from the RW Trees as auxiliary structural evidence. The draft remains the evidence-grounded starting point; the refinement resolves semantic and structural issues before final paper assignment and deterministic validation.}}

\textbf{Paper assignment refinement:} Once the refined taxonomy \(T\) is fixed, an LLM updates the draft assignment \(A_0\) only where hierarchy changes create conflicts, ensuring that every paper belongs to exactly one leaf and yielding the final assignment \(A\).
A deterministic validation pass then repairs any remaining missing or duplicate assignments and prunes unsupported empty leaves. Together, hierarchy refinement, assignment refinement, and validation transform the draft \((T_0,A_0)\) into the final \((T,A)\) used in our experiments.

Finally, \tool uses \(T\) and \(A\) as the writing plan. 
Each category becomes a writing unit with a local evidence package and a fixed set of allowed citations. 
Leaf units synthesize the papers assigned by \(A\), while internal units explain the relations among their child categories. 
The units are generated independently, assembled in taxonomy order, and completed with survey-level sections. 
Thus, \(T\) determines the survey organization and \(A\) bounds the evidence available to each section.

\section{Experiments}
\label{sec:experiments}

We evaluate \tool from research-scope reconstruction to final survey generation. We ask three research questions: 
\textbf{RQ1}: Does Related Work-guided retrieval reconstruct a more complete and accurate research scope? 
\textbf{RQ2}: Does Related Work evidence improve taxonomy construction? 
\textbf{RQ3}: Does the resulting taxonomy improve the quality and credibility of the final survey? 
Formal definitions and implementation details for all evaluation metrics are provided in Appendix~\ref{app:metric_definitions}.

\subsection{Research-Scope Reconstruction}
\label{sec:exp_rq1}

\yuquan{\textcolor[rgb]{0, 0, 0}{Across ten computer-science topics, RQ1 evaluates retrieved paper sets and academic profiles against gold scopes derived from human-survey references and the corresponding topic definitions.}}

\subsubsection{Scope Coverage and Correctness}
\label{sec:exp_rq11}


\textbf{Setting.}
We compare \tool with BM25 and dense retrieval under online and offline access, and with SurveyGen, SurveyForge, and AutoSurvey. We report \#Papers, Coverage, Correctness, and F1; Correctness is determined by an LLM judge given the topic definition and each retrieved paper. More details are provided in Appendix~\ref{app:rq1_metrics}.
\yuquan{\textcolor[rgb]{0, 0, 0}{The initial seeds are listed in Appendix~\ref{app:algorithm_retrieval}.}}
\begin{table}[!htbp]
\centering
\small
\renewcommand{\arraystretch}{0.94}
\setlength{\tabcolsep}{3.5pt}
\resizebox{\columnwidth}{!}{%
\begin{tabular}{llrccc}
\toprule
\textbf{Setting} & \textbf{Method} & \textbf{\#Papers} & \textbf{Coverage} & \textbf{Correctness} & \textbf{F1} \\
\midrule
Online & BM25            & 1,995 & \underline{0.1058} & 0.1014 & 0.1036 \\
Online & Dense Retrieval & 993 & 0.1027 & 0.1717 & 0.1285 \\
\midrule
Offline & BM25            & 500 & 0.0675 & 0.1578 & 0.0946 \\
Offline & Dense Retrieval & 500 & 0.0765 & 0.1900 & 0.1091 \\
\midrule
ASG & SurveyGen   & 246 & 0.0951 & \underline{0.4911} & \underline{0.1593} \\
ASG & SurveyForge & 1,063 & 0.1050 & 0.1401 & 0.1200 \\
ASG & AutoSurvey  & 2,063 & 0.0592 & 0.0586 & 0.0589 \\
\midrule
Ours & \toole & 1,879 & \textbf{0.4097} & \textbf{0.9521} & \textbf{0.5729} \\
\bottomrule
\end{tabular}
}
\caption{\yuquan{\textcolor[rgb]{0, 0, 0}{Research-scope reconstruction across retrieval settings and ASG baselines. \#Papers totals retrieved papers over ten topics; bold and underlined values mark the best and second-best results.}}}
\label{tab:scope_results}
\end{table}

\textbf{Results.}
\yuquan{\textcolor[rgb]{0, 0, 0}{Table~\ref{tab:scope_results} shows that \tool leads in Coverage, Correctness, and F1, expanding the gold scope while maintaining high Correctness over a large retrieved set. This supports author-written Related Work relations as a signal for finding foundational and closely related work.}}
\yuquan{We further tested alternative in-scope seeds on two topics. Correctness remained high, while Coverage and F1 varied with the seed (Appendix~\ref{app:seed_sensitivity}).}

\subsubsection{Academic Profile of the Scope}
\label{sec:exp_rq12}


\textbf{Setting.}
\yuquan{\textcolor[rgb]{0, 0, 0}{We compare retrieved papers with the gold scope using two-sample Kolmogorov--Smirnov tests on Venue H-Index and Year Gap; higher H-Index reflects venue impact, while smaller Year Gap indicates newer work.}}

\begin{table*}[!t]
\centering
\scriptsize
\renewcommand{\arraystretch}{0.94}
\setlength{\tabcolsep}{3pt}
\resizebox{\textwidth}{!}{%
\begin{tabular}{lcccccccccc}
\toprule
& \multicolumn{3}{c}{\textbf{Structure}} & \multicolumn{5}{c}{\textbf{Outline Quality}} & \multicolumn{2}{c}{\textbf{Taxonomy Alignment}} \\
\cmidrule(lr){2-4}\cmidrule(lr){5-9}\cmidrule(lr){10-11}
\textbf{Method} & \textbf{HSR}$\uparrow$ & \textbf{CEDS}$\uparrow$ & \textbf{Node Ratio}$\rightarrow 1$ & \textbf{TU}$\uparrow$ & \textbf{SB}$\uparrow$ & \textbf{HC}$\uparrow$ & \textbf{LO}$\uparrow$ & \textbf{Avg.}$\uparrow$ & \textbf{RWTA}$\uparrow$ & \textbf{PTA}$\uparrow$ \\
\midrule
Vanilla LLM & 0.730126 & 0.448902 & 1.346384 & \yuquan{3.4350} & \yuquan{3.2550} & \yuquan{3.0500} & \yuquan{3.2150} & \yuquan{3.2388} & \underline{0.781515} & \textbf{0.753478} \\
SurveyGen & \textbf{0.754168} & \textbf{0.526661} & \textbf{1.176255} & \yuquan{3.3000} & \yuquan{\underline{3.5500}} & \yuquan{\underline{3.4000}} & \yuquan{\underline{3.4250}} & \yuquan{\underline{3.4188}} & 0.743593 & 0.446429 \\
SurveyForge & 0.712712 & \underline{0.478629} & \underline{1.309190} & \yuquan{2.7750} & \yuquan{2.6250} & \yuquan{2.4750} & \yuquan{2.5250} & \yuquan{2.6000} & 0.705402 & 0.129679 \\
AutoSurvey & 0.738040 & 0.374903 & 2.290563 & \yuquan{2.6850} & \yuquan{2.5200} & \yuquan{2.4050} & \yuquan{2.3750} & \yuquan{2.4962} & 0.731652 & 0.160142 \\
Deep Research & 0.688150 & 0.382977 & 0.348203 & \yuquan{\textbf{3.5000}} & \yuquan{2.8750} & \yuquan{3.3750} & \yuquan{3.1250} & \yuquan{3.2188} & 0.751199 & 0.403720 \\
\toole & \underline{0.742160} & 0.408899 & 2.289222 & \yuquan{\underline{3.4750}} & \yuquan{\textbf{3.7500}} & \yuquan{\textbf{3.5000}} & \yuquan{\textbf{3.6750}} & \yuquan{\textbf{3.6000}} & \textbf{0.797152} & \underline{0.747758} \\
\bottomrule
\end{tabular}
}
\caption{Taxonomy evaluation. TU, SB, HC, and LO denote Topic Uniqueness, Structural Balance, Hierarchical Clarity, and Logical Organization. RWTA denotes RW--Taxonomy Alignment, and PTA denotes Paper--Taxonomy Assignment Accuracy. Bold values indicate the best result, and underlined values indicate the second-best result.}
\label{tab:taxonomy_quality}
\vspace{-8pt}
\end{table*}

\begin{figure}[!t]
\centering

\includegraphics[width=\columnwidth]{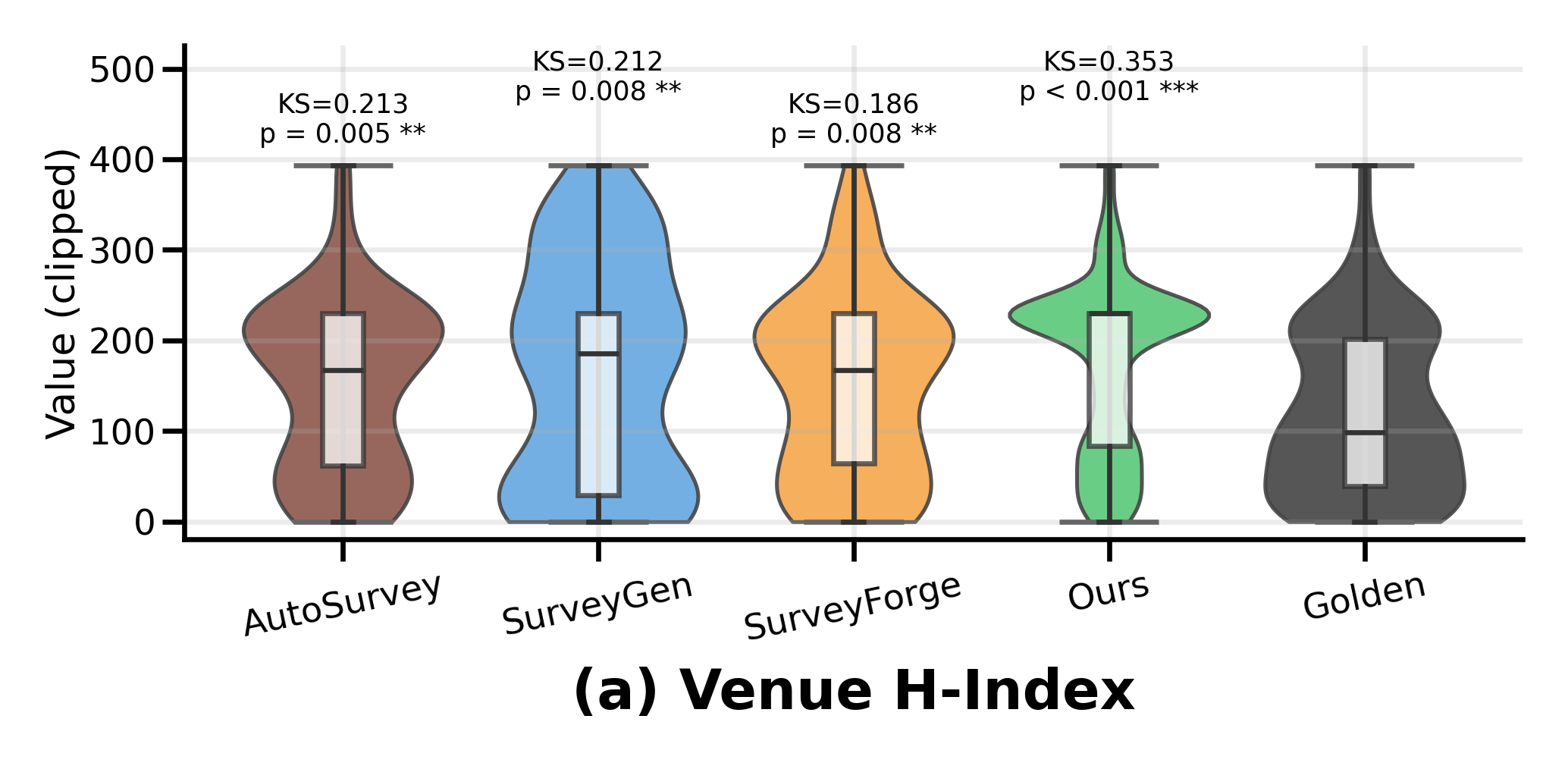}
\includegraphics[width=\columnwidth]{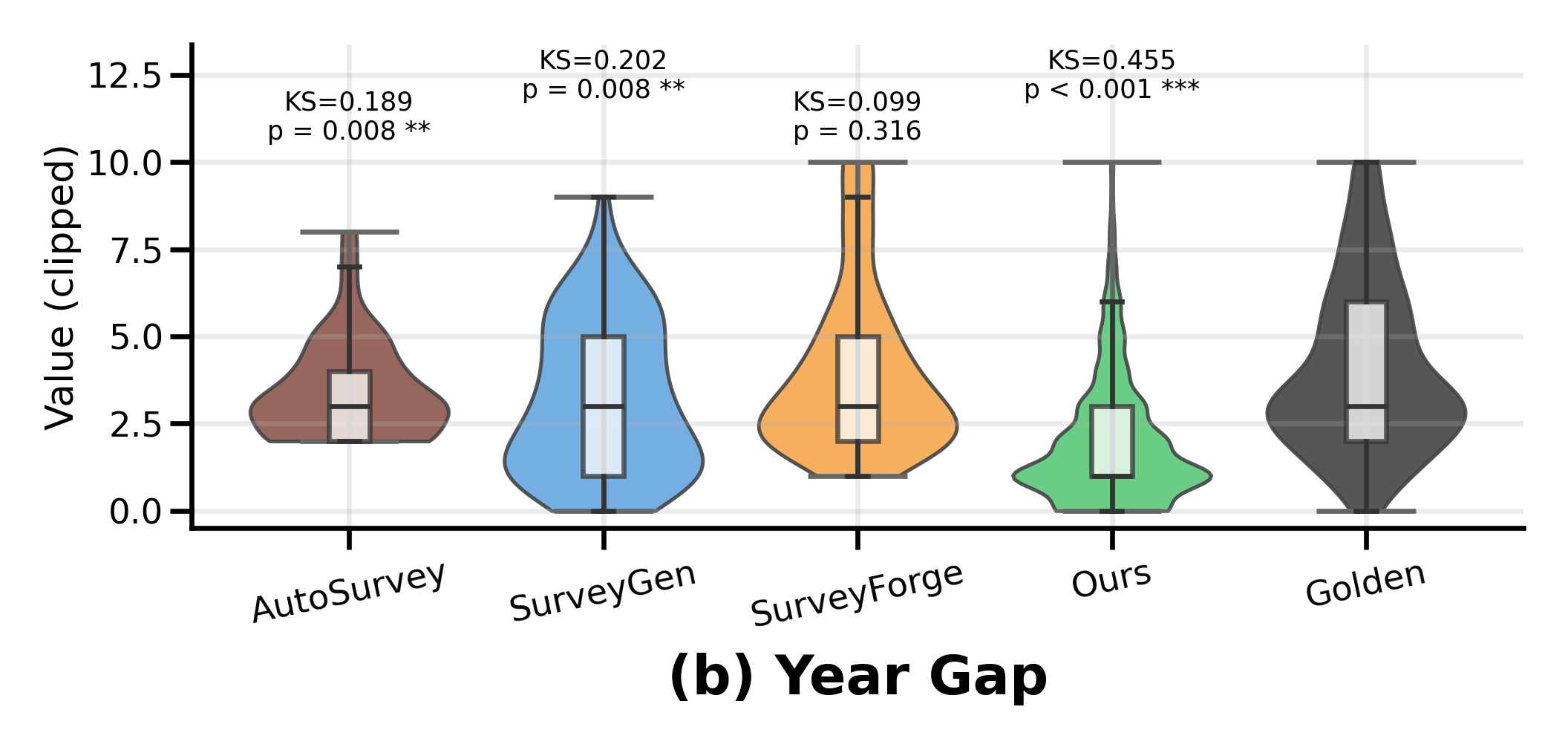}
\caption{Research-scope profile distributions for papers retrieved by each method and papers in the golden scope: (a) Venue H-Index and (b) Year Gap. Each violin aggregates papers across all topics; the inner box indicates the median and interquartile range.}
\label{fig:scope_profile}
\vspace{-4pt}
\end{figure}

\textbf{Results.}
\yuquan{\textcolor[rgb]{0, 0, 0}{Figure~\ref{fig:scope_profile} shows that \tool covers more recent work without shifting the retrieved scope toward lower-impact venues, suggesting that the added papers retain a reasonable academic profile in the pooled ten-topic comparison.}}

\textbf{Expert validation.}
\yuquan{\textcolor[rgb]{0, 0, 0}{Experts judged 90 of 100 sampled papers uniquely retrieved by \tool but absent from the human-curated scope to be valid related work, supporting the LLM-based Correctness measure (Appendix~\ref{app:scope_audit}).}}

\subsection{Taxonomy Construction}
\label{sec:exp_rq2}

RQ2 evaluates whether Related Work evidence produces a taxonomy that is complete, coherent, and grounded in the retrieved papers.

\textbf{Setting.}
\yuquan{\textcolor[rgb]{0, 0, 0}{Using the same ten topics, we take each human-written survey's section hierarchy as the reference taxonomy and compare six methods. We evaluate similarity to this taxonomy in structure and granularity, outline quality, alignment with author-written Related Work evidence, and paper-assignment accuracy. Metric definitions and score aggregation are provided in Appendix~\ref{app:rq2_metrics}.}}

\textbf{Results.}
\yuquan{\textcolor[rgb]{0, 0, 0}{Table~\ref{tab:taxonomy_quality} shows that SurveyGen most closely matches the human-survey structure and granularity. \tool achieves the highest overall outline-quality score, leading in Structural Balance, Hierarchical Clarity, and Logical Organization; Deep Research leads slightly in Topic Uniqueness. \tool also has the strongest RWTA and remains competitive in paper-assignment accuracy. Further taxonomy diagnostics appear in Appendix~\ref{app:taxonomy_diagnostics}.}}
\yuquan{\textcolor[rgb]{0, 0, 0}{A method-blinded expert comparison on two topics favors \toole{}'s taxonomy in 5/6 judgments (83.3\%); the protocol and interview findings appear in Appendix~\ref{app:expert_comparison}.}}

\begin{table*}[!t]
\centering
\small
\renewcommand{\arraystretch}{0.94}
\setlength{\tabcolsep}{3pt}
\resizebox{0.85\textwidth}{!}{%
\begin{tabular}{lcccccc}
\toprule
& \multicolumn{2}{c}{\textbf{Content Correspondence}} & \multicolumn{2}{c}{\textbf{Citation Support}} & \multicolumn{2}{c}{\textbf{Reference Reliability}} \\
\cmidrule(lr){2-3}\cmidrule(lr){4-5}\cmidrule(lr){6-7}
\textbf{Method} & \textbf{SS}$\uparrow$ & \textbf{KPR}$\uparrow$ & \textbf{Cit. Recall}$\uparrow$ & \textbf{Cit. Precision}$\uparrow$ & \textbf{\#References} & \textbf{VRR}$\uparrow$ \\
\midrule
Vanilla LLM & \textbf{0.705198} & \yuquan{53} & 0.441418 & \underline{0.406859} & 254 & 0.550306 \\
SurveyGen & 0.675570 & \yuquan{52} & \underline{0.475770} & 0.355161 & 249 & 0.972690 \\
SurveyForge & 0.626670 & \yuquan{55} & 0.178637 & 0.122507 & 1,063 & \underline{0.979134} \\
AutoSurvey & 0.619481 & \yuquan{\textbf{63}} & 0.275886 & 0.255565 & 2,102 & 0.950845 \\
Deep Research & \underline{0.704412} & \yuquan{54} & 0.301514 & 0.264325 & 134 & 0.898381 \\
\toole & 0.670218 & \underline{60} & \textbf{0.856393} & \textbf{0.629516} & 1,717 & \textbf{0.995528} \\
\bottomrule
\end{tabular}
}
\vspace{-4pt}
\caption{\yuquan{\textcolor[rgb]{0, 0, 0}{Content and citation evaluation of generated surveys. SS, KPR, and VRR denote Semantic Similarity, Key-Point Recall, and Verifiable Reference Rate. Bold marks the best result; underlining marks the second-best.}}}
\label{tab:survey_faithfulness}
\vspace{-6pt}
\end{table*}

\begin{table*}[!t]
\centering
\small
\renewcommand{\arraystretch}{0.94}
\setlength{\tabcolsep}{3pt}
\resizebox{\textwidth}{!}{%
\begin{tabular}{lcccccccccc}
\toprule
& \multicolumn{3}{c}{\textbf{Structure}} & \multicolumn{5}{c}{\textbf{Outline Quality}} & \multicolumn{2}{c}{\textbf{Taxonomy Alignment}} \\
\cmidrule(lr){2-4}\cmidrule(lr){5-9}\cmidrule(lr){10-11}
\textbf{Variant} & \textbf{HSR}$\uparrow$ & \textbf{CEDS}$\uparrow$ & \textbf{Node Ratio}$\rightarrow 1$ & \textbf{TU}$\uparrow$ & \textbf{SB}$\uparrow$ & \textbf{HC}$\uparrow$ & \textbf{LO}$\uparrow$ & \textbf{Avg.}$\uparrow$ & \textbf{RWTA}$\uparrow$ & \textbf{PTA}$\uparrow$ \\
\midrule
w/o taxonomy refinement       & \textbf{0.750085} & 0.299391 & 3.081808 & \yuquan{0.6200} & \yuquan{0.5650} & \yuquan{0.3500} & \yuquan{0.3900} & \yuquan{0.4813} & 0.733383 & 0.257588 \\
\toole{} (full)    & 0.742160 & \yuquan{\textbf{0.408899}} & \yuquan{\textbf{2.289222}} & \yuquan{\textbf{3.4750}} & \yuquan{\textbf{3.7500}} & \yuquan{\textbf{3.5000}} & \yuquan{\textbf{3.6750}} & \yuquan{\textbf{3.6000}} & \textbf{0.797152} & \textbf{0.747758} \\
\bottomrule
\end{tabular}%
}
\vspace{-4pt}
\caption{\yuquan{\textcolor[rgb]{0, 0, 0}{Ablation of taxonomy refinement in \toole. The full pipeline is compared with a version without refinement. Metrics follow Table~\ref{tab:taxonomy_quality}. Bold marks the best result; Node Ratio favors values closer to 1.}}}
\label{tab:taxonomy_ablation}
\vspace{-16pt}
\end{table*}

\subsection{Taxonomy-Guided Survey Generation}
\label{sec:exp_rq3}

RQ3 evaluates the content, correspondence, and credibility of the final survey. All methods generate a survey for each of the same ten topics.

\subsubsection{LLM-Judged Survey Quality}
\label{sec:exp_rq31}


\textbf{Setting.}
\yuquan{\textcolor[rgb]{0, 0, 0}{GPT-5.2, Gemini 3.1 Pro, and DeepSeek V4 Pro score each method's surveys for Coverage, Relevance, and Structure on a five-point scale. Averaging the three dimensions yields 30 topic--judge scores per method; aggregation and full results appear in Appendices~\ref{app:rq3_metrics} and~\ref{app:judge_results}.}}

\begin{figure}[!t]
\centering
\includegraphics[width=1\columnwidth,trim=0 60 0 170,clip]{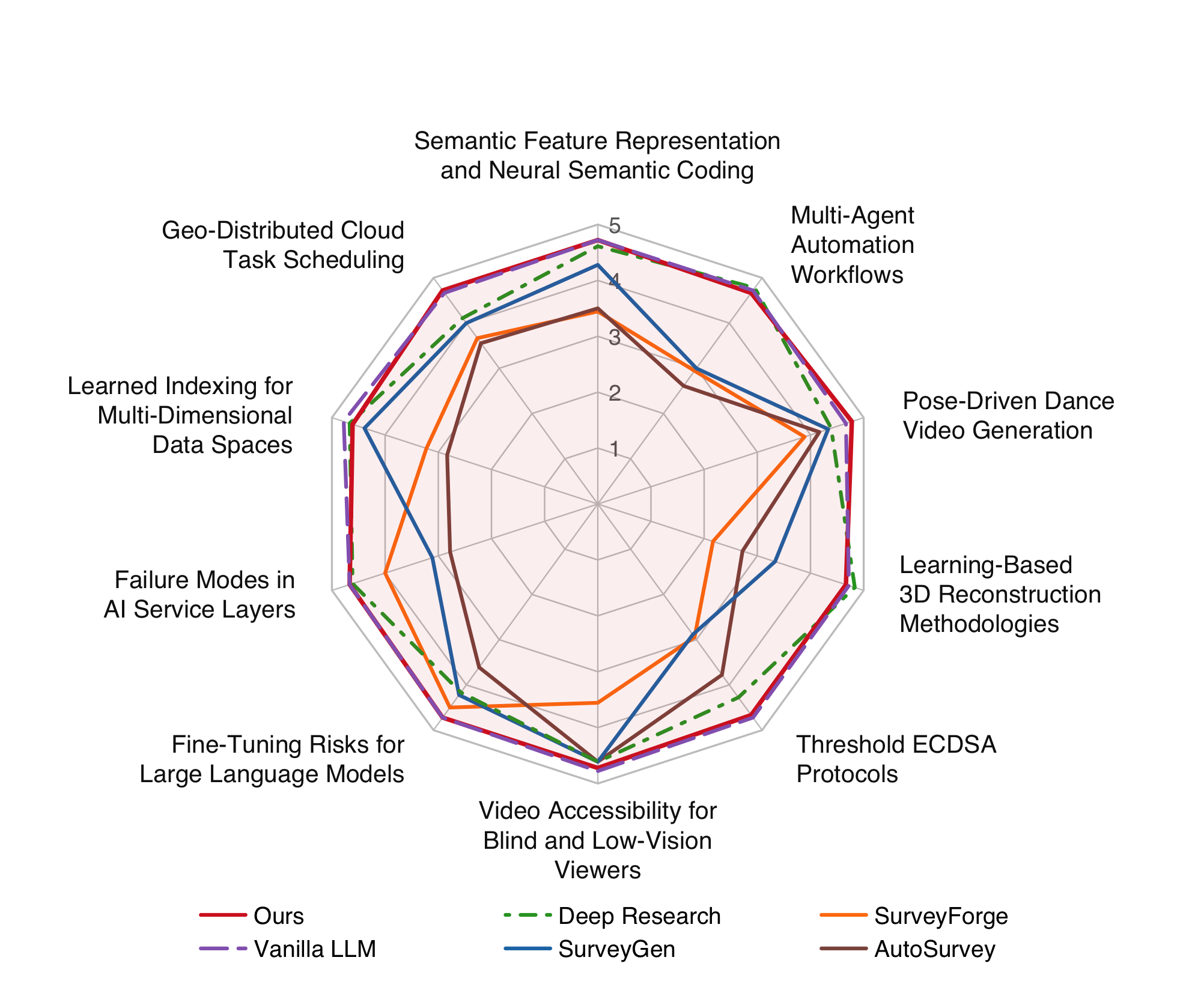}
\caption{\yuquan{\textcolor[rgb]{0, 0, 0}{Survey quality by topic. Axes average Coverage, Relevance, and Structure across three LLM judges.}}}
\vspace{-6pt}
\label{fig:q3_topic_quality}
\end{figure}

\textbf{Results.}
\yuquan{Figure~\ref{fig:q3_topic_quality} shows that \tool maintains high survey-quality scores across topics, outperforming SurveyGen, SurveyForge, AutoSurvey, and Deep Research on average while remaining comparable to Vanilla LLM (Appendix~\ref{app:judge_results}).}
\yuquan{In the same expert study, InsightRAG's survey was preferred in all 6/6 comparisons; semi-structured interviews further examined the reasons for these preferences and remaining weaknesses (Appendix~\ref{app:expert_comparison}).}

\vspace{-4pt}
\subsubsection{\yuquan{\textcolor[rgb]{0, 0, 0}{Content Correspondence and Citation Support of Generated Surveys}}}
\label{sec:exp_rq32}


\textbf{Setting.}
\yuquan{\textcolor[rgb]{0, 0, 0}{SS compares generated and human-survey passages; KPR measures coverage of the human survey's key points. Citation Recall counts claims supported by their full citation group, whereas Citation Precision counts citation instances that independently or necessarily support their claims. We also report reference count and VRR, the fraction resolved to real papers (Appendix~\ref{app:rq3_metrics}).}}

\textbf{Results.}
\yuquan{Table~\ref{tab:survey_faithfulness} shows that \tool ranks second in KPR and leads in Citation Recall, Citation Precision, and VRR. Vanilla LLM achieves higher SS, but substantially lower citation-support and reference-reliability scores.}


\vspace{-4pt}
\subsubsection{\yuquan{Passage-Based Citation Faithfulness}}
\label{sec:exp_passage_faithfulness}

\yuquan{\textcolor[rgb]{0, 0, 0}{\textbf{Setting.} On Multi-Agent Workflow Automation and Video Accessibility, we sample seven claim--citation pairs per survey for the six methods in Table~\ref{tab:survey_faithfulness} (84 total). An automatic judge reads cited full texts and identifies passages supporting each claim. Unverifiable references count as Unsupported; Supported includes full or partial support. We manually audit two pairs per survey (24 total).}}

\yuquan{\textcolor[rgb]{0, 0, 0}{\textbf{Results.} \tool has the highest automatic Supported rate (13/14, 92.9\%) and ties Vanilla LLM in the human audit (4/4). Automatic and human judgments agree on 20/24 pairs (83.3\%, $\kappa=0.655$); the automatic judge is more lenient in all four disagreements. Complete automatic and human results appear in Appendix~\ref{app:passage_faithfulness}.}}

\vspace{-2.4pt}
\subsection{Ablation Study}
\label{sec:exp_ablation}

\textbf{Setting.}
\yuquan{\textcolor[rgb]{0, 0, 0}{We independently ablate three scope-reconstruction components and taxonomy refinement, holding other settings fixed; full scope ablation results appear in Appendix~\ref{app:scope_ablation}.}}


\textbf{Scope reconstruction.}
\yuquan{\textcolor[rgb]{0, 0, 0}{Forward expansion contributes most, followed by backward expansion; loop closure and densification add smaller but consistent gains in Coverage and F1.}}

\textbf{Taxonomy construction.}
\yuquan{\textcolor[rgb]{0, 0, 0}{Table~\ref{tab:taxonomy_ablation} shows that refinement improves OutlineAvg (0.4813 to 3.6000) and PTA (0.2576 to 0.7478), despite slightly lower HSR. Heading correspondence alone does not ensure coherent organization or accurate assignment of retrieved papers to taxonomy leaves.}}

\vspace{-2.4pt}
\section{Conclusion}

We built \tool on a signal automatic survey generation had left unused: the judgments human authors record in Related Work sections about which prior work matters and how it relates.
\yuquan{Insight-guided retrieval follows only the relations an author's own words support, and minimum-conflict aggregation organizes the resulting citation-context evidence into a taxonomy that bounds each section's evidence, improving scope reconstruction, taxonomy quality, citation support, and reference reliability across ten topics.}
\yuquan{\textcolor[rgb]{0, 0, 0}{More broadly, this work provides a concrete case for reconsidering what human insight is already embedded in existing artifacts and how it can be leveraged to support literature discovery, organization, and synthesis.}}

\clearpage
\section*{Limitations}

\paragraph{Dependence on language models.} \tool relies on LLMs for RW Tree extraction, scope verification, taxonomy refinement, and section writing, while part of the evaluation uses LLM judges. Errors may propagate through the pipeline, so the generated survey and intermediate artifacts still require human verification.

\paragraph{\yuquan{Scope of author perspectives.}} \yuquan{\textcolor[rgb]{0, 0, 0}{\tool draws on author-written Related Work sections to organize and synthesize the retrieved literature. Shared omissions or biases in these sections may carry over into the resulting taxonomy and survey. Because retrieval is limited to the literature reached through citation traversal, some parts of the field may remain underrepresented.}}

\paragraph{\yuquan{\textcolor[rgb]{0, 0, 0}{Coverage of recent work.}}} \yuquan{\textcolor[rgb]{0, 0, 0}{Citation traversal can miss recent papers that are not yet indexed or connected to the explored graph; when forward candidate pools exceed the processing budget, prioritizing by citation impact and recency may also miss valuable low-citation work.}}

\paragraph{Risks of over-reliance.} Junior researchers may over-rely on automatically generated surveys instead of developing the skills to search, assess, and synthesize the literature independently. 
This may limit their ability to form their own judgments about a field's scope and structure. We therefore recommend human--AI co-creation: the generated scope, taxonomy, and text should serve as starting points that researchers inspect, revise, and validate.

\paragraph{API cost.} The cold-start stage requires many LLM calls to extract RW Trees and graph relations and to judge whether papers fall within the target scope. Once processed, the extracted structures and relation evidence can be stored and reused, while relevance decisions can be reused for the same scope. Subsequent runs therefore avoid most repeated extraction and verification costs.

\section*{Ethics Statement}

AI assistants were used in preparing this work, in two limited roles. 
For the paper itself, their use was restricted to language polishing and writing assistance;
AI assistants were also used to write part of the implementation code, which the authors reviewed, tested, and take full responsibility for.


\expandafter\def\csname yuquannewbib@Van_Daele_2024\endcsname{}
\expandafter\def\csname yuquannewbib@Xu_2025\endcsname{}
\expandafter\def\csname yuquannewbib@cohan2019scicite\endcsname{}
\expandafter\def\csname yuquannewbib@gao2025flowreasonerreinforcingquerylevelmetaagents\endcsname{}
\expandafter\def\csname yuquannewbib@han2023autoadmoviedescriptioncontext\endcsname{}
\expandafter\def\csname yuquannewbib@hsu2024chime\endcsname{}
\expandafter\def\csname yuquannewbib@hu2025higtl\endcsname{}
\expandafter\def\csname yuquannewbib@mohammad2009using\endcsname{}
\expandafter\def\csname yuquannewbib@palani2023relatedly\endcsname{}
\expandafter\def\csname yuquannewbib@valenzuela2015meaningful\endcsname{}
\expandafter\def\csname yuquannewbib@wang2025evoagentxautomatedframeworkevolving\endcsname{}
\expandafter\def\csname yuquannewbib@zhang2025multiagentarchitecturesearchagentic\endcsname{}
\expandafter\def\csname yuquannewbib@zhu2025context\endcsname{}
\begingroup
\let\yuquanoriginalbibitem\bibitem
\let\bibitem\yuquanbibliographyitem
\bibliography{custom}
\endgroup

\appendix

\section{\yuquan{\textcolor[rgb]{0, 0, 0}{Evaluation Dataset and Baseline Configurations}}}
\label{app:evaluation-dataset}

\begin{table*}[t]
\centering
\footnotesize
\setlength{\tabcolsep}{3.5pt}
\begin{tabular}{c p{0.42\textwidth} p{0.43\textwidth} r}
\hline
ID & CCF Research Area & Evaluation Topic and Source Survey & RW Trees \\
\hline
0 &
Interdisciplinary, Comprehensive, and Emerging Areas &
Multi-Agent Workflow Automation~\citep{survey_multi_agent_workflow} &
25 \\

1 &
Human-Computer Interaction and Ubiquitous Computing &
AI-Powered Video Accessibility for Blind and Low Vision Viewers~\citep{survey_video_accessibility} &
26 \\

2 &
Artificial Intelligence &
Pose-Driven Dance Video Generation~\citep{survey_human_motion_video} &
22 \\

3 &
Computer Graphics and Multimedia &
Learning-Based 3D Reconstruction Methodologies~\citep{survey_3d_reconstruction} &
28 \\

4 &
Theoretical Computer Science &
Threshold ECDSA Protocols~\citep{survey_threshold_ecdsa} &
20 \\

5 &
Network and Information Security &
Risks and Countermeasures of Fine-Tuning Large Language Models~\citep{survey_llm_finetuning_security} &
22 \\

6 &
Software Engineering, System Software, and Programming Languages &
Failure Modes in AI Service Layers~\citep{survey_ai_service_failures} &
20 \\

7 &
Databases, Data Mining, and Content Retrieval &
Learned Indexing for Multi-Dimensional Data Spaces~\citep{survey_multidimensional_learned_index} &
22 \\

8 &
Computer Networks &
Semantic Feature Representation and Neural Semantic Coding~\citep{survey_semantic_communications} &
25 \\

9 &
Computer Architecture, Parallel and Distributed Computing, and
Storage Systems &
Geo-Distributed Cloud Task Scheduling for Higher Compute Utilization~\citep{survey_geodistributed_scheduling} &
38 \\
\hline
\multicolumn{3}{r}{Total} & 248 \\
\hline
\end{tabular}
\caption{
Evaluation topics, their source surveys, and the realized numbers of Related Work (RW) Trees.
}
\label{app-dataset}
\end{table*}

We constructed ten evaluation instances covering the ten research areas used in the China Computer Federation (CCF) venue classification~\citep{ccf2022} (Table~\ref{app-dataset}).
For each research area, we selected a human-written source survey and used either the complete survey or a focused section with sufficient coverage of the target topic.

For instances based on a survey section, the LLM (GPT-5.5) received both the source-survey title and the selected section title; for instances based on the complete survey, it received only the survey title.
Based on the given titles, the LLM generated (1) an evaluation topic title and (2) a focused scope definition describing the research problem, included directions, and exclusion boundaries.
We manually reviewed the generated topic titles and scope definitions for consistency with the selected survey content.
The resulting scope definitions were used consistently for paper selection, citation filtering, and scope verification throughout the experiments.

We extracted the scientific papers cited in each selected survey unit, where the unit was either the complete survey or a selected section.
We then resolved their bibliographic records, removed duplicate versions and non-paper references, and retained only papers that matched the generated scope definition, yielding 304 in-scope candidate papers.

For each retained candidate paper, we identified its author-written Related Work discussion, resolved in-text citation markers against its bibliography, and filtered the resolved cited papers using the same topic-specific scope definition. The remaining citation contexts and cited papers were organized into a hierarchical RW Tree following the definition in Section~\ref{sec:method_rwt}.

After topic-scope filtering and extraction validation, the final dataset contains 248 valid RW Trees.
The resulting trees contain 3,937 citation-bearing leaf occurrences.
Because the same cited paper may occur in multiple RW Trees, these occurrences correspond to 2,611 unique cited papers.

\subsection{\yuquan{\textcolor[rgb]{0, 0, 0}{Baseline Configurations}}}
\label{app:baseline_settings}

\yuquan{\textbf{Setting.} We distinguish controlled retrieval baselines from automatic survey-generation (ASG) systems run in their native settings. BM25 and Dense use the same candidate source within each Offline or Online setting and follow the same RQ1 evaluation protocol; their retrieval cutoffs are reported in Table~\ref{tab:baseline_settings}.}

\begin{table*}[t]
\centering
\small
\setlength{\tabcolsep}{5pt}
\renewcommand{\arraystretch}{1.08}
\begin{tabular}{p{0.22\textwidth}p{0.33\textwidth}p{0.20\textwidth}r}
\toprule
\yuquan{\textbf{Method}} & \yuquan{\textbf{Candidate source}} & \yuquan{\textbf{Retrieval setting}} & \yuquan{\textbf{\#Papers}} \\
\midrule
\yuquan{BM25--Offline} & \yuquan{Frozen corpus (1,064 papers)} & \yuquan{Top-50/topic} & \yuquan{500} \\
\yuquan{Dense--Offline} & \yuquan{Same frozen corpus} & \yuquan{Top-50/topic} & \yuquan{500} \\
\yuquan{BM25--Online} & \yuquan{Semantic Scholar candidate pools} & \yuquan{Top-200/topic} & \yuquan{1,995} \\
\yuquan{Dense--Online} & \yuquan{Semantic Scholar candidate pools} & \yuquan{Top-100/topic} & \yuquan{993} \\
\yuquan{SurveyGen} & \yuquan{Native references} & \yuquan{Native} & \yuquan{246} \\
\yuquan{SurveyForge} & \yuquan{Native references} & \yuquan{Native} & \yuquan{1,063} \\
\yuquan{AutoSurvey} & \yuquan{Native references} & \yuquan{Native} & \yuquan{2,063} \\
\yuquan{\toole} & \yuquan{Citation traversal} & \yuquan{Paper budget 400} & \yuquan{1,879} \\
\bottomrule
\end{tabular}
\caption{\yuquan{\textcolor[rgb]{0, 0, 0}{Retrieval configurations for the RQ1 comparison. \#Papers totals retrieved papers across the ten topics.} BM25 and Dense share a candidate source within each Offline or Online setting, but use different Top-$k$ cutoffs Online. Dense retrieval uses BAAI/bge-base-en-v1.5 for ranking.}}
\label{tab:baseline_settings}
\end{table*}

\yuquan{\textbf{Comparison scope.} BM25 and Dense provide controlled retrieval comparisons. SurveyGen, SurveyForge, and AutoSurvey are compared using their native system outputs, not as budget-matched component baselines.}

\section{\yuquan{\textcolor[rgb]{0, 0, 0}{Additional Retrieval Analyses}}}
\subsection{\yuquan{\textcolor[rgb]{0, 0, 0}{Scope-Reconstruction Ablation}}}
\label{app:scope_ablation}

Table~\ref{tab:scope_ablation} reports the complete scope-reconstruction
ablation results. Each variant removes one component while retaining the rest
of the retrieval pipeline.

\begin{table}[t]
\centering
\footnotesize
\setlength{\tabcolsep}{2.5pt}
\resizebox{\columnwidth}{!}{%
\begin{tabular}{lccc}
\toprule
\textbf{Variant} & \textbf{Coverage}$\uparrow$ &
\textbf{Correctness}$\uparrow$ & \textbf{F1}$\uparrow$ \\
\midrule
w/o Forward Expansion             & 0.0871 & 0.9507 & 0.1595 \\
w/o Backward Expansion            & 0.2210 & \textbf{0.9755} & 0.3603 \\
w/o Loop Closure \& Densification & 0.3927 & 0.9577 & 0.5570 \\
\toole{} (Full)        & \textbf{0.4097} & 0.9521 & \textbf{0.5729} \\
\bottomrule
\end{tabular}%
}
\caption{\yuquan{\textcolor[rgb]{0, 0, 0}{Ablation of scope-reconstruction components in \toole.}} Bold indicates the best
result in each column.}
\label{tab:scope_ablation}
\end{table}

\subsection{\yuquan{\textcolor[rgb]{0, 0, 0}{Seed Sensitivity}}}
\label{app:seed_sensitivity}

\paragraph{\yuquan{\textbf{Setting.}}}
\yuquan{For each of two topics, we asked an LLM to suggest two alternative in-scope seed papers based on the topic. We reran retrieval from the original and alternative seeds, keeping the topic definition, gold scope, and Q1 evaluation protocol fixed. Table~\ref{tab:seed_sensitivity} reports the results for each seed.}

\begin{table*}[t]
\centering
\small
\setlength{\tabcolsep}{3pt}
\renewcommand{\arraystretch}{1.05}
\resizebox{\textwidth}{!}{%
\begin{tabular}{lp{9.0cm}rccc}
\toprule
\yuquan{\textbf{Role}} & \yuquan{\textbf{Seed paper}} & \yuquan{\textbf{\#Papers}} & \yuquan{\textbf{Coverage}} & \yuquan{\textbf{Correctness}} & \yuquan{\textbf{F1}} \\
\midrule
\multicolumn{6}{l}{\yuquan{Multi-Agent Workflow Automation~\citep{survey_multi_agent_workflow}}} \\
\yuquan{Original} & \yuquan{Multi-agent Architecture Search via Agentic Supernet~\citep{zhang2025multiagentarchitecturesearchagentic}} & \yuquan{69} & \yuquan{0.6400} & \yuquan{0.8841} & \yuquan{0.7425} \\
\yuquan{Alternative} & \yuquan{FlowReasoner: Reinforcing Query-Level Meta-Agents~\citep{gao2025flowreasonerreinforcingquerylevelmetaagents}} & \yuquan{105} & \yuquan{0.7200} & \yuquan{0.9238} & \yuquan{0.8093} \\
\yuquan{Alternative} & \yuquan{EvoAgentX: An Automated Framework for Evolving Agentic Workflows~\citep{wang2025evoagentxautomatedframeworkevolving}} & \yuquan{122} & \yuquan{0.8400} & \yuquan{0.8361} & \yuquan{0.8380} \\
\midrule
\multicolumn{6}{l}{\yuquan{AI-Powered Video Accessibility for Blind and Low Vision Viewers~\citep{survey_video_accessibility}}} \\
\yuquan{Original} & \yuquan{Making Short-Form Videos Accessible with Hierarchical Video Summaries~\citep{Van_Daele_2024}} & \yuquan{34} & \yuquan{0.6154} & \yuquan{0.9706} & \yuquan{0.7532} \\
\yuquan{Alternative} & \yuquan{AutoAD: Movie Description in Context~\citep{han2023autoadmoviedescriptioncontext}}
& \yuquan{46} & \yuquan{0.4615} & \yuquan{0.9783} & \yuquan{0.6272} \\
\yuquan{Alternative} & \yuquan{Branch Explorer: Leveraging Branching Narratives to Support Interactive 360$^\circ$ Video Viewing for Blind and Low Vision Users~\citep{Xu_2025}} & \yuquan{62} & \yuquan{0.5769} & \yuquan{0.9516} & \yuquan{0.7183} \\
\bottomrule
\end{tabular}
}
\caption{\yuquan{\textcolor[rgb]{0, 0, 0}{Seed sensitivity of \toole{}'s retrieval. Results cover two topics.} \#Papers counts evaluated retrieval candidates; Coverage, Correctness, and F1 follow the RQ1 definitions.}}
\label{tab:seed_sensitivity}
\end{table*}

\paragraph{\yuquan{\textbf{Results.}}}
\yuquan{Correctness remains high across all six runs (0.8361--0.9783), but Coverage and F1 depend on the seed. For Multi-Agent Workflow Automation, FlowReasoner increases both Coverage (0.6400 to 0.7200) and Correctness (0.8841 to 0.9238) relative to the original seed, while EvoAgentX attains the highest Coverage (0.8400) at lower Correctness (0.8361). For Video Accessibility, Coverage and F1 do not increase monotonically with the number of retrieved papers. These results suggest that different seeds expose different reachable portions of the citation graph, rather than merely changing retrieval volume.}

\section{\yuquan{Human Evaluation}}
\label{app:human_evaluation}

\subsection{\yuquan{Expert Audit of Retrieved Papers}}
\label{app:retrieval_expert_audit}
\label{app:scope_audit}
\label{sec:exp_rq13}

\yuquan{This experiment validates whether the LLM-based Correctness measure used in
RQ1 remains reliable under expert judgment.}

\paragraph{\yuquan{\textbf{Setting.}}}
\yuquan{The automatic measure could be overestimated if the LLM judge labels
irrelevant papers as relevant. We therefore analyze the most uncertain portion
of the reconstructed scope: papers retrieved by \tool but absent
from the human-curated scope. We sample ten papers for each of the ten topics,
giving 100 papers in total. Two PhD researchers independently judge every paper
against the corresponding topic definition. We report human-validated
precision, false-discovery rate, a confidence interval, and extrapolated scope
metrics in Table~\ref{tab:manual_false_positive}.}

\begin{table}[t]
\centering
\small
\setlength{\tabcolsep}{3.5pt}
\renewcommand{\arraystretch}{1.02}
\begin{tabular*}{\columnwidth}{@{\extracolsep{\fill}}lcc@{}}
\toprule
\yuquan{\textbf{Metric}} & \yuquan{\textbf{Evidence}} & \yuquan{\textbf{Value}} \\
\midrule
\yuquan{Validated extra precision} & \yuquan{$90/100$} & \yuquan{0.9000} \\
\yuquan{False-discovery rate} & \yuquan{$10/100$} & \yuquan{0.1000} \\
\yuquan{Precision 95\% CI} & \yuquan{100 samples} & \yuquan{$[0.8256, 0.9448]$} \\
\yuquan{Projected precision} & \yuquan{Extrapolated} & \yuquan{0.9005} \\
\yuquan{Projected recall} & \yuquan{Extrapolated} & \yuquan{0.8915} \\
\yuquan{Projected F1} & \yuquan{Extrapolated} & \yuquan{0.8959} \\
\bottomrule
\end{tabular*}
\caption{\yuquan{Expert validation of \tool papers outside the gold
scope.}}
\label{tab:manual_false_positive}
\end{table}

\paragraph{\yuquan{\textbf{Results.}}}
\yuquan{Most candidate false positives are judged relevant by the PhD evaluators. This
supports the reliability of the LLM-based Correctness measure and explains part
of the apparent mismatch with the gold scope: some papers counted as extra are
valid related work omitted from the human survey.}

\subsection{\yuquan{Method-Blinded Expert Comparison}}
\label{app:expert_comparison}

\yuquan{\textbf{Setting.} We compared InsightRAG and SurveyGen on two topics, Multi-Agent Workflow Automation and AI-Powered Video Accessibility for Blind and Low Vision Viewers. Three domain experts evaluated each topic (six in total). Given the topic scope, they saw both methods' taxonomies and surveys labeled System A and System B without method names; presentation order was balanced across experts. The evaluation guidance covered conceptual coverage, non-redundancy, and hierarchy for taxonomies, and content coverage, topical relevance, and organization for surveys. Experts selected A, B, or similar for each artifact, followed by a semi-structured interview about organization, strengths and weaknesses, and potential research use.}

\yuquan{\textbf{Results.} InsightRAG's taxonomy was preferred in 5/6 comparisons: 2/3 for Multi-Agent Workflow Automation and 3/3 for Video Accessibility. Its survey was preferred in all 6/6 comparisons, with 3/3 preferences on each topic.}

\yuquan{\textbf{Interview findings.} Experts often attributed InsightRAG's advantage to topic-specific organization: its categories were more directly tied to the target problem and more often followed a problem-solving or technical progression, whereas SurveyGen was perceived as broader and more list-like. All six experts perceived a connection between a taxonomy and its resulting survey organization; one noted that this held for both systems. Experts also identified weaknesses in topical focus, coverage of important directions, technical specificity, and progression across sections. In particular, some subtopics remained parallel rather than progressively connected, and cross-cutting or future-oriented sections were not always well integrated with the preceding structure. Experts envisioned using the system when entering or organizing a research area, identifying research gaps, or positioning their own work.}

\subsection{\yuquan{Human Evaluation Protocol}}
\label{app:human_evaluation_protocol}

\paragraph{\yuquan{Participants and compensation.}}
\yuquan{Participants in the retrieved-paper audit, expert comparison and interviews, and passage-based human audit were recruited online. All were PhD students with expertise related to the evaluation topics and were not authors of this paper. Each participant received USD~8 for approximately one hour of participation. Geographic information is not reported.}

\paragraph{\yuquan{Consent and reporting.}}
\yuquan{Before participation, participants were informed of the research purpose and the use of their evaluation responses and interview feedback, and gave their consent. They were explicitly informed that participation was voluntary and that they could withdraw. We report aggregate judgments and anonymized interview findings.}

\paragraph{\yuquan{Materials and instructions.}}
\yuquan{For the 100-paper retrieval audit, annotators received paper URLs and judged relevance against the corresponding topic definition. For the 24-pair citation audit, they were provided with the cited papers' full texts to assess support for the claims. The expert-comparison interface presented the following task instructions and evaluation criteria.}

\begin{quote}
\yuquan{Compare System A and System B for one assigned research topic. Their identities are hidden.}

\yuquan{\textbf{Taxonomy.} Judge the conceptual organization only. Paper identifiers and method names have been removed.}
\begin{itemize}[leftmargin=*,itemsep=0pt,topsep=2pt]
\item \yuquan{Coverage of the important conceptual areas in the stated topic scope.}
\item \yuquan{Distinct and non-redundant categories.}
\item \yuquan{Coherent parent-child organization and appropriate granularity.}
\end{itemize}
\yuquan{Which taxonomy better organizes the target research area overall?}

\yuquan{\textbf{Survey.} Judge the supplied survey text against the target scope. Citation faithfulness is evaluated separately.}
\begin{itemize}[leftmargin=*,itemsep=0pt,topsep=2pt]
\item \yuquan{Coverage of important methods, findings, limitations, and representative work.}
\item \yuquan{Relevance to the stated topic scope, without substantial off-topic material.}
\item \yuquan{Clear organization, coherent synthesis, and useful comparisons across work.}
\end{itemize}
\yuquan{Which survey is better overall for the target research scope?}

\yuquan{The response options for each preference question were ``A is better,'' ``Similar,'' and ``B is better.''}
\end{quote}

\section{\yuquan{\textcolor[rgb]{0, 0, 0}{Additional Taxonomy and Survey Results}}}
\subsection{\yuquan{\textcolor[rgb]{0, 0, 0}{Taxonomy Structural Diagnostics}}}
\label{app:taxonomy_diagnostics}

Node Ratio alone cannot show whether a granularity difference comes from more
categories, excessive depth, or overly detailed leaves. We therefore report
taxonomy size, depth, branching, and leaf statistics in
Table~\ref{tab:taxonomy_structure}.

\begin{table*}[t]
\centering
\small
\setlength{\tabcolsep}{4pt}
\resizebox{\textwidth}{!}{%
\begin{tabular}{lrrrrrrr}
\toprule
\textbf{Method} & \textbf{\#Nodes} & \textbf{Node Ratio} &
\textbf{Max Depth} & \textbf{Avg. Branch.} & \textbf{Leaf Ratio} &
\textbf{Small Leaf \%} & \textbf{Papers/Leaf} \\
\midrule
Vanilla LLM          & 30.20 & 1.3464 & 1.90 & 2.0528 & 0.7217 & 87.69 & 1.5533 \\
SurveyGen   & 27.30 & 1.1763 & 2.00 & 3.7833 & 0.7885 & 10.00 & 7.6436 \\
SurveyForge & 29.80 & 1.3092 & 2.00 & 4.9600 & 0.8321 & 1.60 & 9.6703 \\
AutoSurvey  & 52.10 & 2.2906 & 2.00 & 6.4429 & 0.8655 & 4.44 & 9.3235 \\
Deep Research        & 7.70 & 0.3482 & 1.80 & 2.1500 & 0.8242 & 36.17 & 3.5950 \\
\toole  & 60.20 & 2.2892 & 2.40 & 3.4767 & 0.7351 & 51.79 & 3.7362 \\
\bottomrule
\end{tabular}
}
\caption{Structural statistics of the generated taxonomies. These diagnostic
metrics describe granularity and do not share a single optimization direction.}
\label{tab:taxonomy_structure}
\end{table*}

Table~\ref{tab:taxonomy_structure} shows that the higher Node Ratio of
\tool mainly comes from a moderately deeper hierarchy and a larger
share of small, fine-grained leaves. Its branching factor and leaf ratio do not
indicate an uncontrolled expansion of broad categories. The main issue is leaf
fragmentation, which could be reduced by merging semantically similar
fine-grained nodes.

\subsection{\yuquan{\textcolor[rgb]{0, 0, 0}{Detailed LLM-Judge Results}}}
\label{app:judge_results}

Table~\ref{tab:survey_llm_judges} provides the complete Coverage, Structure,
Relevance, and average scores for each of the three judge models used in RQ3.1.

\begin{table*}[t]
\centering
\fontsize{8.5}{9.5}\selectfont
\renewcommand{\arraystretch}{1.0}
\setlength{\tabcolsep}{2.2pt}
\begin{tabular*}{\textwidth}{@{}l@{\extracolsep{\fill}}cccc@{\hspace{4pt}}cccc@{\hspace{4pt}}cccc@{}}
\toprule
& \multicolumn{4}{c}{\textbf{GPT-5.2}}
& \multicolumn{4}{c}{\textbf{Gemini 3.1 Pro}}
& \multicolumn{4}{c}{\textbf{DeepSeek V4 Pro}} \\
\cmidrule(lr){2-5}\cmidrule(lr){6-9}\cmidrule(lr){10-13}
\textbf{Method}
& \textbf{Cov.} & \textbf{Str.} & \textbf{Rel.} & \textbf{Avg.}
& \textbf{Cov.} & \textbf{Str.} & \textbf{Rel.} & \textbf{Avg.}
& \textbf{Cov.} & \textbf{Str.} & \textbf{Rel.} & \textbf{Avg.} \\
\midrule
Vanilla LLM
& \yuquan{\textbf{4.45}} & \yuquan{\textbf{4.50}} & \yuquan{\textbf{4.50}} & \yuquan{\textbf{4.48}}
& \yuquan{\textbf{5.00}} & \yuquan{\textbf{5.00}} & \textbf{5.00} & \yuquan{\textbf{5.00}}
& \yuquan{\textbf{4.50}} & \yuquan{\textbf{4.60}} & \yuquan{\textbf{4.90}} & \yuquan{\textbf{4.67}} \\
SurveyGen
& \yuquan{3.70} & \yuquan{4.05} & \yuquan{3.50} & \yuquan{3.75}
& \yuquan{3.75} & \yuquan{4.45} & \yuquan{3.70} & \yuquan{3.97}
& \yuquan{3.70} & \yuquan{3.75} & \yuquan{3.75} & \yuquan{3.73} \\
SurveyForge
& \yuquan{3.15} & \yuquan{3.65} & \yuquan{3.10} & \yuquan{3.30}
& \yuquan{4.55} & \yuquan{3.20} & \yuquan{4.10} & \yuquan{3.95}
& \yuquan{3.25} & \yuquan{2.90} & \yuquan{3.00} & \yuquan{3.05} \\
AutoSurvey
& \yuquan{3.00} & \yuquan{3.85} & \yuquan{2.65} & \yuquan{3.17}
& \yuquan{3.85} & \yuquan{4.65} & \yuquan{3.35} & \yuquan{3.95}
& \yuquan{3.15} & \yuquan{3.40} & \yuquan{2.85} & \yuquan{3.13} \\
Deep Research
& \yuquan{4.00} & \yuquan{4.05} & \yuquan{4.35} & \yuquan{4.13}
& \yuquan{4.90} & \yuquan{4.80} & \yuquan{\textbf{5.00}} & \yuquan{4.90}
& \yuquan{4.35} & \yuquan{4.25} & \yuquan{4.85} & \yuquan{4.48} \\
\toole
& \yuquan{4.35} & \yuquan{\textbf{4.50}} & \yuquan{4.40} & \yuquan{4.42}
& \yuquan{\textbf{5.00}} & \yuquan{\textbf{5.00}} & \textbf{5.00} & \yuquan{\textbf{5.00}}
& \yuquan{\textbf{4.50}} & \yuquan{\textbf{4.60}} & \textbf{4.90} & \yuquan{\textbf{4.67}} \\
\bottomrule
\end{tabular*}
\caption{LLM-judged survey quality across three judge models. Cov., Str., and
Rel. denote Coverage, Structure, and Relevance, respectively. Bold values
indicate the best result in each column.}
\label{tab:survey_llm_judges}
\end{table*}

\subsection{\yuquan{\textcolor[rgb]{0, 0, 0}{Passage-Based Citation Faithfulness Audit}}}
\label{app:passage_faithfulness}

\yuquan{\textcolor[rgb]{0, 0, 0}{We evaluated two topics: Multi-Agent
Workflow Automation and AI-Powered Video Accessibility for Blind and Low
Vision Viewers. The six methods were Vanilla LLM, SurveyGen, SurveyForge,
AutoSurvey, Deep Research, and \toole. We sampled seven claim--citation pairs
from each generated survey, yielding 84 pairs.} The automatic judge was not shown
method identities. For 74 pairs with a retrievable cited source, it read the
extracted full text and returned page-marked evidence or the closest passage.
The remaining ten pairs had unverifiable references and were labeled
Unsupported. The judge model was
\textcolor[rgb]{0, 0, 0}{Qwen-Plus}. It internally distinguished fully supported, partially
supported, and unsupported; for the reported binary comparison, Supported
combines the first two labels. Thus, Supported here means at least partial
passage support, not that every assertion in the claim is fully supported.}

\begin{table}[t]
\centering
\small
\setlength{\tabcolsep}{3pt}
\begin{tabular}{lcc}
\toprule
\yuquan{\textbf{Method}} & \yuquan{\textbf{Auto Supported}} & \yuquan{\textbf{Human Supported}} \\
\midrule
\yuquan{Vanilla LLM} & \yuquan{11/14 (78.6\%)} & \yuquan{\textbf{4/4 (100\%)}} \\
\yuquan{SurveyGen} & \yuquan{11/14 (78.6\%)} & \yuquan{1/4 (25\%)} \\
\yuquan{SurveyForge} & \yuquan{4/14 (28.6\%)} & \yuquan{1/4 (25\%)} \\
\yuquan{AutoSurvey} & \yuquan{2/14 (14.3\%)} & \yuquan{1/4 (25\%)} \\
\yuquan{Deep Research} & \yuquan{10/14 (71.4\%)} & \yuquan{2/4 (50\%)} \\
\yuquan{\toole} & \yuquan{\textbf{13/14 (92.9\%)}} & \yuquan{\textbf{4/4 (100\%)}} \\
\bottomrule
\end{tabular}
\caption{\yuquan{\textcolor[rgb]{0, 0, 0}{Passage-based citation support of generated surveys.} Auto scores use 14
claim--citation pairs per method; human scores use four method-blinded pairs
per method. Fully and partially supported automatic labels count as Supported;
unverifiable references count as Unsupported.}}
\label{tab:passage_faithfulness}
\end{table}

\yuquan{\textcolor[rgb]{0, 0, 0}{The human audit covers 24 pairs, two from each generated survey.}
\tool and Vanilla LLM each have four human-supported pairs. Automatic and
human binary judgments agree on 20/24 pairs (83.3\%, $\kappa=0.655$); in all
four disagreements the automatic judge is more lenient. Because this audit
covers only two topics and a small number of pairs per method, it supplements
rather than replaces the ten-topic citation-entry evaluation in the main text.
The actual full-text judge prompt is given in
\S\ref{app:prompt_passage_faithfulness}.}

\section{Metric Definitions and Evaluation Details}
\label{app:metric_definitions}
\label{sec:appendix}

This section collects the formal definitions and implementation details of all
metrics used in Section~\ref{sec:experiments}. Let $\mathcal{T}$ denote the set
of evaluation topics and $m$ a method.

\subsection{Research-Scope Reconstruction Metrics}
\label{app:rq1_metrics}

For topic $t$, let $G_t$ be the human-curated gold paper set, $R_{m,t}$ the
papers retrieved by method $m$, $J_{m,t}\subseteq R_{m,t}$ the retrieved papers
successfully evaluated by the relevance judge, and
$\mathrm{Rel}_{m,t}\subseteq J_{m,t}$ those judged relevant. We compute
topic-level Coverage and Correctness as
\begin{equation}
\begin{aligned}
\mathrm{Coverage}_{m,t}
  &=\frac{|R_{m,t}\cap G_t|}{|G_t|},\\
\mathrm{Correctness}_{m,t}
  &=\frac{|\mathrm{Rel}_{m,t}|}{|J_{m,t}|}.
\end{aligned}
\label{eq:scope_coverage_correctness}
\end{equation}
The reported scores average the corresponding topic-level values:
\begin{equation}
\begin{aligned}
\mathrm{Coverage}_{m}
  &=\frac{1}{|\mathcal{T}|}\sum_{t\in\mathcal{T}}
    \mathrm{Coverage}_{m,t},\\
\mathrm{Correctness}_{m}
  &=\frac{1}{|\mathcal{T}|}\sum_{t\in\mathcal{T}}
    \mathrm{Correctness}_{m,t}.
\end{aligned}
\label{eq:scope_aggregation}
\end{equation}
Their harmonic mean gives
\begin{equation}
\mathrm{F1}_{m}=\frac{2\,\mathrm{Coverage}_{m}\,
\mathrm{Correctness}_{m}}{\mathrm{Coverage}_{m}+
\mathrm{Correctness}_{m}}.
\label{eq:scope_f1}
\end{equation}
Coverage therefore measures overlap with the human scope, whereas Correctness
is relevance precision estimated by the LLM judge rather than gold-set
precision.

For the expert audit, validated extra-paper precision is the fraction of sampled
papers judged relevant, and false-discovery rate is one minus this precision.
Projected precision, recall, and F1 extrapolate the audited relevance rate to
the complete set of papers retrieved outside the gold scope.

\subsection{Taxonomy Evaluation Metrics}
\label{app:rq2_metrics}

Let $V_G$ and $V_C$ be the non-root nodes of the gold and generated taxonomies,
and let $d(\cdot)$ denote node depth. Heading Soft Recall (HSR) matches every
gold heading to its most similar generated heading and penalizes depth
disagreement:
\begin{equation}
\mathrm{HSR}=\frac{1}{|V_G|}\sum_{g\in V_G}\max_{c\in V_C}
\mathrm{sim}(g,c)\,0.9^{|d(g)-d(c)|}.
\label{eq:hsr}
\end{equation}
\yuquan{For the reported HSR and CEDS scores, $\mathrm{sim}(g,c)$ is the cosine
similarity between embeddings of the gold and generated taxonomy node paths,
clipped to $[0,1]$. The root heading is excluded from each path; a node with
no remaining path uses its own heading. The embedding model is
\textcolor[rgb]{0, 0, 0}{Qwen3-Embedding-0.6B}.}
Catalogue Edit Distance Similarity (CEDS) normalizes tree-edit distance by the
combined number of nodes:
\begin{equation}
\mathrm{CEDS}=1-\frac{\mathrm{TED}(T_G,T_C)}{|V_G|+|V_C|}.
\label{eq:ceds}
\end{equation}
\yuquan{The tree-edit cost uses $1-\mathrm{sim}(g,c)$ to relabel matched nodes
and counts nodes in inserted or deleted subtrees.}
Node Ratio is $|V_C|/|V_G|$; unlike the other metrics, it is descriptive, and a
value closer to one indicates a taxonomy size closer to the human reference.

For outline quality, the aggregate score averages Topic Uniqueness (TU),
Structural Balance (SB), Hierarchical Clarity (HC), and Logical Organization
(LO):
\begin{equation}
\mathrm{OutlineAvg}=\frac{\mathrm{TU}+\mathrm{SB}+\mathrm{HC}+\mathrm{LO}}{4}.
\label{eq:outline_average}
\end{equation}

\yuquan{Let $\mathcal{R}$ be the concepts extracted from Related Work trees.
For RW--Taxonomy Alignment, $\mathrm{sim}(r,v)$ is the cosine similarity
between embeddings of an RW concept's path (with nonredundant evidence text
when available) and a taxonomy node's path, using the same embedding model. We compute
the two directional best-match scores}
\begin{equation}
\begin{aligned}
R_{\mathrm{RW}}
  &=\frac{1}{|\mathcal{R}|}\sum_{r\in\mathcal{R}}
    \max_{v\in V_C}\mathrm{sim}(r,v),\\
P_{\mathrm{RW}}
  &=\frac{1}{|V_C|}\sum_{v\in V_C}
    \max_{r\in\mathcal{R}}\mathrm{sim}(v,r),
\end{aligned}
\label{eq:rw_recall_precision}
\end{equation}
and define RW--Taxonomy Alignment as their harmonic mean:
\begin{equation}
\mathrm{RWTA}=\frac{2R_{\mathrm{RW}}P_{\mathrm{RW}}}
{R_{\mathrm{RW}}+P_{\mathrm{RW}}}.
\label{eq:rwta}
\end{equation}
\yuquan{\textcolor[rgb]{0, 0, 0}{HSR and CEDS compare the generated taxonomy with the human-survey reference hierarchy, assessing its correspondence with human-authored organization. RWTA instead measures alignment with organizational concepts extracted from author-written Related Work in the explored literature. Because this material also informs taxonomy construction, RWTA assesses evidence alignment rather than providing an independent human-reference comparison.}}
Paper--Taxonomy Assignment Accuracy (PTA) is the number of paper assignments
judged correct divided by the number of evaluated assignments.

\subsection{Survey Evaluation Metrics}
\label{app:rq3_metrics}

Each LLM judge scores Coverage, Relevance, and Structure on a five-point scale.
The score for one judge--topic evaluation is
\begin{equation}
\begin{aligned}
\mathrm{SurveyAvg}=\frac{1}{3}(&\mathrm{Coverage}+\mathrm{Relevance}\\
&+\mathrm{Structure}).
\end{aligned}
\label{eq:survey_judge_average}
\end{equation}

For Semantic Similarity (SS), let $H$ and $G$ denote the chunks of the human and
generated surveys. We compute best-match similarities in both directions:
\begin{equation}
\begin{aligned}
s_{\mathit{human}\rightarrow\mathit{gen}}
&=\frac{1}{|H|}\sum_{h\in H}\max_{g\in G}\cos(h,g),\\
s_{\mathit{gen}\rightarrow\mathit{human}}
&=\frac{1}{|G|}\sum_{g\in G}\max_{h\in H}\cos(g,h),\\
\mathrm{SS}&=\frac{s_{\mathit{human}\rightarrow\mathit{gen}}+
s_{\mathit{gen}\rightarrow\mathit{human}}}{2}.
\end{aligned}
\label{eq:semantic_similarity}
\end{equation}
For topic $t$, let $K_t$ be the number of human key points and
$K^{\mathrm{cov}}_{m,t}$ the number covered by method $m$. Key-Point Recall
(KPR) is
\begin{equation}
\mathrm{KPR}_{m,t}=100\times\frac{K^{\mathrm{cov}}_{m,t}}{K_t}.
\label{eq:kpr}
\end{equation}

For citation evaluation, a claim is recalled when its complete citation group
supports it. An individual citation is precise when it independently supports
the claim, or when the complete group supports the claim and removing that
citation makes the remaining group insufficient. Let $C_{\mathrm{sup}}$ and
$C_{\mathrm{all}}$ be the numbers of supported and evaluated claims, and let
$I_{\mathrm{nec}}$ and $I_{\mathrm{all}}$ be the numbers of necessary and total
citation instances. Thus,
\begin{equation}
\begin{aligned}
\mathrm{CitRecall}&=\frac{C_{\mathrm{sup}}}{C_{\mathrm{all}}},\\
\mathrm{CitPrecision}&=\frac{I_{\mathrm{nec}}}{I_{\mathrm{all}}}.
\end{aligned}
\label{eq:citation_faithfulness}
\end{equation}
Finally, let $N_{\mathrm{match}}$ and $N_{\mathrm{ref}}$ be the numbers of
matched and extracted reference entries. Verifiable Reference Rate (VRR) is
\begin{equation}
\mathrm{VRR}=\frac{N_{\mathrm{match}}}{N_{\mathrm{ref}}}.
\label{eq:vrr}
\end{equation}
The reported KPR, citation, and reference scores average the corresponding
topic-level values. The \#References column reports the total number of
extracted reference entries rather than claim--citation pairs.

\section{Algorithmic and Implementation Details}
\label{app:algorithmic_details}

\subsection{\yuquan{Backbone Models, Retrieval Budgets, Runtime, and Token Usage}}
\label{app:implementation_budget}

\yuquan{Table~\ref{tab:implementation_budget} summarizes the backbone models, principal limits, and available token estimates for four main stages. The greedy taxonomy constructor is deterministic and makes no LLM calls in that step.}

\begin{table*}[t]
\centering
\small
\setlength{\tabcolsep}{5pt}
\renewcommand{\arraystretch}{1.12}
\begin{tabular}{p{0.23\textwidth}p{0.15\textwidth}p{0.23\textwidth}p{0.27\textwidth}}
\toprule
\yuquan{\textbf{Stage}} & \yuquan{\textbf{Model}} & \yuquan{\textbf{Main limits}} & \yuquan{\textbf{Estimated tokens per topic}} \\
\midrule
\yuquan{Insight-guided retrieval} & \yuquan{Qwen-Max} & \yuquan{Retrieval depth 4; paper budget 400} & \yuquan{1.4M--43.7M input; median 9.1M} \\
\yuquan{Greedy taxonomy construction} & \yuquan{Deterministic greedy} & \yuquan{Maximum taxonomy depth 4} & \yuquan{0 LLM tokens in this step} \\
\yuquan{Global taxonomy refinement} & \yuquan{Qwen-Max} & \yuquan{Maximum taxonomy depth 4} & \yuquan{4.7k--31.1k recorded input} \\
\yuquan{Evidence-bounded survey generation} & \yuquan{Qwen-Plus} & \yuquan{N/A} & \yuquan{14k--84k final-output proxy} \\
\bottomrule
\end{tabular}
\caption{\yuquan{\textcolor[rgb]{0, 0, 0}{Implementation settings and token estimates for \toole.} Input estimates and the final-output proxy have different accounting bases; these figures do not constitute a complete end-to-end token total.}}
\label{tab:implementation_budget}
\end{table*}

\yuquan{Across the ten evaluation topics, \tool retrieved 1,879 papers in total (187.9 per topic on average). Survey generation took 5.0--28.2 minutes per topic (median 6.4 minutes); this is survey-generation time, not the runtime of the full retrieval-to-survey pipeline.}

\subsection{Insight-Guided Retrieval}
\label{app:algorithm_retrieval}

The pseudocode below makes explicit both exploration operators and their
scheduling. The research request is converted into \(d\) using the
scope-structuring prompt (\S\ref{app:prompt_scope_structuring}).
\textsc{Scope-Admission} and \textsc{Extract-RWT} invoke the templates in
\S\ref{app:prompt_scope_admission} and \S\ref{app:prompt_rwt_extraction}.
The schedule uses a forward-round bound \(F_{\max}\), a per-paper Forward
candidate budget \(K_f\), a Backward depth bound \(D_{\max}\), and a global
paper budget \(B_p\).
\yuquan{\textcolor[rgb]{0, 0, 0}{For each of the ten main-study topics, we manually selected one representative, high-impact paper within the target scope as the initial seed. The pipeline resolves the supplied title against Semantic Scholar and verifies the paper's fit with the scope before expansion. Table~\ref{tab:initial_seeds} lists the seeds.}}

\begin{table*}[t]
\centering
\small
\setlength{\tabcolsep}{5pt}
\renewcommand{\arraystretch}{1.08}
\begin{tabular}{cp{0.84\textwidth}}
\toprule
\yuquan{\textcolor[rgb]{0, 0, 0}{\textbf{Topic ID}}} & \yuquan{\textcolor[rgb]{0, 0, 0}{\textbf{Initial seed paper}}} \\
\midrule
\yuquan{\textcolor[rgb]{0, 0, 0}{0}} & \yuquan{\textcolor[rgb]{0, 0, 0}{Multi-agent Architecture Search via Agentic Supernet~\citep{zhang2025multiagentarchitecturesearchagentic}}} \\
\yuquan{\textcolor[rgb]{0, 0, 0}{1}} & \yuquan{\textcolor[rgb]{0, 0, 0}{Making Short-Form Videos Accessible with Hierarchical Video Summaries~\citep{Van_Daele_2024}}} \\
\yuquan{\textcolor[rgb]{0, 0, 0}{2}} & \yuquan{\textcolor[rgb]{0, 0, 0}{Everybody Dance Now~\citep{TODO_seed_pose}}} \\
\yuquan{\textcolor[rgb]{0, 0, 0}{3}} & \yuquan{\textcolor[rgb]{0, 0, 0}{POCO: Point Convolution for Surface Reconstruction~\citep{TODO_seed_3d_reconstruction}}} \\
\yuquan{\textcolor[rgb]{0, 0, 0}{4}} & \yuquan{\textcolor[rgb]{0, 0, 0}{Fast Multiparty Threshold ECDSA with Fast Trustless Setup~\citep{TODO_seed_threshold_ecdsa}}} \\
\yuquan{\textcolor[rgb]{0, 0, 0}{5}} & \yuquan{\textcolor[rgb]{0, 0, 0}{Poisoning Language Models During Instruction Tuning~\citep{TODO_seed_llm_finetuning}}} \\
\yuquan{\textcolor[rgb]{0, 0, 0}{6}} & \yuquan{\textcolor[rgb]{0, 0, 0}{Not What You\textquotesingle ve Signed Up For: Compromising Real-World LLM-Integrated Applications with Indirect Prompt Injection~\citep{TODO_seed_ai_service_failures}}} \\
\yuquan{\textcolor[rgb]{0, 0, 0}{7}} & \yuquan{\textcolor[rgb]{0, 0, 0}{Learning Multi-dimensional Indexes~\citep{TODO_seed_learned_indexing}}} \\
\yuquan{\textcolor[rgb]{0, 0, 0}{8}} & \yuquan{\textcolor[rgb]{0, 0, 0}{Deep Learning-Based Image Semantic Coding for Semantic Communications~\citep{TODO_seed_semantic_coding}}} \\
\yuquan{\textcolor[rgb]{0, 0, 0}{9}} & \yuquan{\textcolor[rgb]{0, 0, 0}{Energy-Aware Cloud Workflow Applications Scheduling With Geo-Distributed Data~\citep{TODO_seed_geo_scheduling}}} \\
\bottomrule
\end{tabular}
\caption{\yuquan{\textcolor[rgb]{0, 0, 0}{Initial seed papers for the ten main-study topics. Topic IDs follow Table~\ref{app-dataset}.}}}
\label{tab:initial_seeds}
\end{table*}

\begin{AlgorithmBox}{Algorithm 1: Forward Operator}
procedure FORWARD(p, d, Sa, Sf, Se, G, R)
  C <- papers citing p, each paired with the
       passage in which it discusses p
  C <- remove clearly uninformative candidates from C
  if the number of candidates in C exceeds Kf then
    C <- prioritize C by citation impact and recency
    C <- retain the first Kf candidates in C
  end if

  for each (q, passage) in C do
    decision <- SCOPE-ADMISSION(q, d, passage)
    if decision = ACCEPT then
      if q not in Sa then
        Sa <- Sa union {q}
        R[q] <- EXTRACT-RWT(q, d)
      end if
      G <- ADD-EDGE(G, q, p, passage)
    end if
  end for

  Se <- Se union {p}
  Sf <- Sa minus Se
  return Sa, Sf, Se, G, R
end procedure
\end{AlgorithmBox}

\begin{AlgorithmBox}{Algorithm 2: RW-Guided Backward Operator}
procedure RW-BACKWARD(p, d, Sa, Sf, Se, G, R, depth)
  C <- leaf papers of R[p], each paired with
       its author-written source passage

  for each (q, passage) in C do
    if the paper budget Bp is reached then break
    decision <- SCOPE-ADMISSION(q, d, passage)
    if decision = ACCEPT then
      if q not in Sa then
        Sa <- Sa union {q}
        R[q] <- EXTRACT-RWT(q, d)
        depth[q] <- depth[p] + 1
      end if
      G <- ADD-EDGE(G, p, q, passage)
    end if
  end for

  Se <- Se union {p}
  Sf <- Sa minus Se
  return Sa, Sf, Se, G, R, depth
end procedure
\end{AlgorithmBox}

\begin{AlgorithmBox}{Algorithm 3: Retrieval Scheduling and Reseeding}
procedure SEED-RUN(d, seed)
  decision <- SCOPE-ADMISSION(seed, d, no passage)
  if decision = REJECT then
    return empty set, empty graph, empty tree set
  end if

  Sa <- {seed}; Sf <- {seed}; Se <- empty set
  G <- graph containing seed
  R[seed] <- EXTRACT-RWT(seed, d)

  for forward_step <- 1 to Fmax do
    if Sf is empty then break
    Fcurrent <- ORDER-BY-INFLUENCE(snapshot of Sf)
    for each p in Fcurrent do
      Sa, Sf, Se, G, R <-
        FORWARD(p, d, Sa, Sf, Se, G, R)
    end for
  end for

  Sf <- Sa
  Se <- empty set
  for each p in Sa do depth[p] <- 0

  while the paper budget Bp is not reached do
    Bcurrent <- {p in snapshot of Sf : depth[p] < Dmax}
    if Bcurrent is empty then break
    for each p in Bcurrent do
      if the paper budget Bp is reached then break
      Sa, Sf, Se, G, R, depth <-
        RW-BACKWARD(p, d, Sa, Sf, Se, G, R, depth)
    end for
  end while

  G <- LOOP-CLOSURE(G)
  return Sa, G, R
end procedure

procedure INSIGHT-GUIDED-RETRIEVAL(request, p0)
  d <- STRUCTURE-SCOPE(request)
  P0, G0, R0 <- SEED-RUN(d, p0)

  H <- papers in G0 that are cited at least once
       and were not prior Forward seeds
  H <- TOP-3-BY-IN-DEGREE(H)

  subgraphs <- {G0}; trees <- {R0}
  for each h in H do
    Ph, Gh, Rh <- SEED-RUN(d, h)
    subgraphs <- subgraphs union {Gh}
    trees <- trees union {Rh}
  end for

  G <- MERGE-AND-LOOP-CLOSE(subgraphs)
  P <- papers in G
  return P, G, trees
end procedure
\end{AlgorithmBox}

\noindent A seed initializes the search state only after passing
\textsc{Scope-Admission}. Each Forward round takes a snapshot of the current
frontier and applies the operator to every paper in that snapshot. Influence
determines processing order only; it does not exclude any frontier paper.
Each processed paper enters \(S_e\), admitted citing papers enter \(S_a\), and
\(S_f\) is updated to \(S_a\setminus S_e\), forming the next round's frontier.
Forward terminates when \(S_f\) is empty or \(F_{\max}\) rounds have completed.
Before Backward begins, \(S_f\) is reset to \(S_a\) and \(S_e\) is cleared.
Each Backward round similarly expands every eligible paper in a snapshot of
the frontier, so newly admitted predecessors form the next recursive round.
Backward terminates when no frontier paper remains below \(D_{\max}\) or the
paper budget \(B_p\) is reached. Each \textsc{Seed-Run} ends with loop closure
before the independently constructed subgraphs are merged.

\subsection{Minimum-Conflict Taxonomy Construction}
\label{app:algorithm_taxonomy}

The objective in Section~\ref{sec:method_taxonomy_formulation} defines the
ideal taxonomy--assignment pair. The implementation operationalizes this
objective in three stages: it constructs a shared candidate concept space from
citation contexts, recursively builds an evidence-supported greedy draft, and
then resolves the remaining semantic and structural conflicts through global
refinement. The LLM calls used in the last stage are reported in
\S\ref{app:method_prompts_taxonomy}.

\begin{AlgorithmBox}{Algorithm 4: Candidate Concept Space Construction}
procedure BUILD-CANDIDATE-SPACE(P, G)
  C <- empty set
  w(c, p) <- 0 for every c and p

  for each citation edge (p', p) in E(G) do
    passage <- author-written context on (p', p)
    phrases <- EXTRACT-CONCEPT-PHRASES(passage)
    for each c in phrases do
      C <- C union {c}
      w(c, p) <- w(c, p) + 1
    end for
  end for

  C, w <- MERGE-EQUIVALENT-CONCEPTS(C, w)
  C, w <- FILTER-BY-EVIDENCE-SUPPORT(C, w)
  delta <- ASSIGN-SEMANTIC-DIMENSIONS(C)

  for each p in P without supported concept evidence do
    C <- C union {c_bottom}
    w(c_bottom, p) <- 0
  end for

  return C, w, delta
end procedure
\end{AlgorithmBox}

\noindent \textsc{Extract-Concept-Phrases} removes citation markers and other
non-semantic material before retaining phrases that describe the role or
category of the cited paper. Semantically equivalent phrases are consolidated
before evidence filtering, while conflicting assignments remain separate in
the candidate space. Thus, a paper may retain positive support under multiple
candidate concepts.

\begin{AlgorithmBox}{Algorithm 5: Greedy Draft Taxonomy Construction}
procedure BUILD-GREEDY-DRAFT(P, C, w, delta, Gamma)
  create root r containing P
  SPLIT(r, P, depth = 0)
  T0 <- tree rooted at r
  A0 <- assignments induced by the leaves of T0
  return T0, A0
end procedure

procedure SPLIT(v, Pv, depth)
  if depth reaches the bound in Gamma
     or too few papers remain in Pv then return

  Cv <- concepts supported by papers in Pv
  Cv <- remove concepts with insufficient support
        or violating Gamma
  if Cv is empty then return

  d_star <- dimension in delta with greatest total support
  Cv <- {c in Cv : delta(c) = d_star}

  H <- empty list; covered <- empty set
  while the branching budget is not reached do
    c_star <- candidate in Cv minus H with the largest
              marginal coverage and evidence support
    if no supported c_star remains then break
    append c_star to H
    covered <- covered union papers supported by c_star
  end while
  if H is empty then return

  create one child of v for every concept in H
  for each paper p in Pv do
    assign p to the child with the highest
      normalized support w(c,p) / sum_{h in H} w(h,p)
    if p has insufficient support then
      place p in a fallback child or defer its assignment
    end if
  end for

  for each child u of v do
    SPLIT(u, papers assigned to u, depth + 1)
  end for
end procedure
\end{AlgorithmBox}

\noindent Algorithm~5 applies the same split recursively. It terminates a
branch when the depth bound is reached, too few papers remain, or no supported
and constraint-compatible split is available. The result is the draft
taxonomy--assignment pair \((T_0,A_0)\). Appendix~\ref{app:complexity}
formalizes the objective and establishes its NP-hardness.

\begin{AlgorithmBox}[yuquan orange algorithm]{Algorithm 6: Global Refinement and Finalization}
procedure REFINE-AND-FINALIZE(d, P, T0, A0, {R_p})
  Qc <- PROFILE-CONCEPTS(concepts in T0)
  Qp <- PROFILE-PAPERS(P)
  ER <- AGGREGATE-RW-TREE-EVIDENCE({R_p})
        as recurring categories and parent-child relations

  T <- GLOBAL-REFINE(d, T0, Qc, Qp, ER)
  A <- REFINE-PAPER-ASSIGNMENTS(P, T, A0, Qp)
  T, A <- VALIDATE-AND-REPAIR(T, A)

  return T, A
end procedure
\end{AlgorithmBox}

\noindent \yuquan{\textcolor[rgb]{0, 0, 0}{\textsc{Global-Refine} designs the final hierarchy from the evidence-grounded greedy draft \(T_0\), using aggregated RW Tree patterns as auxiliary structural evidence.}}
\textsc{Refine-Paper-Assignments}
takes \(A_0\) as the default, revises only assignments invalidated by changes
to the hierarchy, and maps every paper to exactly one leaf of the fixed
taxonomy. The final validation checks assignment completeness and uniqueness
and removes unsupported empty leaves.

\section{Complexity of Minimum-Conflict Taxonomy Construction}
\label{app:complexity}

This appendix restates the optimization problem of
Section~\ref{sec:method_taxonomy_formulation} in full, proves that it is
NP-hard, and isolates the structural feature that causes the hardness.

\subsection{Problem Restatement}
\label{app:complexity_statement}

An instance of \textsc{MCTC} consists of a paper set \(P\), a candidate concept
set \(\mathcal{C}\) whose non-fallback concepts have passed global evidence
filtering, and non-negative integer evidence weights \(w(c,p)\) for
\((c,p)\in\mathcal{C}\times P\). It also contains a dimension map
\(\delta:\mathcal{C}\rightarrow\mathcal{D}\), positive integer depth and
branching bounds \(B\) and \(k\), and a parsimony coefficient
\(\lambda\ge 0\). The root \(r\) is not a candidate concept. The candidate set
contains a fallback concept \(c_{\bot}\), with \(w(c_{\bot},p)=0\) for every
\(p\in P\). It can be selected as a child of the root and receive any paper,
guaranteeing that a feasible assignment remains available even when the paper
has no supported candidate concept.

A solution is represented by hierarchy-arc indicators
\(x_{u,c}\in\{0,1\}\) for distinct
\(u\in\{r\}\cup\mathcal{C}\) and \(c\in\mathcal{C}\), concept indicators
\(y_c\in\{0,1\}\), assignment indicators \(a_{c,p}\in\{0,1\}\), level values
\(\ell_c\), and dimension indicators \(z_{u,d}\in\{0,1\}\). For convenience,
let \(y_r=1\) and \(\ell_r=0\). Unless stated otherwise, the constraints below
range over \(u\in\{r\}\cup\mathcal{C}\), \(c,v\in\mathcal{C}\),
\(p\in P\), and \(d\in\mathcal{D}\), with hierarchy indicators defined only
for distinct endpoints. These variables satisfy
\begin{equation}
\label{eq:app_mctc_constraints}
\begin{aligned}
&\textstyle\sum_{u\in(\{r\}\cup\mathcal{C})\setminus\{c\}}x_{u,c}=y_c,\\
&x_{u,c}\le y_u,\\
&x_{u,c}=1\Rightarrow \ell_c=\ell_u+1\le B,\\
&\textstyle\sum_{\substack{c\in\mathcal{C}\\c\ne u}}x_{u,c}\le k\,y_u,\\
&\textstyle\sum_{c\in\mathcal{C}}a_{c,p}=1,\\
&a_{c,p}\le y_c,\\
&a_{c,p}+x_{c,v}\le 1,\\
&x_{u,c}\le z_{u,\delta(c)},\\
&\textstyle\sum_{d\in\mathcal{D}}z_{u,d}\le 1,\\
&\textstyle\sum_{v\in\mathcal{C}\setminus\{c_{\bot}\}}x_{c_{\bot},v}=0,
\end{aligned}
\end{equation}
The parent and level constraints give every selected concept exactly one
parent and enforce a rooted acyclic hierarchy of depth at most \(B\). The
branching constraint applies to every internal node, including the root. The
assignment constraints make the assignment total and single-valued, while
\(a_{c,p}+x_{c,v}\le 1\) ensures that papers are assigned only to leaf
concepts. Finally, the dimension constraints require the children of any
parent to share one organizational dimension. Let \(T\) denote the resulting
tree and \(\mathcal{C}(T)=\{c:y_c=1\}\) its selected concepts, and define
\(A(p)=c\) whenever \(a_{c,p}=1\). Under these constraints, \textsc{MCTC}
maximizes the following objective:
\begin{equation}
\label{eq:app_mctc_objective}
\begin{aligned}
\mathcal{J}(T,A)={}&
\sum_{c\in\mathcal{C}(T)}\sum_{p\in P}w(c,p)
+\sum_{p\in P}w\bigl(A(p),p\bigr)\\
&-\lambda\bigl|\mathcal{C}(T)\bigr|.
\end{aligned}
\end{equation}
The first term rewards the selection of concepts with strong aggregate
evidence, the second rewards evidence-supported leaf assignments, and the
third discourages unnecessarily complex taxonomies.

\subsection{NP-Hardness}
\label{app:complexity_theorem}

\begin{theorem}
\label{thm:app_mctc_hard}
The \textsc{MCTC} optimization problem is NP-hard. Hardness holds already for
instances with \(B=1\), a single semantic dimension \(|\mathcal{D}|=1\),
\(\lambda=1\), binary evidence weights, and
\(\sum_{p\in P}w(c,p)=3\) for every non-fallback concept.
\end{theorem}

\begin{proof}
We reduce from \textsc{Exact Cover by 3-Sets} (\textsc{X3C})
\citep{garey1979computers}. An \textsc{X3C} instance consists of a universe
\(\mathcal{U}=\{e_1,\dots,e_{3q}\}\) and a family
\(\mathcal{S}=\{S_1,\dots,S_m\}\), where every \(S_j\) contains exactly three
elements. It asks whether there is a subfamily of exactly \(q\) sets that
partitions \(\mathcal{U}\).

Construct an \textsc{MCTC} instance as follows. Create one paper \(p_i\) for
each element \(e_i\), one concept \(c_j\) for each set \(S_j\), and the
fallback concept \(c_{\bot}\). Define
\begin{equation*}
w(c_j,p_i)=
\begin{cases}
1,& e_i\in S_j,\\
0,& e_i\notin S_j,
\end{cases}
\qquad
w(c_{\bot},p_i)=0.
\end{equation*}
Assign every concept the same dimension, and set
\(B=1\), \(k=q\), and \(\lambda=1\). We compare the optimum against the
threshold \(\theta=5q\). The construction is polynomial in the size of the
\textsc{X3C} instance.

Suppose first that the \textsc{X3C} instance has an exact cover indexed by
\(I\subseteq\{1,\dots,m\}\), with \(|I|=q\). Select the concepts
\(\{c_j:j\in I\}\), attach them directly to the root, and do not select
\(c_{\bot}\). Because the selected sets partition \(\mathcal{U}\), every paper
\(p_i\) can be assigned to the unique selected concept whose set contains
\(e_i\). The resulting taxonomy is feasible: it has depth one, exactly \(q\)
root children, and a single organizational dimension. Each selected concept
has support three, so the concept-support term is \(3q\); every paper receives
an assignment of weight one, so the assignment-support term is also \(3q\);
and the complexity penalty is \(q\). Therefore,
\begin{equation*}
\mathcal{J}(T,A)=3q+3q-q=5q=\theta.
\end{equation*}

Conversely, suppose a feasible \textsc{MCTC} solution satisfies
\(\mathcal{J}(T,A)\ge 5q\). Let \(I=\{j:y_{c_j}=1\}\), let \(s=|I|\), and
let \(b=y_{c_{\bot}}\in\{0,1\}\). Since \(B=1\), every selected concept is a
leaf attached directly to the root. The root branching bound gives
\(s+b\le q\). Each ordinary concept has total support three, whereas the
fallback has support zero, so the concept-support term equals \(3s\). The
assignment-support term is at most \(3s\), because the \(s\) selected ordinary
concepts jointly provide positive evidence for at most \(3s\) distinct papers.
Consequently,
\begin{equation*}
\mathcal{J}(T,A)
\le 3s+3s-(s+b)
=5s-b
\le 5q.
\end{equation*}
The assumed lower bound forces equality throughout. Thus \(b=0\), \(s=q\),
and the assignment-support term is \(3q\). Hence every one of the \(3q\)
papers is assigned with weight one to a selected concept, so the corresponding
\(q\) sets cover all elements of \(\mathcal{U}\). Since each selected set has
size three, these \(q\) sets contain exactly \(3q\) element occurrences; a
cover of the \(3q\)-element universe is therefore possible only if they are
pairwise disjoint. They form an exact cover.

The \textsc{X3C} instance is therefore positive if and only if the constructed
\textsc{MCTC} instance admits a solution of value at least \(5q\). An exact
polynomial-time algorithm for \textsc{MCTC} would consequently decide
\textsc{X3C} in polynomial time. Hence \textsc{MCTC} is NP-hard.
\end{proof}

Because the reduction uses a flat taxonomy, the hardness already arises from
selecting a supported subset of concepts under the branching budget; it does
not require a deep hierarchy.

\section{Extended Related Work}
\label{app:full_related_work}

\subsection{Citation-Aware Scholarly Retrieval and Relation Modeling}
\label{app:full_rw_citation}

Scientific literature retrieval has traditionally relied on lexical matching or embedding-based representations \citep{coha2020specter}. Although these approaches effectively capture content-level similarity, they provide limited support for modeling complex relations among papers, such as identifying corroborating or conflicting findings and tracing the development of a technique. Consequently, the resulting knowledge structure can remain fragmented \citep{shao2026scinet}. Existing scientific retrieval benchmarks also focus primarily on topical discovery: LitSearch \citep{ajith2024litsearch} and PASA \citep{he2025pasa} primarily evaluate the identification of papers relevant to a specified research topic, while STARK incorporates network structure but remains centered on structured multi-hop queries over predefined entities \citep{wu2024stark}.

To move beyond surface-level similarity, another line of work exploits citation-graph structure. \yuquan{Traditional citation analysis, including co-citation and main-path analysis, operates over paper-level citation structures, but does not explicitly encode the textual semantics of why one paper cites another \citep{zhang2026internatlas}.}
\yuquan{Modern scholarly infrastructures expose citation links, while some also provide citation contexts or influential-citation annotations \citep{zhang2026internatlas}.}
Building on these resources, structure-aware retrieval and RAG methods such as LitFM and CG-RAG use citation propagation and multi-hop retrieval to improve scientific question answering and synthesis \citep{zhang2024litfm,hu2025cgrag}.

\yuquan{In these retrieval systems, citation links primarily serve as graph-connectivity or relevance signals. Unless citation context is explicitly modeled, the link itself does not specify whether the citing author describes an extension, improvement, contrast, or background relation. A complementary line of work therefore examines the semantics of individual citations. SciCite classifies citation intent as background, method, or result comparison \citep{cohan2019scicite}, while Valenzuela et al.\ distinguish important from incidental citations \citep{valenzuela2015meaningful}. These studies show that citations are not semantically interchangeable links. Their focus, however, is on characterizing citation instances rather than reconstructing a topic-specific literature scope from relations expressed across papers.}

\yuquan{Relatedly likewise recognizes Related Work sections as a resource for understanding a research area \citep{palani2023relatedly}. It helps readers explore Related Work paragraphs across papers and develop topic outlines, making distributed author perspectives easier to inspect. Its contribution is to support readers in navigating and reconciling these perspectives; it does not itself integrate them into an automatically reconstructed literature scope, taxonomy, and survey. \tool instead uses the passages surrounding citations as author-attributed relational evidence throughout these stages.}

\subsection{\yuquan{Scholarly Taxonomy Construction}}
\label{app:full_rw_taxonomy}

\yuquan{Organizing papers into a hierarchical taxonomy is also central to literature reviews. CHIME generates topical categories and assigns studies from a given paper set, with experts correcting category relations and paper assignments \citep{hsu2024chime}. Its primary focus is organizing a supplied collection rather than discovering the literature scope to be organized. HiGTL combines paper text and citation-graph structure for hierarchical clustering, then verbalizes concepts at each level \citep{hu2025higtl}. In this approach, citation links contribute graph connectivity, rather than the citing authors' descriptions of what those links mean. Zhu et al.\ construct taxonomies through LLM-generated, multi-aspect paper representations and dynamic clustering, and introduce the TaxoBench-CS benchmark \citep{zhu2025context}. Their categories are derived from representations of paper content. Together, these methods construct useful hierarchies from paper content or citation-network structure, but do not explicitly use judgments expressed in author-written Related Work as evidence for discovering candidate concepts, weighing their support, and organizing papers. \tool draws candidate concepts and support weights from citation contexts written by citing authors, linking taxonomy construction to the reconstructed literature scope.}

\subsection{Automatic Survey Generation}
\label{app:full_rw_asg}

\yuquan{With the rapid growth of scientific literature, automatic survey generation is becoming an increasingly important capability for both AI systems and human researchers. Mohammad et al.\ combined citation texts with multi-document summarization to generate technical surveys, showing that citation texts contain survey-worthy information not readily available from abstracts \citep{mohammad2009using}. Recent LLM-based approaches commonly formulate this task as a hierarchical retrieval--planning--writing pipeline: agents first retrieve papers relevant to a target topic, construct an outline, and then compose individual sections in parallel \citep{wang2024autosurvey}. AutoSurvey established the paradigm of combining semantic retrieval with parallel generation \citep{wang2024autosurvey}, which was subsequently extended by AutoSurvey2 \citep{wu2025autosurvey2} and LLM$\times$MapReduce \citep{zhou2024llmmapreduce}. Later studies improve different components of this pipeline. SurveyForge introduces memory-driven retrieval and heuristic planning learned from human-written outlines \citep{yan2025surveyforge}; SurveyX performs structured preprocessing with an AttributeTree \citep{liang2025surveyx}; SurveyGen incorporates quality-aware feedback \citep{bao2025surveygen}; and multi-agent systems explore collaborative generation through specialized agent roles \citep{go2025lira}. To support evaluation across scientific fields, recent work has also introduced discipline-aware benchmarks \citep{surveylens2026}.}

Despite these advances, most existing systems pass retrieved evidence to downstream generation modules as unstructured text. Reasoning about how concepts, methods, and research themes relate to one another is therefore deferred to the generation stage, requiring an LLM to organize long and heterogeneous evidence without an explicit conceptual or taxonomic structure \citep{pedinotti2026structsurvey}. Recent work has begun to move structural reasoning earlier in the pipeline. StructSurvey, for example, extracts entities, relations, and a lightweight taxonomy on demand from retrieved abstracts \citep{pedinotti2026structsurvey}, while DeepSurvey models cross-paper relations through full-text reading, clustering, and comparative analysis \citep{yang2026deepsurvey}. Recovering structure this way, however, carries two limits. First, the relations are inferred rather than attributed: the model proposes them from whatever text it has read, and no source sentence is required to stand behind any particular relation, so an error in the extracted structure leaves nothing for a reader to check it against. Second, and more fundamentally, reading a paper yields only that paper's own vantage point. A paper states its problem, its method, and its reading of the work it builds on, but how the field later positioned it---which subsequent work extended it, which contested it, which reframed the problem it addressed---is written by the authors who cited it. The evidence needed to place a paper within a global scope is therefore distributed across other papers, and no amount of careful reading of each retrieved paper in isolation will merge those vantage points into one. \tool collects them instead: it follows citations in both directions, takes the relation from what the citing authors wrote, and aggregates these scattered judgments into a unified scaffold that supports both retrieval and generation.

\yuquan{Taken together, the two limitations meet at the same point: a paper read on its own supplies only its own vantage point, and an untyped citation edge supplies only the fact of a link, whereas the passage an author wrote around that citation supplies the relation itself. \tool builds on this observation. It does not treat citations merely as graph structures or similarity signals for retrieval, nor does it stop at extracting isolated pairwise relations between papers; it regards Related Work sections as distributed scholarly judgments already articulated by authors, and aggregates their curated citation paths and explicit relational statements into a cross-paper and temporal scaffold that jointly supports literature retrieval, taxonomy construction, and survey generation.}

\section{Method Prompt Templates}
\label{app:method_prompts}

\subsection{Prompts for Insight-Guided Retrieval}
\label{app:method_prompts_retrieval}

This subsection reports the three LLM templates used by retrieval. Runtime
values are replaced by angle-bracketed fields. Forward and RW-guided Backward
share the same scope-admission template; reseeding, loop closure, and graph
merging are rule-based and therefore require no additional prompts.

\subsubsection{Research-Scope Structuring}
\label{app:prompt_scope_structuring}

This prompt converts the user's request into the structured research scope
\(d\) used by all subsequent admission decisions.

\begin{SystemPromptBox}
You are an expert research librarian and academic editor. You translate research requests into strict academic search configurations. Output only the structured fields requested by the user prompt.
\end{SystemPromptBox}

\begin{UserPromptBox}
# Task

Translate the following user research request into a
Strict Academic Search Configuration.

# Instructions

1. Refine Topic:
   Convert the request into standard academic terminology.

2. Define Precisely:
   Write a one-sentence positive definition of the core
   concept. If the user already provides a clear definition,
   preserve it faithfully.

3. Extract Constraints:
   Identify what the user explicitly or implicitly excludes.
   Preserve user-provided exclusion criteria.

4. Infer Defaults:
   If target audience or perspective is unspecified, use
   sensible academic defaults.

5. Year Range:
   Set year_range only when the user explicitly specifies
   a temporal range; otherwise return null.

# User Query

<USER_RESEARCH_REQUEST>

# Required Structured Output

{
  "topic": "<REFINED_ACADEMIC_TOPIC>",
  "definitions": "<POSITIVE_SCOPE_DEFINITION>",
  "constraints": "<EXCLUSIONS_AND_BOUNDARIES>",
  "target_audience": "<TARGET_AUDIENCE>",
  "perspective": "<SURVEY_PERSPECTIVE>",
  "year_range": "<EXPLICIT_RANGE_OR_NULL>"
}
\end{UserPromptBox}

\subsubsection{Scope Admission}
\label{app:prompt_scope_admission}

This prompt implements a binary scope-admission decision. For the initial seed,
the citation context field is empty; Forward and Backward candidates receive
the passage through which they were retrieved.

\begin{SystemPromptBox}
You are an extremely strict academic reviewer. Make a binary decision about whether a paper falls strictly within the exact scope of the given research topic. Return relevant=true only when the paper's primary research scope matches the target scope; otherwise return relevant=false. Output only a valid JSON object.
\end{SystemPromptBox}

\begin{UserPromptBox}
# Task

Judge whether the following paper's primary research scope
is identical to the research topic definition below.
The paper must directly study, advance, or evaluate the
exact problem defined, rather than being merely related,
applicable, upstream, or downstream.

# Research Topic

<RESEARCH_TOPIC>

# Exact Topic Definition

<STRUCTURED_RESEARCH_SCOPE>

# Strict Judgment Rules

Apply every rule. A violation of any exclusion rule yields NO.

Mandatory scope alignment:
1. The paper's primary research question or contribution
   must fall squarely within the exact scope.
2. The paper must explicitly position itself within the
   defined research area.

Explicit exclusions:
3. An upstream component technology that could be used by
   the topic, but is not itself about the topic, yields NO.
4. A downstream application that uses the topic to solve a
   different problem, without advancing the topic, yields NO.
5. An adjacent or overlapping field addressing a different
   research question yields NO.
6. A general-purpose method, broad framework, survey, or
   toolkit not designed for the exact topic yields NO.
7. A tangential mention in related work, motivation, or
   future work yields NO.

Use of citation context:
8. Citation context is supplementary evidence of the
   candidate's role and contribution. It cannot by itself
   justify YES; the paper's own content must confirm that
   its primary scope matches the topic.

When to answer YES:
9.  The core contribution directly advances the exact
    research problem.
10. The paper provides a dataset, benchmark, evaluation
    framework, or user study designed specifically for the
    exact topic.
11. The paper addresses a sub-problem explicitly included
    in the topic definition.

Default:
12. If uncertain, answer NO. Answer YES only with high
    confidence that the paper's primary scope matches the
    defined topic.

# Citation Context

<CITATION_OR_RELATED_WORK_PASSAGE>

# Paper to Judge

Title: <PAPER_TITLE>

Paper Content:
<PAPER_CONTENT>

# Required Output

{
  "relevant": true,
  "reason": "<ONE_SENTENCE_JUSTIFICATION>"
}
\end{UserPromptBox}

\subsubsection{RW Tree Extraction}
\label{app:prompt_rwt_extraction}

This prompt extracts the local hierarchy and resolves its leaf citations
against the paper's reference list.

\begin{SystemPromptBox}
You are an expert researcher specializing in the target topic. Given the title, structured research scope, and full text of a paper, locate its Related Work or literature-review section and its reference list. Extract only citations discussed in Related Work, resolve each citation marker to the full paper title, and organize the retained papers into a fine-grained hierarchical RW Tree. Keep direct predecessors, competitors, and topic-specific components; discard generic background, neighboring tasks, and general surveys used only for context. Do not use generic category names such as Methods or Related Work, remove empty categories, and preserve the exact source sentence for every leaf citation. Return only the requested JSON object.
\end{SystemPromptBox}

\begin{UserPromptBox}
{
  "paper_title": "<PAPER_TITLE>",
  "focus_topic": "<RESEARCH_TOPIC>",
  "research_scope": "<STRUCTURED_RESEARCH_SCOPE>",
  "full_paper_text": "<FULL_PAPER_TEXT>",
  "required_output": {
    "root_title": "<PAPER_TITLE>",
    "focus_topic": "<RESEARCH_TOPIC>",
    "should_expand": true,
    "relevant_categories": [
      {
        "category_name": "<FINE_GRAINED_CATEGORY>",
        "sub_categories": [
          {
            "category_name": "<FINE_GRAINED_SUBCATEGORY>",
            "leaves": [
              {
                "citation_id": "<IN_TEXT_CITATION_MARKER>",
                "title": "<RESOLVED_FULL_PAPER_TITLE>",
                "insight": "<TOPIC_SPECIFIC_CONTRIBUTION>",
                "source_context": "<EXACT_SOURCE_SENTENCE>"
              }
            ]
          }
        ],
        "leaves": [
          {
            "citation_id": "<IN_TEXT_CITATION_MARKER>",
            "title": "<RESOLVED_FULL_PAPER_TITLE>",
            "insight": "<TOPIC_SPECIFIC_CONTRIBUTION>",
            "source_context": "<EXACT_SOURCE_SENTENCE>"
          }
        ]
      }
    ]
  }
}
\end{UserPromptBox}

\subsection{Prompts for Minimum-Conflict Taxonomy Construction}
\label{app:method_prompts_taxonomy}

The greedy candidate-space and draft-construction stages are deterministic and
use no LLM prompt. The following templates are used only for semantic
profiling, late-stage global refinement, and conflict-aware paper assignment
refinement. Each is sent as a single user message; the implementation does not
add a separate system message. The experiments use the coarse profile setting
shown below.

\subsubsection{Concept Profile Generation}
\label{app:prompt_concept_profile}

This template compresses each draft concept into a short definition and three
keywords for global refinement.

\begin{UserPromptBox}
Generate a semantic profile for the academic taxonomy
concept "<CONCEPT_NAME>".

Return strict JSON:
{
  "definition": "a phrase of at most 15 words stating
                   the core definition of the concept",
  "keywords": ["<KEYWORD_1>", "<KEYWORD_2>",
               "<KEYWORD_3>"]
}
\end{UserPromptBox}

\subsubsection{Paper Profile Generation}
\label{app:prompt_paper_profile}

This template extracts the compact problem, method, and keywords used to
reason about the paper during hierarchy refinement and assignment.

\begin{UserPromptBox}
Paper: <PAPER_TITLE>
Abstract: <PAPER_ABSTRACT>

Extract the key information used to classify this paper
in a taxonomy. Return strict JSON:
{
  "problem": "a phrase of at most 10 words stating the
              core research problem",
  "method": "a phrase of at most 10 words stating the
             method or technical approach",
  "keywords": ["<KEYWORD_1>", "<KEYWORD_2>",
               "<KEYWORD_3>"]
}
\end{UserPromptBox}

\subsubsection{Global Taxonomy Structural Refinement}
\label{app:prompt_global_taxonomy_refinement}

\yuquan{\textcolor[rgb]{0, 0, 0}{This template refines structure only, using the research scope, concept and paper profiles, the greedy draft taxonomy, and auxiliary evidence aggregated from RW Trees. It contains no paper assignment.}}

\begin{UserPromptBox}
## Task: Design a Fine-Grained Academic Taxonomy

You are a senior researcher designing a comprehensive
taxonomy for a survey on "<TOPIC_NAME>". Your only task is
to design the taxonomy structure. Do not assign papers.

### Research Topic

<STRUCTURED_RESEARCH_SCOPE>

### Core Concepts

<CONCEPT_PROFILES>

### Papers in Scope

Each line has the following format:
Title | Keywords | Problem | Method

<PAPER_PROFILES>

### Baseline Taxonomy

The following draft is provided for reference. You should
significantly improve upon it:

<GREEDY_DRAFT_TAXONOMY>

### Related Work Tree Evidence

The following auxiliary evidence is aggregated from the
paper-level RW Trees. It captures how source papers organize
their related literature, but it is not ground truth.

Top recurring categories:
<RW_TREE_CATEGORY_SUPPORT>

Top recurring parent-child relations:
<RW_TREE_RELATION_SUPPORT>

Compact per-paper RW Trees:
<COMPACT_RW_TREES>

Use high-support RW Tree categories and parent-child
relations to shape the taxonomy when they are consistent
with the topic definition and paper profiles. A recurring
theme may be retained, merged into a semantically broader
branch, or excluded when it is clearly noisy or off-scope.
Do not copy isolated low-support terms verbatim.

### Design Requirements

1. Ground every branch in the research-scope definition.
2. Use 4-6 first-level categories aligned with core
   research dimensions.
3. Give each first-level category 4-7 subcategories;
   research-dense areas may use deeper subcategories up to
   depth <MAX_DEPTH>.
4. Target <TARGET_LEAF_COUNT> leaves, within the range
   <MIN_LEAF_COUNT>-<MAX_LEAF_COUNT>.
5. Make the categories mutually exclusive and collectively
   exhaustive.
6. Use specific, descriptive academic names of 3-6 words.
7. Avoid vague names such as Other, General,
   Miscellaneous, Advanced, or Emerging.
8. Make sibling categories distinct and non-synonymous.
9. Keep children as specializations of their parent.
10. Balance granularity so that papers can distribute
    naturally across the leaves.

### Required Output

{
  "rationale": "two or three sentences explaining the
                design and refinement of the baseline",
  "taxonomy": {
    "concept": "<TOPIC_NAME>",
    "children": [
      {
        "concept": "<FIRST_LEVEL_CATEGORY>",
        "definition": "<ONE_SENTENCE_SCOPE>",
        "children": [
          {
            "concept": "<SUBCATEGORY>",
            "definition": "<ONE_SENTENCE_SCOPE>"
          }
        ]
      }
    ]
  }
}

Return strict JSON containing structure and definitions
only. Do not include a papers field.
\end{UserPromptBox}

\subsubsection{Paper Assignment Refinement}
\label{app:prompt_paper_leaf_assignment}

Once the refined structure is fixed, this template takes the draft assignment
\(A_0\) as its default and revises only assignments invalidated by changes to
the hierarchy. It returns a complete assignment so that every paper remains
mapped to exactly one refined leaf.

\begin{UserPromptBox}
## Task: Refine Paper Assignments after Taxonomy Refinement

### Refined Leaf Categories

The leaf categories below use 0-based indices. Each line
contains the full category path and its definition.

<INDEXED_LEAF_CATEGORIES>

### Papers and Draft Assignments

Each line has the format:
[8-character ID] Title | Keywords | Method |
Draft leaf path in A0

<PAPER_BATCH_WITH_DRAFT_ASSIGNMENTS>

### Rules

1. Treat each paper's draft assignment in A0 as the default.
2. Preserve the draft assignment when its leaf remains in
   the refined taxonomy, or when it has a clear semantic
   counterpart and the paper remains compatible with that
   leaf.
3. Reassign a paper only when its draft leaf has been
   removed, merged, split, or semantically changed so that
   the original assignment is no longer valid.
4. When reassignment is required, choose the most
   semantically relevant refined leaf using the paper title,
   keywords, method, full category path, and category
   definition.
5. Assign every paper to exactly one refined leaf category.
6. Do not modify the taxonomy structure.

### Required Output

{
  "assignments": {
    "<PAPER_ID_1>": <LEAF_INDEX>,
    "<PAPER_ID_2>": <LEAF_INDEX>
  }
}

Return strict JSON. Every paper ID in the batch must appear
exactly once in the assignments object. Preserve its draft
assignment unless refinement makes that assignment invalid.
\end{UserPromptBox}

\subsection{Prompts for Taxonomy-Guided Survey Generation}
\label{app:method_prompts_survey_generation}

Survey generation uses six single-message templates: two for local relation
mining, one shared by leaf and internal writing units, and three for survey-level
planning and completion. Runtime data is replaced below by angle-bracketed
fields. RW Tree snippets may inform semantic relations but are explicitly
non-citable unless the corresponding paper also appears in the unit's allowed
citation set.

\subsubsection{Local Paper-Relation Mining}
\label{app:prompt_paper_relation_mining}

This template is applied only to leaf units containing at least two papers.

\begin{UserPromptBox}
You are helping write an academic survey.

## Task

Find useful local paper relations inside the current
taxonomy concept.

Use only the provided papers and evidence. You do not need
to connect every pair. Only papers listed in
allowed_citations are citable. RW Tree snippets are
semantic background only; papers mentioned inside them are
not citable unless they also appear in allowed_citations.

## Decision Labels

- direct_relation: explicit evidence exists, such as a
  direct citation, citation context, or shared RW Tree
  discussion.
- hidden_relation: no strong direct evidence exists, but
  the papers are meaningfully comparable, complementary,
  sequential, or methodologically different within the
  current concept.
- none: no useful writing relation exists. Do not output
  relations labeled none.

## Writing Constraints

- Be conservative when proposing hidden_relation.
- Do not claim superiority unless it is explicitly
  supported by the evidence.
- Produce a relation_sentence that can be used directly in
  a survey paragraph.
- Refer to papers using only numeric citation markers such
  as [12].
- Use only markers listed in allowed_citations.
- Do not cite paper titles inside square brackets.
- Do not mention external papers that appear only in RW
  Tree or source-context snippets.
- Return strict JSON only.

## Required Output

{
  "paper_relations": [
    {
      "source_paper": "<PAPER_ID>",
      "target_paper": "<PAPER_ID>",
      "decision": "direct_relation | hidden_relation",
      "relation_mode": "comparison | progression |
                        complementarity | contrast |
                        shared_problem | other",
      "confidence": 0.0,
      "relation_sentence": "<SURVEY_READY_SENTENCE>",
      "evidence_summary": "<EVIDENCE_SUMMARY>"
    }
  ]
}

## Local Input

{
  "concept": "<CONCEPT>",
  "taxonomy_path": "<TAXONOMY_PATH>",
  "concept_definition": "<CONCEPT_DEFINITION>",
  "allowed_citations": <ALLOWED_CITATIONS>,
  "papers": <LOCAL_PAPER_PROFILES>,
  "graph_evidence": <LOCAL_GRAPH_AND_RW_TREE_EVIDENCE>
}
\end{UserPromptBox}

\subsubsection{Child-Concept Relation Mining}
\label{app:prompt_child_relation_mining}

This template is applied to internal units containing at least two direct
children.

\begin{UserPromptBox}
You are helping write transitions inside an academic
survey section.

## Task

Find useful relations among the direct child concepts of
the current parent concept.

Use only the provided child concepts, paper summaries, and
evidence. You do not need to connect every pair. Only
papers listed in allowed_citations are citable. Evidence
snippets are semantic background only; papers mentioned
there are not citable unless they appear in
allowed_citations.

## Decision Labels

- direct_relation: the child paper sets have explicit
  citation or citation-context evidence connecting them.
- hidden_relation: the child concepts are semantically
  useful to compare, contrast, order, or connect in
  writing.
- none: no useful writing relation exists. Do not output
  relations labeled none.

## Writing Constraints

- The relation_sentence should explain how the two child
  concepts relate inside the parent section.
- Do not force a relation when the children are merely
  siblings without a meaningful transition.
- Refer to representative papers only with numeric citation
  markers such as [12].
- Use only markers listed in allowed_citations.
- Do not cite paper titles inside square brackets.
- Do not mention external papers that appear only in
  evidence snippets.
- Return strict JSON only.

## Required Output

{
  "child_concept_relations": [
    {
      "source_concept": "<CHILD_CONCEPT>",
      "target_concept": "<CHILD_CONCEPT>",
      "decision": "direct_relation | hidden_relation",
      "relation_mode": "comparison | progression |
                        complementarity | contrast |
                        parallel_dimensions | other",
      "direction": "source_to_target | target_to_source |
                    none",
      "confidence": 0.0,
      "relation_sentence": "<SURVEY_READY_SENTENCE>",
      "evidence_summary": "<EVIDENCE_SUMMARY>"
    }
  ]
}

## Local Input

{
  "parent_concept": "<PARENT_CONCEPT>",
  "taxonomy_path": "<TAXONOMY_PATH>",
  "allowed_citations": <ALLOWED_CITATIONS>,
  "children": <CHILD_CONCEPT_PROFILES>,
  "graph_evidence": <CROSS_CHILD_CITATION_EVIDENCE>
}
\end{UserPromptBox}

\subsubsection{Writing-Unit Generation}
\label{app:prompt_writing_unit_generation}

Leaf and internal units share this template and are distinguished by the
\texttt{unit\_type} field in the package.

\begin{UserPromptBox}
You are writing one subsection of an academic survey.

## Goal

Write a coherent survey subsection using the provided
taxonomy scope, papers, and relation sentences.

## Rules

- Use only papers listed in allowed_citations.
- Use only citation markers listed in allowed_citations.
- Do not invent citation numbers, paper titles,
  abbreviations, aliases, citations, or paper relations.
- If a paper appears only in an RW Tree, source context, or
  evidence snippet, treat it as non-citable background and
  do not mention it by name.
- Do not write a paper-by-paper annotated bibliography.
- For unit_type = section_overview, write a high-level
  overview of the child concepts using
  child_concept_relations; avoid detailed paper-by-paper
  synthesis.
- For unit_type = paper_subsection, synthesize the assigned
  papers using paper_relations.
- Organize the subsection around comparison, progression,
  complementarity, and contrast when supported by the
  supplied relation sentences.
- Phrase hidden relations conservatively.
- Cite papers only with numeric markers such as [12].
- Do not add a References section or the subsection title;
  the pipeline supplies both separately.
- If this is a section overview, write a concise bridge and
  do not summarize every paper.
- Return Markdown prose only, without JSON.

## Writing Package

{
  "section_id": "<SECTION_ID>",
  "subsection_title": "<TITLE>",
  "taxonomy_path": "<TAXONOMY_PATH>",
  "subsection_scope": "<SCOPE>",
  "unit_type": "section_overview | paper_subsection",
  "allowed_citations": <ALLOWED_CITATIONS>,
  "overall_outline": <SURVEY_OUTLINE>,
  "previous_section": <PREVIOUS_SECTION>,
  "next_section": <NEXT_SECTION>,
  "papers": <LOCAL_PAPERS>,
  "paper_relations": <PAPER_RELATIONS>,
  "children": <CHILD_CONCEPTS>,
  "child_concept_relations": <CHILD_RELATIONS>
}
\end{UserPromptBox}

\subsubsection{Survey-Level Planning}
\label{app:prompt_survey_level_planning}

This template plans the sections that surround the already-written taxonomy
body; it does not rewrite that body.

\begin{UserPromptBox}
You are planning the global structure of a taxonomy-guided
academic survey.

## Task

Create a concise survey-level plan that wraps an
already-written taxonomy body. The taxonomy body must not
be rewritten. The taxonomy may change across runs, so use
the provided taxonomy dynamically.

The final survey should contain:
- Title and Abstract;
- 1. Introduction;
- 2. Background and Scope;
- 3. Taxonomy Overview;
- the existing taxonomy body, using the dynamic section
  interval in survey_structure;
- Cross-Cutting Discussion;
- Open Challenges and Future Directions;
- Conclusion; and
- References.

## Planning Constraints

- Do not invent papers, citations, counts, or taxonomy
  dimensions.
- Use the existing taxonomy as the core body.
- Do not assume a fixed research area, fixed taxonomy
  section names, or a fixed number of top-level sections.
- Use only section_summaries to define taxonomy dimensions.
- When mentioning corpus size, use only corpus_stats.
- Do not claim first, largest, comprehensive, empirically
  grounded, or state-of-the-art unless explicitly supported
  by the input.
- Do not critique existing surveys unless they are
  explicitly provided as references.
- Focus on reader onboarding, cross-cutting synthesis, and
  conclusion.
- Return strict JSON only.

## Required Output

{
  "survey_title": "<TITLE>",
  "survey_topic": "<TOPIC>",
  "central_thesis": "<CENTRAL_THESIS>",
  "reader_problem": "<READER_PROBLEM>",
  "scope": ["<IN_SCOPE_ITEM>"],
  "out_of_scope": ["<OUT_OF_SCOPE_ITEM>"],
  "main_contributions": ["<CONTRIBUTION>"],
  "background_terms": [
    {
      "term": "<TERM>",
      "plain_definition": "<DEFINITION>",
      "why_needed": "<WHY_NEEDED>"
    }
  ],
  "taxonomy_dimensions": [
    {
      "section_number": "<SECTION_NUMBER>",
      "title": "<DIMENSION_TITLE>",
      "core_question": "<CORE_QUESTION>",
      "coverage": "<COVERAGE>",
      "role_in_survey": "<ROLE>",
      "connection_to_next": "<CONNECTION>",
      "representative_citations": ["[1]", "[2]"]
    }
  ],
  "cross_cutting_themes": [
    {
      "title": "<THEME>",
      "description": "<DESCRIPTION>",
      "applies_to_sections": ["4", "6"],
      "supporting_citations": ["[1]", "[2]"]
    }
  ],
  "future_directions": [
    {
      "title": "<DIRECTION>",
      "problem": "<UNRESOLVED_PROBLEM>",
      "why_it_matters": "<IMPORTANCE>",
      "possible_direction": "<POSSIBLE_DIRECTION>",
      "supporting_citations": ["[1]", "[2]"]
    }
  ],
  "conclusion_message": "<CONCLUSION_MESSAGE>"
}

## Input

{
  "corpus_stats": <CORPUS_STATISTICS>,
  "survey_structure": <DYNAMIC_SECTION_STRUCTURE>,
  "outline": <TAXONOMY_OUTLINE>,
  "section_summaries": <TOP_LEVEL_SECTION_SUMMARIES>,
  "references": <COMPACT_GLOBAL_REFERENCES>,
  "taxonomy_body": <ASSEMBLED_TAXONOMY_BODY>
}
\end{UserPromptBox}

\subsubsection{Front-Matter Generation}
\label{app:prompt_front_matter_generation}

\begin{UserPromptBox}
You are writing the front matter of a taxonomy-guided
academic survey.

## Task

Write only the title, abstract, introduction,
background/scope, and taxonomy overview. Do not write the
existing taxonomy body. Use the provided plan and taxonomy
dimensions dynamically.

## Required Markdown Structure

# <SURVEY_TITLE>

## Abstract

## 1. Introduction

## 2. Background and Scope

## 3. Taxonomy Overview

## Writing Goals

- Make the survey accessible to readers unfamiliar with
  the exact taxonomy.
- Explain the topic, reader problem, scope, and contribution
  using the supplied plan.
- Define the survey topic and four to six reusable terms in
  Background and Scope.
- Clarify included and excluded literature without
  enumerating every taxonomy dimension in that section.
- Explain the taxonomy organization and summarize each
  top-level dimension dynamically in Taxonomy Overview.
- A compact Markdown table may be used in the overview.

## Citation Constraints

- Use citations sparingly, with at most two citation
  clusters per paragraph.
- Use only numeric markers from the provided references.
- Do not invent citations or add a References section.
- Do not mention local files or paper identifiers.
- Use only corpus_stats when stating corpus size.
- Do not claim first, largest, comprehensive, empirically
  grounded, or state-of-the-art unless supported by the
  plan.
- Do not compare against existing surveys unless they are
  explicitly present in the references.
- Prefer conservative descriptions such as corpus-derived,
  taxonomy-guided, or based on the current corpus.
- Do not introduce dimensions absent from
  plan.taxonomy_dimensions.
- Return Markdown only, without meta commentary.

## Input

{
  "corpus_stats": <CORPUS_STATISTICS>,
  "survey_structure": <DYNAMIC_SECTION_STRUCTURE>,
  "plan": <SURVEY_LEVEL_PLAN>,
  "outline": <TAXONOMY_OUTLINE>,
  "references": <COMPACT_GLOBAL_REFERENCES>
}
\end{UserPromptBox}

\subsubsection{Back-Matter Generation}
\label{app:prompt_back_matter_generation}

\begin{UserPromptBox}
You are writing the back matter of a taxonomy-guided
academic survey.

## Task

Write only the discussion, open challenges, and conclusion
that follow an already-written taxonomy body. Do not
rewrite or repeat the taxonomy body. Use the supplied plan
and dynamic survey_structure.

## Required Markdown Structure

## <DISCUSSION_SECTION>. Cross-Cutting Discussion

## <FUTURE_SECTION>. Open Challenges and Future Directions
### <FUTURE_SECTION>.1 <CHALLENGE_TITLE>
### <FUTURE_SECTION>.2 <CHALLENGE_TITLE>

## <CONCLUSION_SECTION>. Conclusion

## Writing Goals

- Synthesize themes across taxonomy branches instead of
  summarizing each section again.
- Make the discussion reader-oriented and concrete.
- For future directions, state what remains unresolved, why
  current work is insufficient, and what could be done next.
- Write two to five future-direction subsections derived
  from plan.future_directions; do not invent unrelated
  challenges.
- Keep the conclusion short and introduce no new claims or
  papers.

## Citation Constraints

- Use citations sparingly, with at most two citation
  clusters per paragraph.
- Use only numeric markers from the provided references.
- Do not invent citations or add a References section.
- Do not mention local files or paper identifiers.
- Use only corpus_stats when stating corpus size.
- Do not claim first, largest, comprehensive, empirically
  grounded, or state-of-the-art unless supported by the
  plan.
- Prefer conservative descriptions such as corpus-derived,
  taxonomy-guided, or based on the current corpus.
- Do not introduce dimensions absent from
  plan.taxonomy_dimensions.
- Return Markdown only, without meta commentary.

## Input

{
  "corpus_stats": <CORPUS_STATISTICS>,
  "survey_structure": <DYNAMIC_SECTION_STRUCTURE>,
  "plan": <SURVEY_LEVEL_PLAN>,
  "outline": <TAXONOMY_OUTLINE>,
  "references": <COMPACT_GLOBAL_REFERENCES>,
  "taxonomy_body": <RENUMBERED_TAXONOMY_BODY>
}
\end{UserPromptBox}

\section{Evaluation Prompt Templates}
\label{app:evaluation_prompts}

This section reports the complete templates used for LLM-based evaluation.
Angle-bracketed fields are populated at runtime with the corresponding topic,
paper, taxonomy, or survey content. Each prompt is presented as a breakable
card within the standard two-column appendix layout.
\yuquan{Method identifiers are retained for result aggregation but are not included in
the judge-facing prompts.}

\subsection{Retrieved-Paper Relevance Judgment}
\label{app:prompt_retrieval_relevance}

This prompt determines whether each retrieved paper is substantively relevant
to the target research scope.

\begin{SystemPromptBox}
You are a strict evaluator for a survey-paper retrieval experiment. Your task is to decide whether a candidate research paper is substantively relevant to the target survey topic. Use only the topic prompt, paper title, and paper abstract. Do not reward superficial keyword overlap. Mark relevant=true only if the paper would reasonably belong in the scope of a survey on this topic. Do not think step by step. Return exactly one compact JSON object as the first and only output. Do not include markdown, explanations outside JSON, or hidden reasoning.
\end{SystemPromptBox}

\begin{UserPromptBox}
{
  "topic_name": "<TOPIC_NAME>",
  "topic_prompt": "<TOPIC_SCOPE>",
  "paper_title": "<PAPER_TITLE>",
  "paper_abstract": "<PAPER_ABSTRACT>",
  "decision_instructions": [
    "Do not reason step by step. Answer immediately.",
    "Output exactly one JSON object and nothing else.",
    "Return relevant=true if this paper is substantively in scope for the topic.",
    "Return relevant=false if it is off-topic, too broad, only shares keywords, or lacks a clear connection.",
    "Use score from 0.0 to 1.0, where 1.0 means clearly in scope.",
    "Use confidence from 0.0 to 1.0.",
    "Keep reason under 20 words."
  ],
  "required_json_schema": {
    "relevant": "boolean",
    "score": "number between 0 and 1",
    "confidence": "number between 0 and 1",
    "reason": "short string"
  }
}
\end{UserPromptBox}

\subsection{Taxonomy Evaluation}
\label{app:prompt_taxonomy_evaluation}

Taxonomy evaluation uses one prompt to assess the organization of the complete
taxonomy and another to assess individual paper-to-taxonomy assignments.

\subsubsection{Taxonomy Outline Quality Judgment}
\label{app:prompt_taxonomy_quality}

\begin{SystemPromptBox}
You are a strict evaluator for academic survey taxonomy quality. Return exactly one compact JSON object as the first and only output. Do not include chain-of-thought or markdown.
\end{SystemPromptBox}

\begin{UserPromptBox}[yuquan prompt]
{
  "task": "Evaluate the organization quality of a generated academic survey taxonomy for taxonomy construction. This is not an evaluation of the final survey writing outline.",
  "topic_id": "<TOPIC_ID>",
  "topic_name": "<TOPIC_NAME>",
  "target_scope": {
    "source": "<TOPIC_PROMPT_SOURCE>",
    "input_prompt": "<TOPIC_SCOPE>"
  },
  "outline_includes_paper_counts": false,
  "scale": "0 to 5, where 0 is very poor and 5 is excellent.",
  "rubric": {
    "topic_uniqueness": "Sibling nodes should be distinct and non-overlapping. Penalize duplicate, near-duplicate, or mixed-scope technical categories. Do not give extra credit merely because a node is a generic survey section such as Foundations, Background, Applications, Future Work, or Open Challenges.",
    "structural_balance": "The taxonomy should have reasonable breadth and depth for the technical concept space. Penalize over-fragmented, overly flat, or lopsided trees, but do not require conventional survey-outline sections such as Foundations or Open Challenges.",
    "hierarchical_clarity": "Parent-child relations should be clear abstraction-to-specific relations. Penalize children that do not belong under their parent. Generic survey-writing sections are not automatically clearer than technical taxonomy categories.",
    "logical_organization": "The taxonomy should organize the field's technical concepts coherently for taxonomy construction. Do not reward a generic survey narrative flow (e.g., foundations -> methods -> applications -> challenges) unless the supplied taxonomy nodes are specifically meaningful for the target scope."
  },
  "instructions": [
    "Use only the supplied target scope and taxonomy outline.",
    "Evaluate whether the taxonomy is well organized for this specific survey topic, not for a broader adjacent field.",
    "Judge taxonomy-construction quality, not final survey prose, writing quality, paper retrieval quality, or citation quality.",
    "A taxonomy may be high quality even if it does not include generic final-survey sections such as Foundations, Background, Applications, Future Work, or Open Challenges.",
    "Conversely, do not give high scores to a taxonomy only because it uses generic survey-outline sections; the categories must be technically specific and useful for the target scope.",
    "You must return all four rubric dimensions, even when a score is low.",
    "Give fine-grained two-decimal scores from 0.00 to 5.00 (e.g., 3.75, 4.20, 2.85); do not default to whole numbers or half points.",
    "Return JSON only.",
    "Return one top-level JSON object with exactly the required keys. Do not return only one dimension.",
    "Do not use a top-level 'reason' key. Put short reasons inside each dimension object."
  ],
  "output_contract": "The response must contain topic_uniqueness, structural_balance, hierarchical_clarity, and logical_organization. The overall average is computed by the harness, so do not worry about reporting it.",
  "required_json_schema": {
    "topic_uniqueness": {
      "score": 0.00,
      "reason": "short phrase"
    },
    "structural_balance": {
      "score": 0.00,
      "reason": "short phrase"
    },
    "hierarchical_clarity": {
      "score": 0.00,
      "reason": "short phrase"
    },
    "logical_organization": {
      "score": 0.00,
      "reason": "short phrase"
    }
  },
  "taxonomy_outline": "<TAXONOMY_OUTLINE>"
}
\end{UserPromptBox}

\subsubsection{Paper-to-Taxonomy Assignment Judgment}
\label{app:prompt_taxonomy_assignment}

\begin{SystemPromptBox}
You are a strict evaluator for paper-to-taxonomy assignment quality. Return exactly one compact JSON object as the first and only output. Do not include chain-of-thought or markdown.
\end{SystemPromptBox}

\begin{UserPromptBox}
{
  "task": "Judge whether a paper is assigned to an appropriate taxonomy node.",
  "topic_id": "<TOPIC_ID>",
  "topic_name": "<TOPIC_NAME>",
  "target_scope": "<TOPIC_SCOPE>",
  "paper": {
    "title": "<PAPER_TITLE>",
    "abstract": "<PAPER_ABSTRACT>",
    "related_work_tree_summary": "<RELATED_WORK_TREE_SUMMARY>"
  },
  "assigned_taxonomy_node_path": "<ASSIGNED_NODE_PATH>",
  "sibling_taxonomy_nodes": [
    "<SIBLING_NODE_PATH_1>",
    "<SIBLING_NODE_PATH_2>"
  ],
  "rubric": {
    "correct_1": "The assigned node path is a substantively reasonable category for the paper's main problem, method, application, or role in the related-work tree.",
    "correct_0": "The assigned node is off-topic, too broad to be meaningful when a better sibling exists, or describes a different method/application/threat/model than the paper."
  },
  "instructions": [
    "Use the topic scope, paper abstract, related-work tree summary, assigned node, and sibling nodes.",
    "Judge the assignment itself, not the writing quality of the taxonomy.",
    "Return correct=1 only when the node is a good taxonomy placement for this paper.",
    "Return correct=0 when the paper only has vague keyword overlap with the node.",
    "Give score from 0 to 5 where 5 is clearly correct and 0 is clearly incorrect.",
    "Return JSON only."
  ],
  "required_json_schema": {
    "correct": 0,
    "score": 0,
    "reason": "short phrase"
  }
}
\end{UserPromptBox}

\subsection{Survey Quality Evaluation}
\label{app:prompt_survey_quality}

\begin{SystemPromptBox}
\yuquan{You are a neutral academic evaluator assessing an anonymized generated survey. Evaluate the supplied survey against the target topic definition using only evidence visible in the provided text. Consider strengths and limitations symmetrically. Do not infer which system produced the survey, and do not reward or penalize document length by itself. When inline citations and a reference list are supplied, evaluate citation practice that is directly observable in the document, but do not claim that a cited source exists or supports a claim without access to that source; source-level citation faithfulness is evaluated separately. Do not think step by step. Return exactly one compact JSON object as the first and only output.}
\end{SystemPromptBox}

\begin{UserPromptBox}[yuquan prompt]
{
  "prompt_version": "q3_balanced_v5_blinded_citation_aware",
  "task": "Score this generated survey paper for the target topic.",
  "score_scale": {
    "minimum": 1,
    "maximum": 5,
    "meaning": {
      "1": "poor or mostly unusable",
      "2": "weak with major gaps",
      "3": "acceptable but incomplete or uneven",
      "4": "strong with identifiable but non-fundamental limitations",
      "5": "excellent for the supplied scope, with no important content-quality limitation visible in the text"
    },
    "half_points_allowed": true
  },
  "score_calibration": {
    "5.0": "Excellent for the supplied scope, with no important visible content-quality limitation.",
    "4.0_to_4.5": "Strong survey with minor limitations, incomplete comparisons, or modest unevenness.",
    "3.0_to_3.5": "Acceptable but clearly incomplete, uneven, or partially off-scope.",
    "1.0_to_2.5": "Weak, substantially off-topic, structurally flawed, or mostly unusable."
  },
  "metrics": {
    "coverage": "How well the survey covers the major concepts and research directions that can reasonably be inferred from the supplied topic scope, with appropriate depth. Do not require datasets or benchmarks unless relevant.",
    "relevance": "How consistently the content contributes to understanding the target topic. Necessary background and useful cross-cutting discussion are relevant.",
    "structure": "How clearly the survey groups research directions, explains relationships and differences among them, and develops a coherent narrative through its taxonomy, hierarchy, transitions, and cross-cutting synthesis.",
    "avg": "Arithmetic mean of coverage, relevance, and structure."
  },
  "important_instructions": [
    "Judge the survey as a research survey, not as a retrieval list.",
    "Evaluate strengths and limitations symmetrically and ground both in concrete section headings or passages from the supplied survey.",
    "Do not reward or penalize length, citation count, or number of sections by itself.",
    "A strong survey should organize and synthesize research directions rather than merely enumerate papers, but detailed coverage is not a weakness when it remains coherent and useful.",
    "Allow necessary background and cross-cutting discussion when it helps explain the target topic; do not treat material as off-topic merely because it extends beyond the exact query wording.",
    "Consider datasets, benchmarks, evaluation protocols, limitations, or future directions only when they are substantively relevant to the target topic; their absence must not trigger an automatic score cap.",
    "Judge coverage only from the supplied topic scope and survey. Do not claim field-wide completeness or identify missing literature that cannot reasonably be inferred from the provided scope.",
    "Use the full 1-5 scale according to the rubric. Do not apply any score cap mechanically; scores must reflect the overall balance of demonstrated strengths and limitations.",
    "Separate topical breadth from redundancy: broad but well-organized coverage may score highly, while repeated content without additional synthesis may be penalized.",
    "Return JSON only.",
    "The full inline citations and reference list are included. Evaluate only observable citation practice: whether important scholarly claims are attributed, citations are integrated into the discussion, and citation markers and the reference list are internally usable and consistent.",
    "Do not infer that a citation is fabricated, unverifiable, or semantically unsupported solely from its title, year, style, or unfamiliarity. Full source-level support is evaluated in a separate citation-faithfulness audit.",
    "Citation observations are qualitative context and must not override the three content scores unless citation presentation directly impairs coverage, relevance, organization, or synthesis."
  ],
  "required_json_schema": {
    "coverage": "number from 1 to 5",
    "relevance": "number from 1 to 5",
    "structure": "number from 1 to 5",
    "avg": "number from 1 to 5",
    "strengths": "list of 1-3 concise strengths, each naming a concrete section or passage",
    "limitations": "list of 0-3 concise limitations, each naming a concrete section or passage",
    "citation_observations": "one concise sentence about observable citation practice, or 'Not assessed' in content-only mode",
    "confidence": "one of: low, medium, high",
    "rationale": "a balanced explanation under 80 words"
  },
  "topic_id": "<TOPIC_ID>",
  "topic_name": "<TOPIC_NAME>",
  "topic_scope": "<TOPIC_SCOPE>",
  "survey_text": "<GENERATED_SURVEY_TEXT>"
}
\end{UserPromptBox}

\subsection{Key-Point Recall Evaluation}
\label{app:prompt_key_point_recall}

Key-Point Recall is evaluated in two stages: extracting a fixed set of key
points from the human-written survey and checking their presence in each
generated survey.

\subsubsection{Human-Survey Key-Point Extraction}
\label{app:prompt_key_point_extraction}

\begin{SystemPromptBox}
You are a careful research-survey evaluator. Return exactly one compact JSON object as the first and only output. Do not include chain-of-thought.
\end{SystemPromptBox}

\begin{UserPromptBox}
{
  "task": "Extract key points from a human-written academic survey.",
  "topic_id": "<TOPIC_ID>",
  "topic_name": "<TOPIC_NAME>",
  "definition": "A key point is a self-contained, survey-level scientific point: a major subtopic, method family, comparison axis, finding, dataset/evaluation theme, limitation, or future direction that a generated survey should cover.",
  "instructions": [
    "Extract key points from the human-written survey text only.",
    "Do not output paper titles as standalone key points.",
    "Prefer conceptual points that can be checked against another survey.",
    "Each key point must be one concise sentence.",
    "Avoid duplicates and overly narrow implementation details.",
    "Return exactly 10 key points, balanced across the survey."
  ],
  "required_json_schema": {
    "key_points": [
      {
        "id": "KP01",
        "key_point": "one self-contained sentence"
      }
    ]
  },
  "human_survey_chunks": [
    {
      "chunk_id": 1,
      "text": "<HUMAN_SURVEY_CHUNK_1>"
    },
    {
      "chunk_id": 2,
      "text": "<HUMAN_SURVEY_CHUNK_2>"
    }
  ]
}
\end{UserPromptBox}

\subsubsection{Generated-Survey Key-Point Presence Check}
\label{app:prompt_key_point_presence}

\begin{UserPromptBox}[yuquan prompt]
{
  "task": "QA-style key point presence check.",
  "topic_id": "<TOPIC_ID>",
  "topic_name": "<TOPIC_NAME>",
  "chunk_index": "<CHUNK_INDEX>",
  "chunk_count": "<CHUNK_COUNT>",
  "question": "For each human-survey key point, decide whether the generated survey chunk explicitly or clearly semantically covers it.",
  "scoring_rule": "present=1 only if the generated survey covers the substance of the key point. present=0 if absent, only vaguely related, contradicted, or only mentioned as an unexplained keyword.",
  "required_json_schema": {
    "results": [
      {
        "id": "KP01",
        "present": 0,
        "evidence": "short supporting phrase, or empty if absent"
      }
    ]
  },
  "key_points": [
    {
      "id": "KP01",
      "key_point": "<HUMAN_SURVEY_KEY_POINT>"
    }
  ],
  "generated_survey_chunk": "<GENERATED_SURVEY_CHUNK>"
}
\end{UserPromptBox}

\subsection{\yuquan{Citation Recall and Precision Evaluation}}
\label{app:prompt_citation_faithfulness}

This prompt jointly obtains the group-level judgments used for Citation Recall
and the citation-level judgments used for Citation Precision.

\begin{SystemPromptBox}
You are a strict citation-faithfulness evaluator for academic survey text. Return exactly one compact JSON object as the first and only output. Do not include chain-of-thought.
\end{SystemPromptBox}

\begin{UserPromptBox}[yuquan prompt]
{
  "task": "Q3-3 citation faithfulness evaluation.",
  "metric_definitions": {
    "citation_recall": "For each claim c_i, h(c_i, Ref_i)=1 iff the whole citation group Ref_i is sufficient to support the claim; otherwise 0.",
    "citation_precision": "For each citation r_ik in Ref_i, the citation is relevant/necessary iff it alone supports the claim, OR the full group supports the claim and removing r_ik makes the remaining citation group insufficient."
  },
  "judging_instructions": [
    "Judge conservatively from the claim text and the provided citation entries only.",
    "A citation entry is usually a title/URL/metadata string, not the full paper.",
    "Do not assume support that is not clearly implied by the citation entry.",
    "If a citation entry is missing or unresolved, treat it as not supporting unless another provided entry is sufficient.",
    "For one-citation groups, without_support must be 0 because no citation remains.",
    "Return binary 0/1 values only, with short evidence phrases."
  ],
  "required_json_schema": {
    "results": [
      {
        "claim_id": "C0001",
        "group_support": 0,
        "group_evidence": "short phrase",
        "citations": [
          {
            "ref_id": "1",
            "single_support": 0,
            "without_support": 0,
            "evidence": "short phrase"
          }
        ]
      }
    ]
  },
  "claims": [
    {
      "claim_id": "<CLAIM_ID>",
      "topic_name": "<TOPIC_NAME>",
      "claim": "<CLAIM_TEXT>",
      "citation_group": [
        {
          "ref_id": "<REFERENCE_ID>",
          "entry": "<REFERENCE_ENTRY>",
          "resolved": "<BOOLEAN>"
        }
      ]
    }
  ]
}
\end{UserPromptBox}

\subsection{\yuquan{Passage-Based Citation Support Judgment}}
\label{app:prompt_passage_faithfulness}

\yuquan{The following is the full-text prompt used for the 74 pairs with
retrievable cited sources in Appendix~\ref{app:passage_faithfulness}. It has no
method identifier. The judge returns three labels; the reported binary
Supported rate combines fully and partially supported, while unverifiable
references are marked Unsupported without a model call. The source text is
supplied with page markers of the form [PAGE N].}

\begin{SystemPromptBox}
\yuquan{You are a strict academic citation-faithfulness evaluator.
You receive one survey claim and the complete page-marked text of one cited paper.
Judge only whether that paper supports the claim. Return one JSON object only.
Do not use outside knowledge. Identify the strongest source passage and its page.}
\end{SystemPromptBox}

\begin{UserPromptBox}[yuquan prompt]
{
  "task": "Judge whether the cited paper supports this survey claim.",
  "definitions": {
    "fully_supported": "The paper supports every material assertion in the claim.",
    "partially_supported": "The paper supports only part of the claim or the claim overstates it.",
    "unsupported": "The paper does not support the claim or contradicts it."
  },
  "claim": "<CLAIM_TEXT>",
  "citation_label": "Citation A",
  "rules": [
    "Read the complete supplied paper text before deciding.",
    "Do not infer support from the title or outside knowledge.",
    "For fully_supported or partially_supported, give the strongest supporting passage and page.",
    "For unsupported, give the closest relevant passage and explain the gap.",
    "Copy passages as faithfully as possible from the supplied text."
  ],
  "required_json": {
    "label": "fully_supported | partially_supported | unsupported",
    "passage": "supporting or closest source passage",
    "page": 1,
    "unsupported_parts": ["unsupported portion, if any"],
    "reason": "brief justification",
    "confidence": "high | medium | low"
  },
  "cited_paper_full_text": "<CITED_PAPER_FULL_TEXT>"
}
\end{UserPromptBox}

\end{document}